\documentclass[a4paper,11pt]{article}

\usepackage{lipsum}                  
\usepackage{xargs}
\usepackage{xfrac}
\usepackage{amssymb}
\usepackage{amsmath}

\newcommand{\tr}{\text{Tr}}

\newcommand{\h}{{\bf h}}

\newcommand{\rank}{{\text{rank}}}

\newcommand{\w}{ w}

\newcommand{\blambda}{\lambda}

\newcommand{\y}{{\bf y}}

\newcommand{\bigo}{\text{O}}

\def\Hs{{\mathcal H}}

\newcommand{\B}{{\bf B}}
\newcommand{\M}{M}
\newcommand{\W}{ W}

\newcommand{\ones}{{\bf 1}}

\newcommand{\E}{\mathbb{E}}

\newcommand{\N}{{\mathcal N}}
\newcommand{\R}{{\mathbb{R}}}

\newcommand{\bn}{\text{\sc Bn}}
\newcommand{\bnp}{\text{\sc Bn}\;} 

\newcommand{\I}{I}

\renewcommand{\P}{{\mathcal P}}
 
\newcommand{\diag}{{\text{diag}}}

\usepackage{thmtools,thm-restate}

\usepackage{eqparbox}
\newcommand{\comment}[1]{}

\def\Wm{{\bf W}}

\usepackage{color}
\usepackage{fullpage}
\usepackage{amsmath}
\usepackage{amsthm}
\usepackage{microtype}
\usepackage{graphicx}
\usepackage{subfigure}
\usepackage{rotating}
\usepackage{framed}
\usepackage{algorithm}
\usepackage{algorithmic}

\usepackage{multirow}

\usepackage{booktabs} 
\usepackage[export]{adjustbox}

\usepackage{hyperref}
\usepackage{amsfonts} 
\usepackage{nicefrac} 
\usepackage{microtype}
\usepackage{booktabs} 
\usepackage[utf8]{inputenc} 
\usepackage[T1]{fontenc}   
\usepackage{hyperref}     
\usepackage{url}

\author{Hadi Daneshmand$^{1*}$, Jonas Kohler$^{1}$\thanks{Shared first authorship, $^{1}$ETH Zurich, $^2$INRIA Paris}, \\ Francis Bach$^{2}$, Thomas Hofmann$^{1}$, Aurelien Lucchi$^{1}$ }

\title{Batch Normalization Provably Avoids Rank Collapse for Randomly Initialised Deep Networks}

\newtheorem{Theorem}{Theorem}
\newtheorem{Definition}{Definition}
\newtheorem{definition}[Definition]{Definition}
\newtheorem {theorem}[Theorem]    {Theorem}
\newtheorem {lemma} [Theorem]    {Lemma}
\newtheorem {corollary}  [Theorem]    {Corollary}

\newtheorem {proposition}[Theorem]    {Proposition}

\newtheorem {example} [Theorem]    {Example}

\begin{document}

\maketitle

\begin{abstract}
Randomly initialized neural networks are known to become harder to train with increasing depth, unless architectural enhancements like residual connections and batch normalization are used. We here investigate this phenomenon by revisiting the connection between random initialization in deep networks and spectral instabilities in products of random matrices. Given the rich literature on random matrices, it is not surprising to find that the rank of the intermediate representations in unnormalized networks collapses quickly with depth. In this work we highlight the fact that batch normalization is an effective strategy to avoid rank collapse for both linear and ReLU networks. Leveraging tools from Markov chain theory, we derive a meaningful lower rank bound in deep linear networks. Empirically, we also demonstrate that this rank robustness generalizes to ReLU nets.Finally, we conduct an extensive set of experiments on real-world data sets, which confirm that rank stability is indeed a crucial condition for training modern-day deep neural architectures.
\end{abstract}

\section{Introduction and related work}

Depth is known to play an important role in the expressive power of neural networks~\cite{telgarsky2016benefits}. Yet, increased depth typically leads to a drastic slow down of learning with gradient-based methods, which is commonly attributed to unstable gradient norms in deep networks~\cite{hochreiter1998vanishing}. One key aspect of the training process concerns the way the layer weights are initialized. When training contemporary neural networks, both practitioners and theoreticians advocate the use of randomly initialized layer weights with i.i.d.~entries from a zero mean (Gaussian or uniform) distribution. This initialization strategy is commonly scaled such that the variance of the layer activation stays constant across layers \cite{glorot2010understanding,he2015delving}. However, this approach can not avoid spectral instabilities as the depth of the network increases. For example,

\cite{saxe2013exact} observes that for linear neural networks, such initialization lets all but one singular values of the last layers activation collapse towards zero as the depth increases. 

Nevertheless, recent advances in neural architectures have allowed the training of very deep neural networks with standard i.i.d.~initialization schemes \textit{despite} the above mentioned shortcomings. Among these, both residual connections and normalization layers have proven particularly effective and are thus in widespread use (see \cite{ioffe2015batch,salimans2016weight,he2015delving} to name just a few). Our goal here is to bridge the explanatory gap between these two observations by studying the effect of architectural enhancements on the spectral properties of randomly initialized neural networks. We also provide evidence for a strong link of the latter with the performance of gradient-based optimization algorithms.

One particularly interesting architectural component of modern day neural networks is Batch Normalization (\bn) \cite{ioffe2015batch}. This simple heuristics that normalizes the pre-activation of hidden units across a mini-batch, has proven tremendously effective when training deep neural networks with gradient-based methods. Yet, despite of its ubiquitous use and strong empirical benefits, the research community has not yet reached a broad consensus, when it comes to a theoretical explanation for its practical success. Recently, several alternatives to the original
``internal covariate shift'' hypothesis \cite{ioffe2015batch} have appeared in the literature: decoupling optimization of direction and length of the parameters~\cite{kohler2018exponential}, auto-tuning of the learning rate for stochastic gradient descent~\cite{arora2018theoretical}, widening the learning rate range \cite{bjorck2018understanding}, alleviating sharpness of the Fisher information matrix \cite{karakida2019normalization}, and smoothing the optimization landscape~\cite{santurkar2018does}. Yet, most of these candidate justifications are still actively debated  within the community. For example, \cite{santurkar2018does} first made a strong empirical case against the original internal covariate shift hypothesis. Secondly, they argued that batch normalization simplifies optimization by smoothing the loss landscape. However, their analysis is on a per-layer basis and treats only the largest eigenvalue. Furthermore, even more recent empirical studies again dispute these findings, by observing the exact opposite behaviour of \bn~on a ResNet20 network~\cite{yao2019pyhessian}. 

\subsection{On random initialization and gradient based training}

In light of the above discussion, we take a step back -- namely to the beginning of training -- to find an interesting property that is provably present in batch normalized networks and can serve as a solid basis for a more complete theoretical understanding. 

The difficulty of training randomly initialized, un-normalized deep networks with gradient methods is a long-known fact, that is commonly attributed to the so-called vanishing gradient effect, i.e.,~a decreasing gradient norm as the networks grow in depth (see, e.g., \cite{schoenholz2016deep}).
A more recent line of research tries to explain this effect by the condition number of the input-output Jacobian (see, e.g.,  \cite{yang2017mean,yang2019mean,pennington2018emergence,bjorck2018understanding}). Here, we study the spectral properties of the above introduced initialization with a particular focus on the rank of the hidden layer activations over a batch of samples. The question at hand is whether or not the network preserves a diverse data representation which is necessary to disentangle the input in the final classification layer.

As a motivation, consider the results of Fig.~\ref{fig:learning_depth}, which plots accuracy and output rank when training batch-normalized and un-normalized neural networks of growing depth on the Fashion-MNIST dataset~\cite{xiao2017fashion}. As can be seen, the rank in the last hidden layer of the vanilla networks collapses with depth and they are essentially unable to learn (in a limited number of epochs) as soon as the number of layers is above 10. The rank collapse indicates that the direction of the output vector has become independent of the actual input. In other words, the randomly initialized network no longer preserves information about the input. Batch-normalized networks, however, preserve a high rank across all network sizes and their training accuracy drops only very mildly as the networks reach depth 32.

\begin{figure}[h!]
 \centering 
          \begin{tabular}{c@{}c@{}}
          \includegraphics[width=0.475\linewidth]{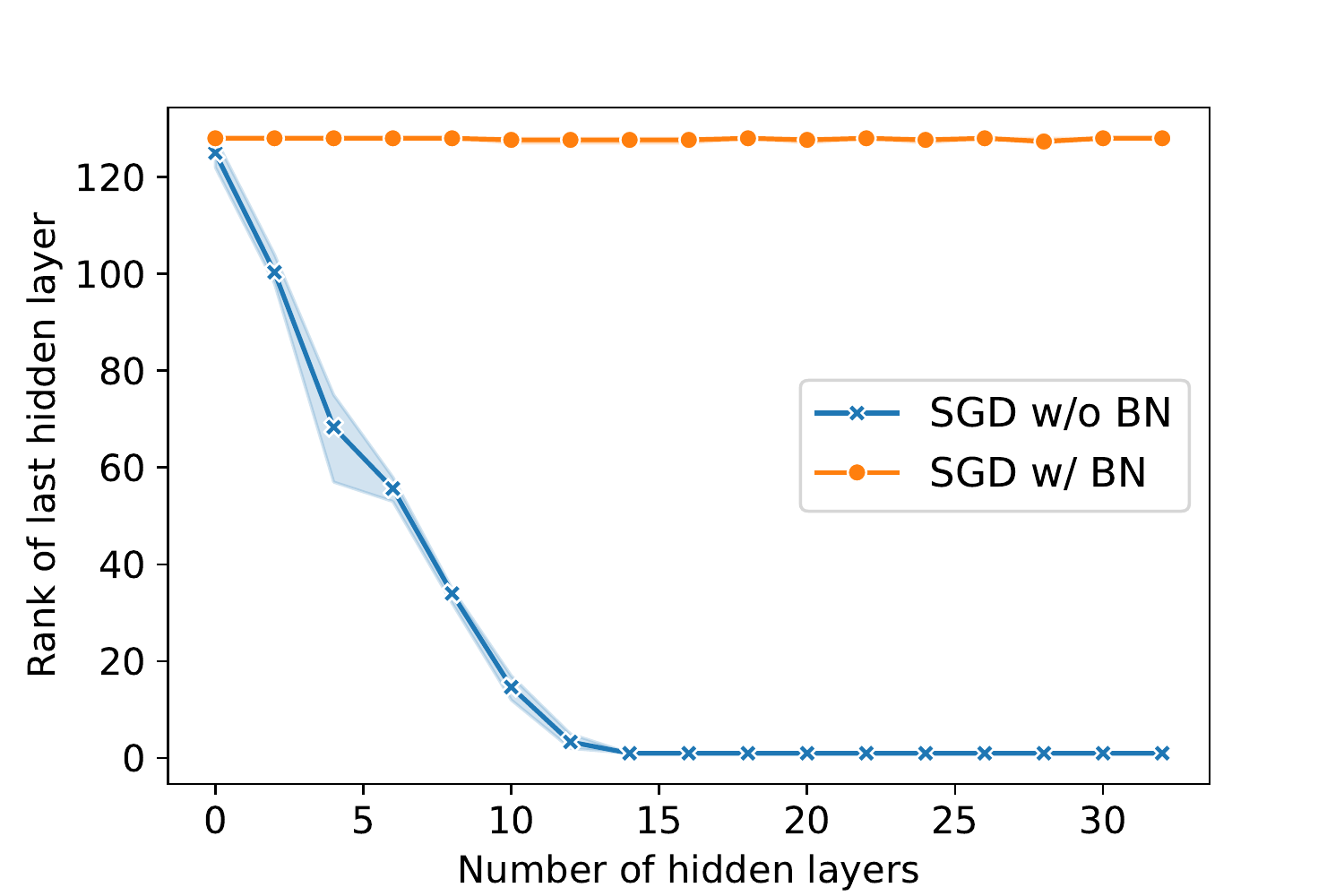}  &    \includegraphics[width=0.475\linewidth]{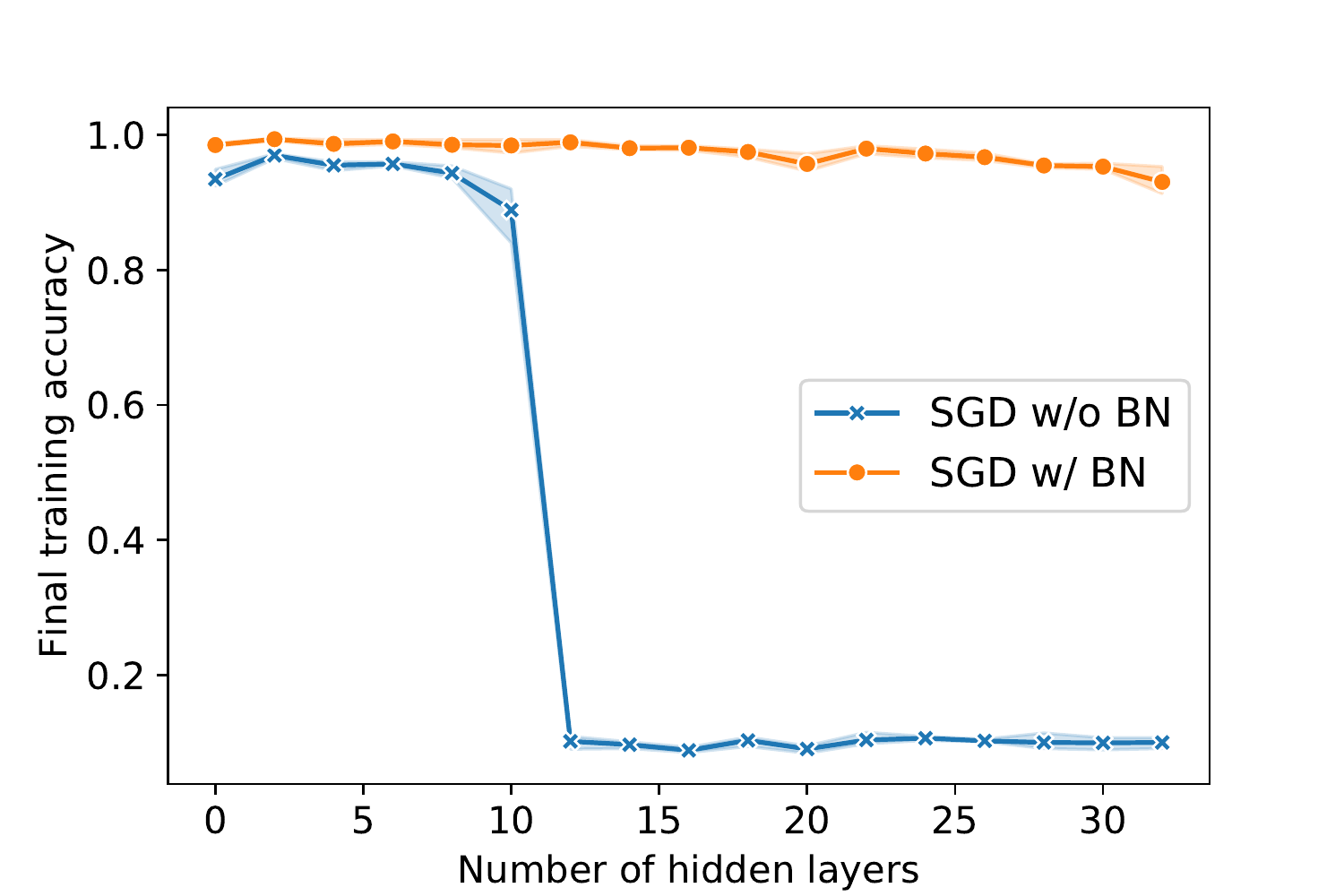}
        
 	  \end{tabular}
    \caption{\footnotesize{\textbf{Effect of depth on rank and learning}, on the Fashion-MNIST dataset with ReLU multilayer perceptrons (MLPs) of depth 1-32 and width 128 hidden units. Left: Rank\protect\footnotemark~ after random initialization as in PyTorch~\cite{paszke2019pytorch,glorot2010understanding}. Right: Training accuracy after training 75 epochs with SGD, batch size 128 and grid-searched learning rate. Mean and 95\% confidence interval of 5 independent runs.}}
    \label{fig:learning_depth}
     
\end{figure}

\footnotetext{Computed using $torch.matrix\_rank()$, which regards singular values below $\sigma_{\max} \times d \times 10^{-7}$ as zero. This is consistent with both Matlab and Numpy.}

The above example shows that both rank and optimization of even moderately-sized, unnormalized networks scale poorly with depth. 
Batch-normalization, however, stabilizes the rank in this setting and the obvious question is whether this effect is just a slow-down or even simply a numerical phenomenon, or whether it actually generalizes to networks of infinite depth.

In this work we make a strong case for the latter option by showing a remarkable stationarity aspect of \bn{}. Consider for example the case of passing $N$ samples $x_i \in \R^d$ arranged column-wise in an input matrix $X \in \R^{d \times N}$ through a very deep network with fully-connected layers. Ideally, from an information propagation perspective, the network should be able to differentiate between individual samples, regardless of its depth \cite{schoenholz2016deep}. However, as can be seen in Fig.~\ref{fig:vanilla_bn_rank}, the hidden representation of $X$ collapses to a rank one matrix in vanilla networks, thus mapping all $x_i$ to the same line in~$\R^d$. Hence, the hidden layer activations and along with it the individual gradient directions become independent from the input $x_i$ as depth goes to infinity. We call this effect ``directional'' gradient vanishing (see Section \ref{sec:rank_collapse} for a more thorough explanation). 

Interestingly,
this effect does not happen in batch-normalized networks, which yield -- as we shall prove in Theorem \ref{thm:linear-rank-lowerboun} -- a stable rank for \textit{any} depth, thereby preserving a disentangled representation of the input and hence allowing the training of very deep networks. These results substantiate earlier empirical observations made by  \cite{bjorck2018understanding} for random \bn{}-nets, and also validates the claim that \bn{} helps with \textit{deep information propagation}~\cite{schoenholz2016deep}.

\begin{figure*}[h!]
    \centering
    \begin{tabular}{c c c}
      \includegraphics[width=0.4\textwidth]{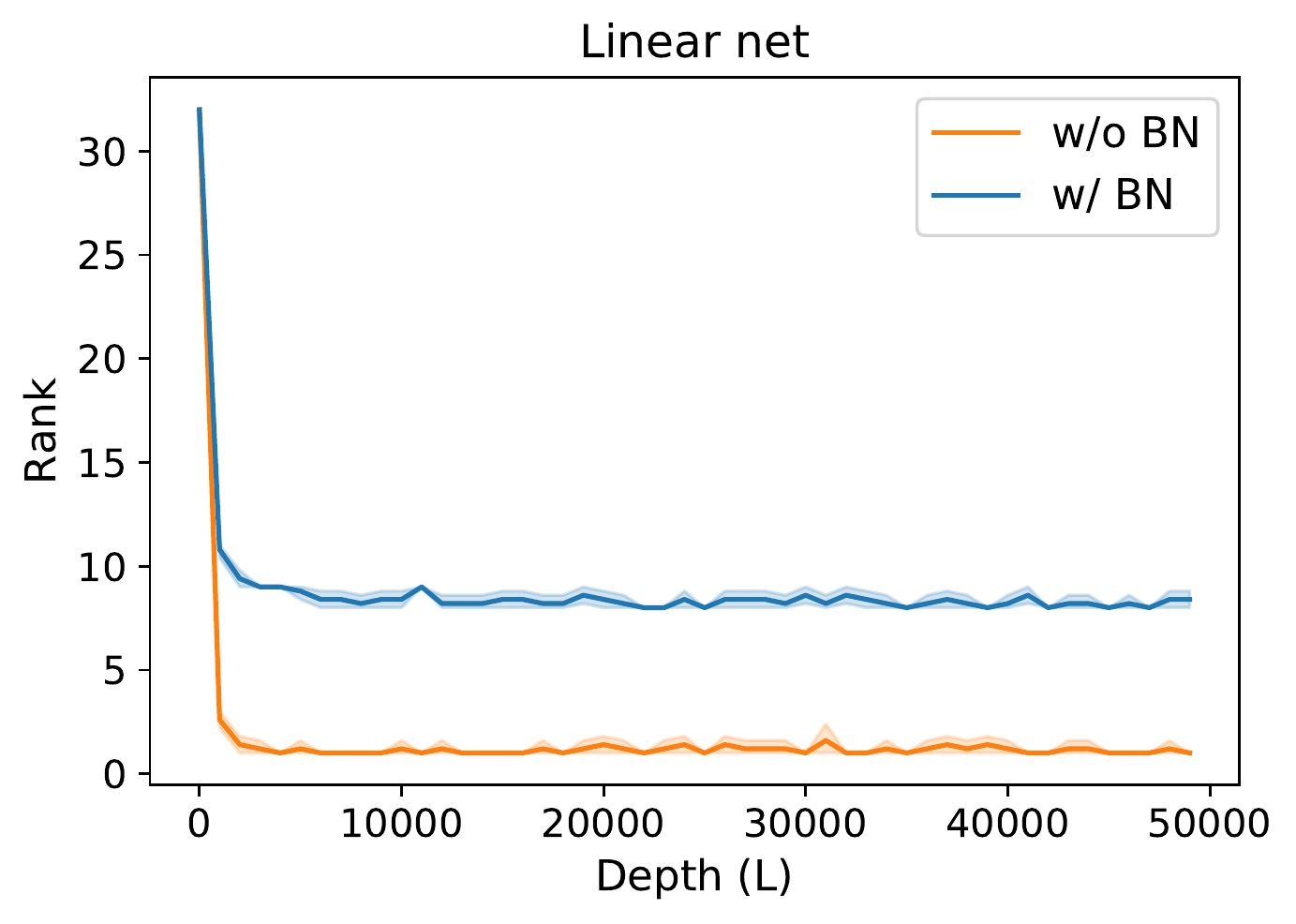} & \includegraphics[width=0.4\textwidth]{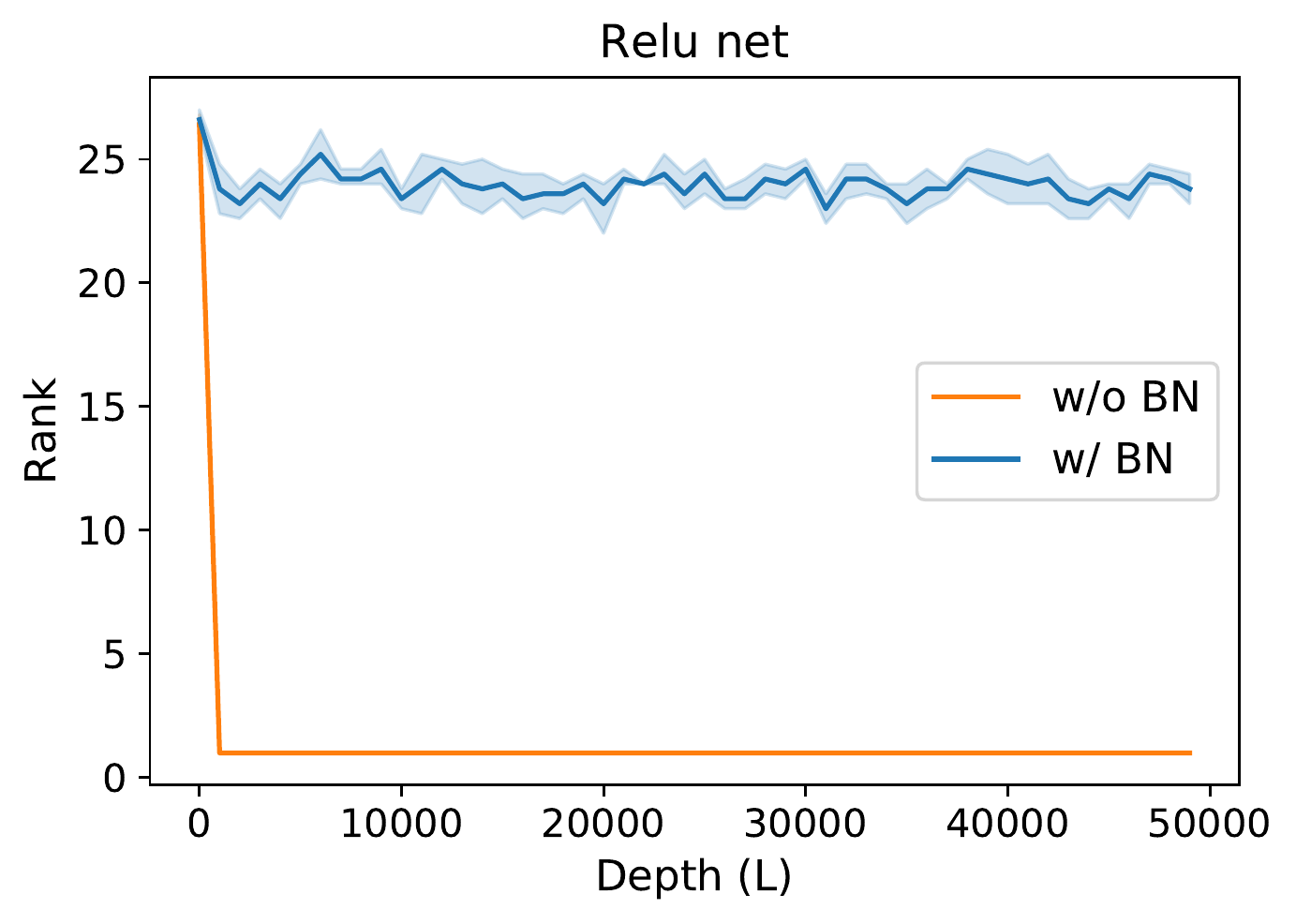}  
    \end{tabular}
    \caption{\textbf{Rank comparison of last hidden activation}: Log(rank) of the last hidden layer's activation over total number of layers (blue for \bn{}- and orange for vanilla-networks) for Gaussian inputs. Networks are MLPs of width $d=32$. (Left) Linear activations, (Right) ReLU activations. Mean and $95\%$ confidence interval of 10 independent runs. While the rank quickly drops in depth for both networks, BN stabilizes the rank above $\sqrt{d}$.}
    \label{fig:vanilla_bn_rank}
\end{figure*}
\subsection{Contributions}
In summary, the work at hand makes the following two key contributions: 

\textbf{(i)} We theoretically prove that \bnp indeed avoids rank collapse for deep linear neural nets under standard initialization and for any depth.
In particular, we show that \bn~can be seen as a computationally cheap rank preservation operator, which may not yield hidden matrices with full rank but still preserves sufficient modes of variation in the data to achieve a scaling of the rank with $\Omega(\sqrt{d}$), where $d$ is the width of the network. Subsequently, we leverage existing results from random matrix theory~\cite{bougerol2012products} to complete the picture with a simple proof of the above observed rank collapse for linear vanilla networks, which interestingly holds regardless of the presence of residual connections (Lemma \ref{lemma:rank_vanishing}). Finally, we connect the rank to difficulties in gradient based training of deep nets by showing that a rank collapse makes the directional component of the gradients independent of the input.

\textbf{(ii)} We empirically show that the rank is indeed a crucial quantity for gradient-based learning. In particular, we show that both the rank and the final training accuracy quickly diminish in depth unless \bn~ layers are incorporated in both simple feed-forward and convolutional neural nets. To take this reasoning beyond mere correlations, we actively intervene with the rank of networks before training and show that (a) one can break the training stability of \bn~by initializing in a way that reduces its rank-preserving properties, and (b) a rank-increasing pre-training procedure for vanilla networks can recover their training ability even for large depth. Interestingly, our pre-training method allows vanilla SGD to outperform \bn{} on very deep MLPs. In all of our experiments, we find that SGD updates preserve the order of the initial rank throughout optimization, which underscores the importance of the rank at initialization for the entire convergence behavior.

\section{Background and Preliminaries}
\paragraph{Network description.}

We consider a given input $X \in \R^{d \times N}$ containing $N$ samples in $\R^d$. Let $\ones_k \in \R^k$ denote the k-dimensional all one vector and $H_\ell^{(\gamma)}$ denote the hidden representation of $X$ in layer $\ell$ of a \bn{}-network with residual connections. The following recurrence summarizes the network mapping
\begin{align}\label{eq:nn}
    H_{\ell+1}^{(\gamma)} = \bn_{0, \ones_d}(H_\ell^{(\gamma)} + \gamma W_\ell H_\ell^{(\gamma)}), \quad  H_{0}^{(\gamma)} = X,
\end{align}
where $W_\ell \in \R^{d \times d}$ and $\gamma$ regulates the skip connection strength (in the limit, $\gamma = \infty$ recovers a network without skip connection)\footnote{For the sake of simplicity, we here assume that the numbers of hidden units is equal across layers. In App.~\ref{sec:main_theorm} we show how our results extend to nets with varying numbers of hidden units.}. Throughout this work, we consider the network weights $W_\ell$ to be initialized as follows. 
\begin{definition}[Standard weight initialization]
\label{definition:weight_init} The elements of weight matrices $W_\ell$ are i.i.d. samples from a distribution $\P$ that has zero-mean, unit-variance, and its density is symmetric around zero\footnote{Two popular choices for $\P$ are the Gaussian distribution $\N(0,1)$ and the uniform distribution $\mathcal{U}([-1,1])$. The variance can be scaled with the choice of $\gamma$ to match the prominent initializations from \cite{he2015delving} and \cite{glorot2010understanding}. Note that the symmetry implies that the law of each element $[W_\ell]_{ij}$ equates the law of $-[W_\ell]_{ij}$.}. We use the notation $\mu$ for the probability distribution of the weight matrices.
\end{definition}
 We define the BN operator $\bn_{\alpha,\beta}$ as in the original paper \cite{ioffe2015batch}, namely
\begin{align}
     \bn_{\alpha,\beta}(H) = \beta \circ \left( \diag \left( \M(H) \right) \right)^{-\sfrac{1}{2}} H +  \alpha \ones^\top_N, M(H) & := \frac{1}{N} H H^\top \label{eq:M},
\end{align}
where $\circ$ is a row-wise product. Both $\alpha \in \R^d$ and $\beta \in \R^d$ are trainable parameters. Throughout this work we assume the initialization $\alpha=0$ and $\beta= \ones_d$, and also omit corrections of the mean activity. As demonstrated empirically in Fig.~\ref{fig:broken_bn}, and theoretically in App.~\ref{section:init_consequences_appendix} this simplification does not change the performance of \bn{} in our settings.

\paragraph{Rank notions.} \label{sec:rank_notions}

To circumvent numerical issues involved in rank computations we introduce a soft notion of the rank denoted by $\rank_\tau(H)$ (soft rank). Specifically, let $\sigma_1, \dots, \sigma_d$ be the singular values of $H$. Then, given a $\tau>0$, we define $\rank_\tau(H)$ as
\begin{align} \label{eq:softrank}
   \rank_\tau(H) = \sum_{i=1}^d \ones(\sigma^2_i/N \geq \tau).
\end{align}
Intuitively, $\rank_\tau(H)$ indicates the number of singular values whose absolute values are greater than $\sqrt{N\tau}$. It is clear that $\rank_\tau(H)$ is less or equal to $\rank(H)$ for all matrices $H$.  For analysis purposes, we need an analytic measure of the collinearity of the columns and rows of $H$. Inspired by the so-called stable rank (see, e.g.,  \cite{tropp2015introduction}), we thus introduce the following quantity 
\begin{align}\label{eq:rH}
    r(H) = \tr(M(H))^2/\| M(H) \|_F^2, \quad M(H) = H H^\top/N. 
\end{align}
In contrast to the algebraic rank, $r(H)$ is differentiable with respect to $H$. Furthermore, the next lemma proves that the above quantity lower-bounds the soft-rank for the hidden representations.
\begin{lemma} \label{lemma:soft_rank_lowerbound}
For an arbitrary matrix $H \in \R^{d \times d}$, $\rank(H) \geq r(H)$.   For the sequence $\{H_\ell^{(\gamma)}\}_{\ell=1}^\infty$ defined in Eq.~\eqref{eq:M}, $\rank_\tau(H_\ell^{(\gamma)}) \geq (1-\tau)^2 r(H_\ell^{(\gamma)})$ holds for $\tau \in [0,1]$.
\end{lemma}

\section{Batch normalization provably prevents rank collapse} \label{sec:rank_collapse}
Since our empirical observations hold equally for both non-linear and linear networks, we here focus on improving the theoretical understanding in the linear case, which constitutes a growing area of research \cite{saxe2013exact,kawaguchi2016deep,bartlett2019gradient,arora2018convergence}.
First, inspired by \cite{dieuleveut2017bridging} and leveraging tools from Markov Chain theory, our main result proves that the rank of linear batch-normalized networks scales with their width as $\Omega(\sqrt{\text{width}})$.
Secondly, we leverage results from random matrix theory \cite{boucheron2013concentration} to contrast our main result to unnormalized linear networks which we show to provably collapse to rank one, even in the presence of residual connections.

\subsection{Main result}
In the following we state our main result which proves that batch normalization indeed prevents the rank of all hidden layer activations from collapsing to one. Please see Appendix \ref{sec:main_theorm} for the more formal version of this theorem statement.

\begin{theorem} \label{thm:linear-rank-lowerboun}[Informal]
Suppose that the $\rank(X) = d$ and that the weights $W_\ell$ are initialized in a standard i.i.d.~zero-mean fashion (see Def.~\ref{definition:weight_init}). Then, the following limits exist such that 
\begin{align} \label{eq:average_lowerbound}
    \lim_{L \to \infty} \frac{1}{L} \sum_{\ell=1}^L \rank_\tau(H_\ell^{(\gamma)})  \geq \lim_{L \to \infty} \frac{(1-\tau)^2}{L} \sum_{\ell=1}^L  r(H_\ell^{(\gamma)})= \Omega((1-\tau)^2\sqrt{d})
\end{align}
holds almost surely for a sufficiently small $\gamma$ (independent of $\ell$) and any $\tau \in [0,1)$, under some additional technical assumptions. Please see Theorem \ref{thm:linear-rank-lowerbound-formal} in the Appendix for the formal statement.
\end{theorem}

Theorem~\ref{thm:linear-rank-lowerboun} yields a non trivial width-dependency. Namely, by setting for example $\tau := {1}/{2}$, the result states that the average number of singular values with absolute value greater than $\sqrt{N/2}$ is at least 
$\Omega(\sqrt{d})$ on average. To put this into context: If one were to replace $\diag(M)^{-\sfrac{1}{2}}$ by the \textit{full} inverse $(M)^{-\sfrac{1}{2}}$ in Eq.~\eqref{eq:M}, then \bn{} would effectively constitute a classical whitening operation such that all $\{H_{\ell}^{(\gamma)}\}_{\ell=1}^L$ would be full rank (equal to $d$). However, as noted in the original \bn~ paper \cite{ioffe2015batch}, whitening is obviously expensive to compute and furthermore prohibitively costly to incorporate in back-propagation. As such, \bn~can be seen as a computationally inexpensive approximation of whitening, which does not yield full rank hidden matrices but still preserves sufficient variation in the data to provide a rank scaling as $\Omega(\sqrt{d}$). Although the lower-bound in Thm.~\ref{thm:linear-rank-lowerboun} is established on the average over infinite depth (i.e., $L\to \infty$), Corollary~\ref{cor:nonasymptotic} (in App.~\ref{sec:main_theorm}) proves that the same bound holds for all $\rank(H_\ell)$ and $\rank_\tau(H_\ell)$.

\paragraph{Necessary assumptions.}
The above result relies on two key assumptions:  (i) First, the input $X$ needs to be full rank. (ii) Second, the weights have to be drawn according to the standard initialization scheme. We believe that both assumptions are indeed necessary for \bn{} to yield a robust rank. 

Regarding (i), we consider a high input rank a natural condition since linear neural nets cannot possibly \textit{increase} the rank when propagating information through their layers. Of course, full rank is easily achieved by an appropriate data pre-processing. Yet, even when the matrix is close to low rank we find that \bn~ is actually able to amplify small variations in the data (see Fig.~\ref{fig:markov_chain_d}.b).\footnote{Intuitively this means that even if two data points are very close to each other in the input space, their hidden presentation can still be disentangled in batch-normalized networks (see Appendix~\ref{sec:main_theorm} for more details)} 
Regarding (ii), we derive -- based on our theoretical insights -- an adversarial initialization strategy that corrupts both the rank robustness and optimization performance of batch-normalized networks, thus suggesting that the success of \bn~ indeed relies heavily on the standard i.i.d. zero-mean initialization.

 \paragraph{Experimental validation.}
In order to underline the validity of Theorem~\ref{thm:linear-rank-lowerboun} we run multiple simulations by feeding Gaussian data of dimensionality $d=N$ into networks of growing size and with different residual strengths. For each network, we compute the mean and standard deviation of the soft rank $\rank_\tau$ with $\tau=0.5$. As depicted in Fig.~\ref{fig:markov_chain_d}, the curves clearly indicate a $\Omega(\sqrt{d})$ dependency for $\lim_{L\to \infty} \sum_{\ell=1}^L \rank_\tau(H_{\ell})/L$, just as predicted in the Theorem. Although the established guarantee requires the weight on the parametric branch (i.e.,~$\gamma$) to be small, the results of Fig.~\ref{fig:markov_chain_d} indicate that the established lower bound holds for a much wider range including the case where no residual connections are used at all ($\gamma=\infty$). 
\begin{figure}[h!]
 \centering 
          \begin{tabular}{c@{}c}
             
     \adjincludegraphics[width=0.43\linewidth, trim={10pt 1pt 30pt 30pt},clip]{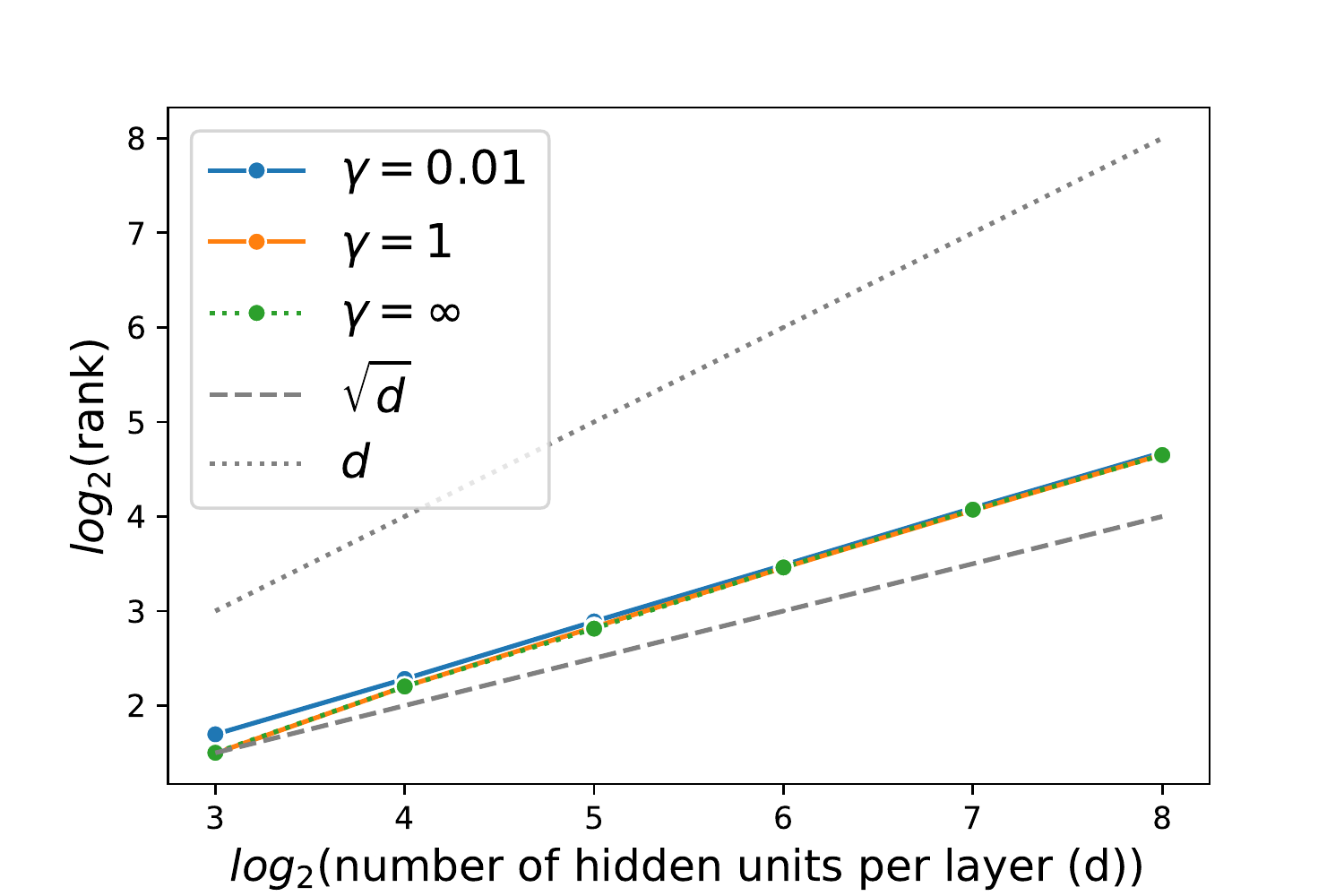} &  \adjincludegraphics[width=0.43\linewidth, trim={10pt 1pt 30pt 30pt},clip]{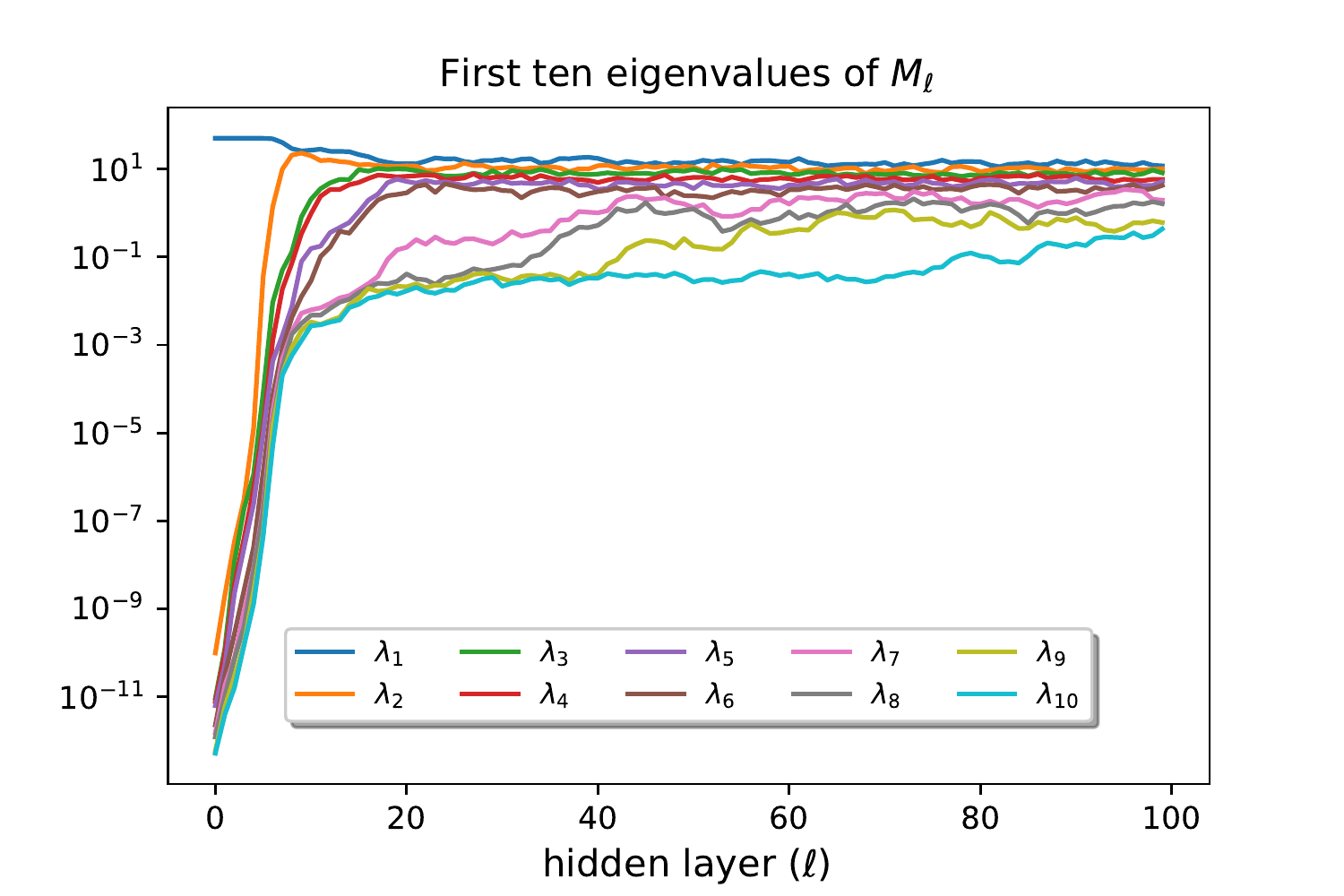} \\ 
     a) & b)
          
 	  \end{tabular}
          \caption{\footnotesize{a) Result of Theorem~\ref{thm:linear-rank-lowerboun} for different values of $\gamma$, where $\gamma=\infty$ depicts networks \textit{without} skip connections. Each point is the average $\rank_{\sfrac{1}{2}}$ over depth ($L=10^6$) of nets of width $d\in\{8,16,..,256\}$ an on x-axis. b) Top 10 singular values of $H_\ell^{(\gamma)}$ for increasing values of $\ell$ given nearly collinear inputs. As can be seen, \bn~ quickly amplifies smaller variations in the data while reducing the largest one. }}
          \label{fig:markov_chain_d}
\end{figure}

\subsection{Comparison with unnormalized networks}
In order to stress the importance of the above result, we now compare the predicted rank of $H_\ell$ with the rank of unnormalized linear networks, which essentially constitute a linear mapping in the form of a product of random matrices. The spectral distribution of products of random matrices with i.i.d.~standard Gaussian elements has been studied extensively \cite{bjorck2018understanding,forrester2013lyapunov,liu2016bulk}. Interestingly, one can show that the gap between the top and the second largest singular value increases with the number of products (i.e., $\ell$) at an exponential rate\footnote{The growth-rate of the $i$-th singular value is determined by the $i$-th Lyapunov exponent of the product of random matrices. We refer the reader to \cite{forrester2013lyapunov} for more details on Lyapunov exponents.} \cite{forrester2013lyapunov,liu2016bulk}. Hence, the matrix converges to a rank one matrix after normalizing by the norm. In the following, we extend this result to products of random matrices with a residual branch that is obtained by adding the identity matrices. Particularly, we consider the hidden states $\widehat{H}_\ell$ of the following linear residual network:
\begin{align}\label{eq:linear_net}
    \widehat{H}_\ell = \B_\ell X, \quad \B_\ell := \prod_{k=1}^\ell (I+\gamma \W_k).
\end{align}

Since the norm of $\widehat{H}_\ell$ is not necessarily bounded, we normalize as
$
    \widetilde{H}_\ell = B_\ell X/\|B_\ell\|.
$
 The next lemma characterizes the limit behaviour of $\{ \widetilde{H}_\ell\}$. 
 
\begin{lemma} \label{lemma:rank_vanishing}
Suppose that $\gamma \in (0,1)$ and assume the weights $\W_\ell$ to be initialized as in Def.~\ref{definition:weight_init} with element-wise distribution $\P$. Then we have for linear networks, which follow recursion \eqref{eq:linear_net}, that:
\begin{itemize}
    \item[a.] If $\P$ is standard Gaussian, then the sequence $\{ \widetilde{H}_\ell \}$ converges to a rank one matrix. 
    \item[b.] If $\P$ is uniform$[-\sqrt{3},\sqrt{3}]$, then there exists a monotonically increasing sequence of integers $\ell_1 < \ell_2, \dots$ such that the sequence $\{\widetilde{H}_{\ell_k}\}$ converges to a rank one matrix. 
\end{itemize}

\end{lemma}
 This results stands in striking contrast to the result of Theorem~\ref{thm:linear-rank-lowerboun} established for batch-normalized networks.\footnote{According to the observations in Fig. \ref{fig:vanilla_bn_rank}, the result of part b holds for the usual sequence of indices $\{ \ell_k = k \}$, which indicates that $\{\widetilde{H}_{k}\}$ converges to a rank one matrix even in the case of uniform initialization.} 
 Interestingly, even residual skip connections cannot avoid rank collapse for very deep neural networks, unless one is willing to incorporate a depth dependent down-scaling of the parametric branch as for example done in \cite{allen2018convergence}, who set $\gamma = O(\frac{1}{L})$ . Remarkably, Theorem~\ref{thm:linear-rank-lowerboun} shows that \bn~layers provably avoid rank collapse \textit{without} requiring the networks to become closer and closer to identity.

\paragraph{Implications of rank collapse on gradient based learning.}
In order to explain the severe consequence of rank collapse on optimization performance reported in Fig. \ref{fig:learning_depth}, we study the effect of rank one hidden-layer representations on the gradient of the training loss for distinct input samples. Let $\mathcal{L}_i$ denote the training loss for datapoint $i$ on a vanilla network as in Eq. (\ref{eq:linear_net}). Furthermore, let the final classification layer be parametrized by $W_{L+1}\in \mathbb{R}^{d_{out}\times d}$. Then, given that the hidden presentation at the last hidden layer $L$ is rank one, the normalized gradients of the loss with respect to weights of individual neurons $k\in 1,...,d_{out}$ in the classification layer (denoted by $\nabla_{W_{L+1,k}}\mathcal{L}_i$, where $\|\nabla_{W_{L+1,k}}\mathcal{L}_i\|=1$) are collinear for any two datapoints $i$ and $j$, i.e.  $\nabla_{W_{L+1,k}}\mathcal{L}_i=\mp \nabla_{W_{L+1,k}}\mathcal{L}_j$. A formal statement is presented in Prop.~\ref{prop:grad_direction_formal} in the Appendix alongside empirical validations on a VGG19 network (Fig. \ref{fig:vgg}). This result implies that the commonly accepted vanishing gradient \textit{norm} hypothesis is not descriptive enough since SGD does not take small steps into the \textit{right} direction, but into a random direction that is \textit{independent} from the input. In other words, deep neural networks are prone to \textit{directional gradient vanishing} after initialization, which is caused by the collapse of the last hidden layer activations to a very small subspace (one line in $\mathbb{R}^d$ in the extreme case of rank one activations).

\section{The important role of the rank}\label{sec:pretraining}
The preceding sections highlight that the rank of the hidden representations is a key difference between random vanilla and \bn~ networks. We now provide three experimental findings that substantiate the particular importance of the rank at the beginning of training: First, we find that an unsupervised, rank-increasing pre-training allows SGD on vanilla networks to outperform \bn{} networks. Second, we show that the performance of \bn{}-networks is closely tied to a high rank at initialization. Third, we report that SGD updates preserve the initial rank magnitude throughout the optimization process.
\paragraph{Outperforming BN using a pre-training step.} 
As discussed above, batch normalization layers are very effective at avoiding rank collapse. Yet, this is of course not the only way to preserve rank. Based upon our theoretical insights, we leverage the lower bound established in Eq.~\eqref{eq:rH} to design a pre-training step that not only avoids rank collapse but also accelerates the convergence of SGD. Our proposed procedure is both simple and computationally cheap. Specifically, we \emph{maximize} the lower-bound $r(H_\ell)$ (in Eq.~\eqref{eq:rH}) on the rank of the hidden presentation $H_\ell$ in each layer $\ell$. Since this function is differentiable with respect to its input, it can be optimized sufficiently by just a few steps of (stochastic) gradient ascent (see Section \ref{sec:details_pretraining} in the Appendix for more details).

\begin{figure}[h!]
    \centering
    \begin{tabular}{c@{}c@{}c@{}}
     \includegraphics[width=0.3333\textwidth]{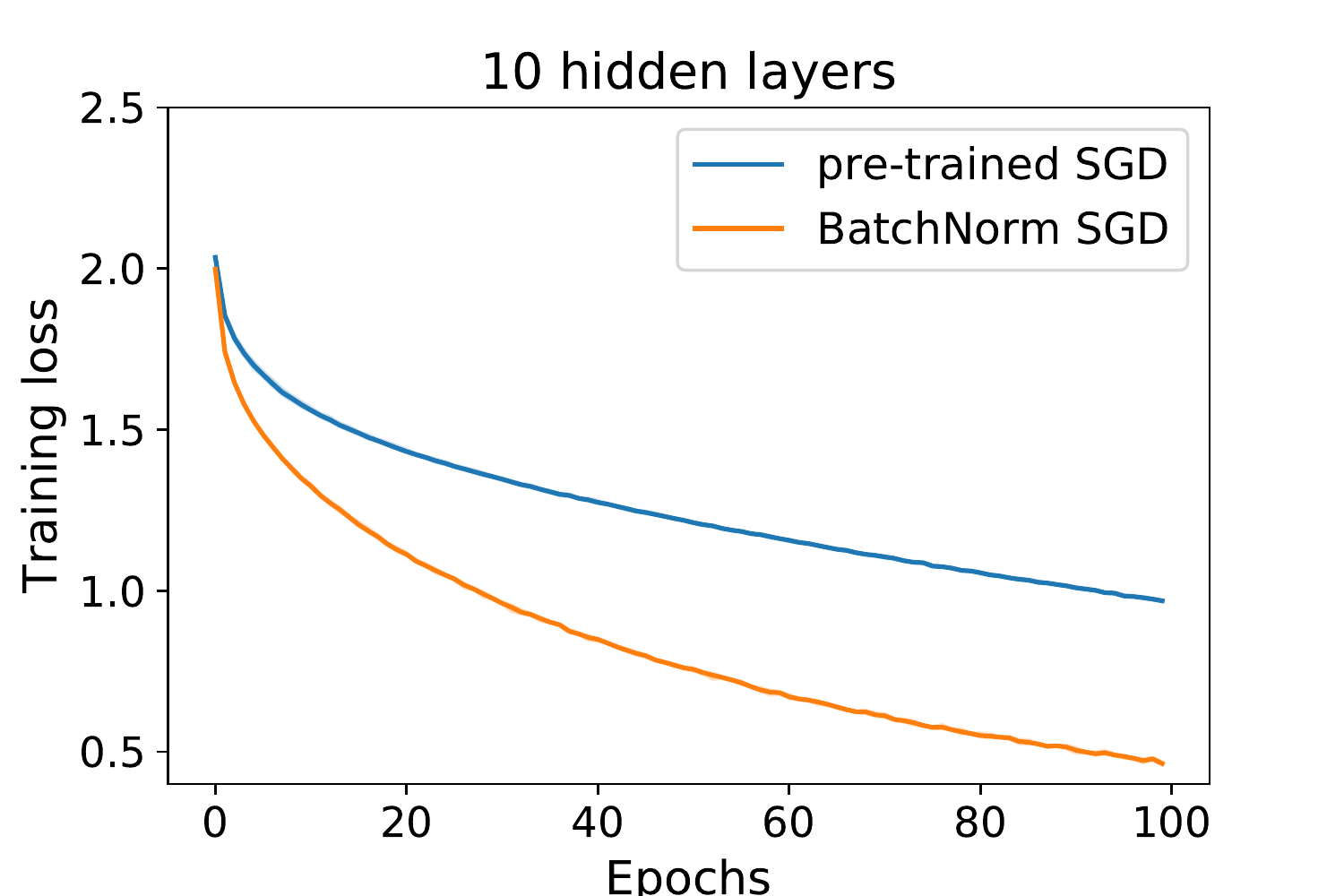} & \includegraphics[width=0.3333\textwidth]{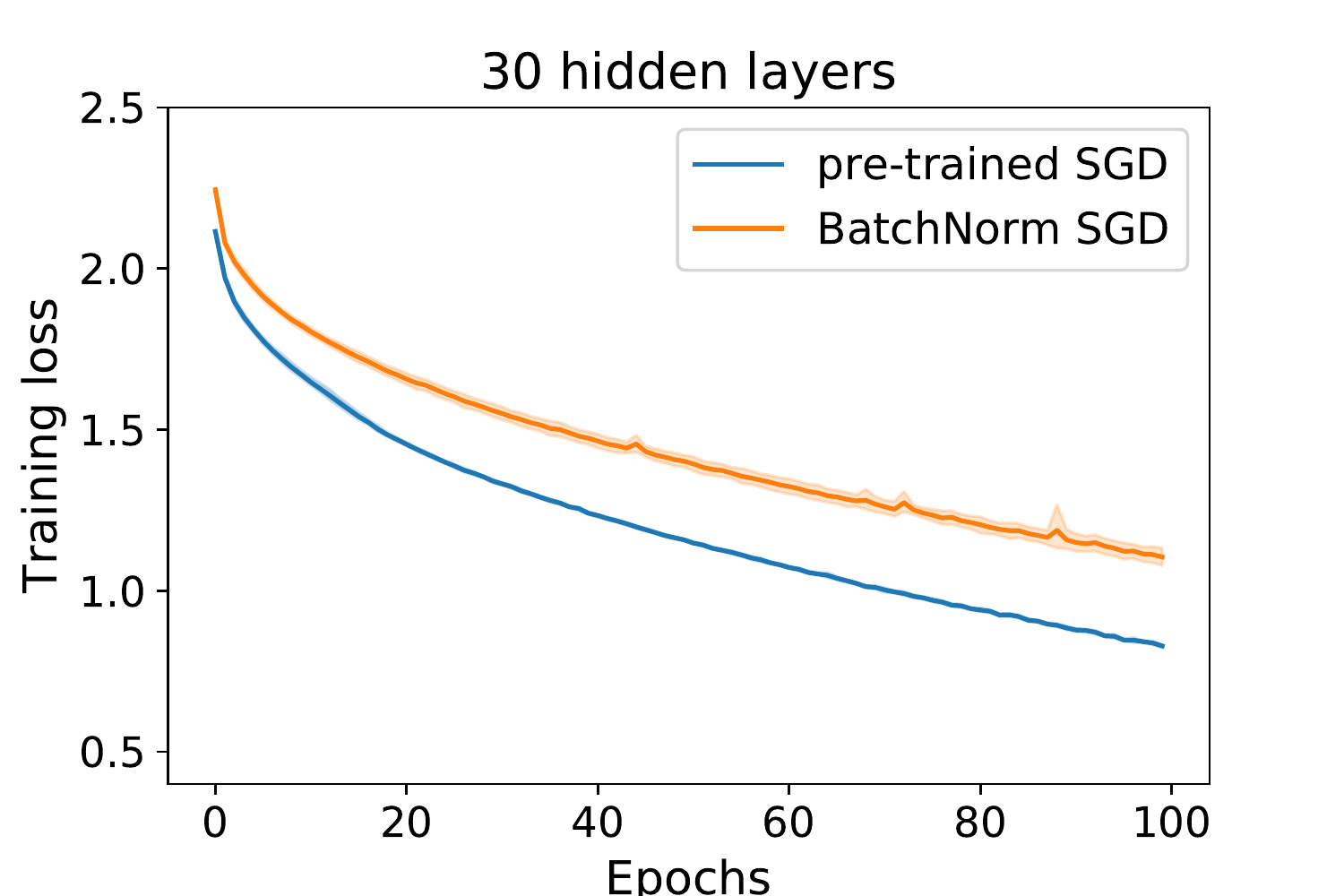} & \includegraphics[width=0.3333\textwidth]{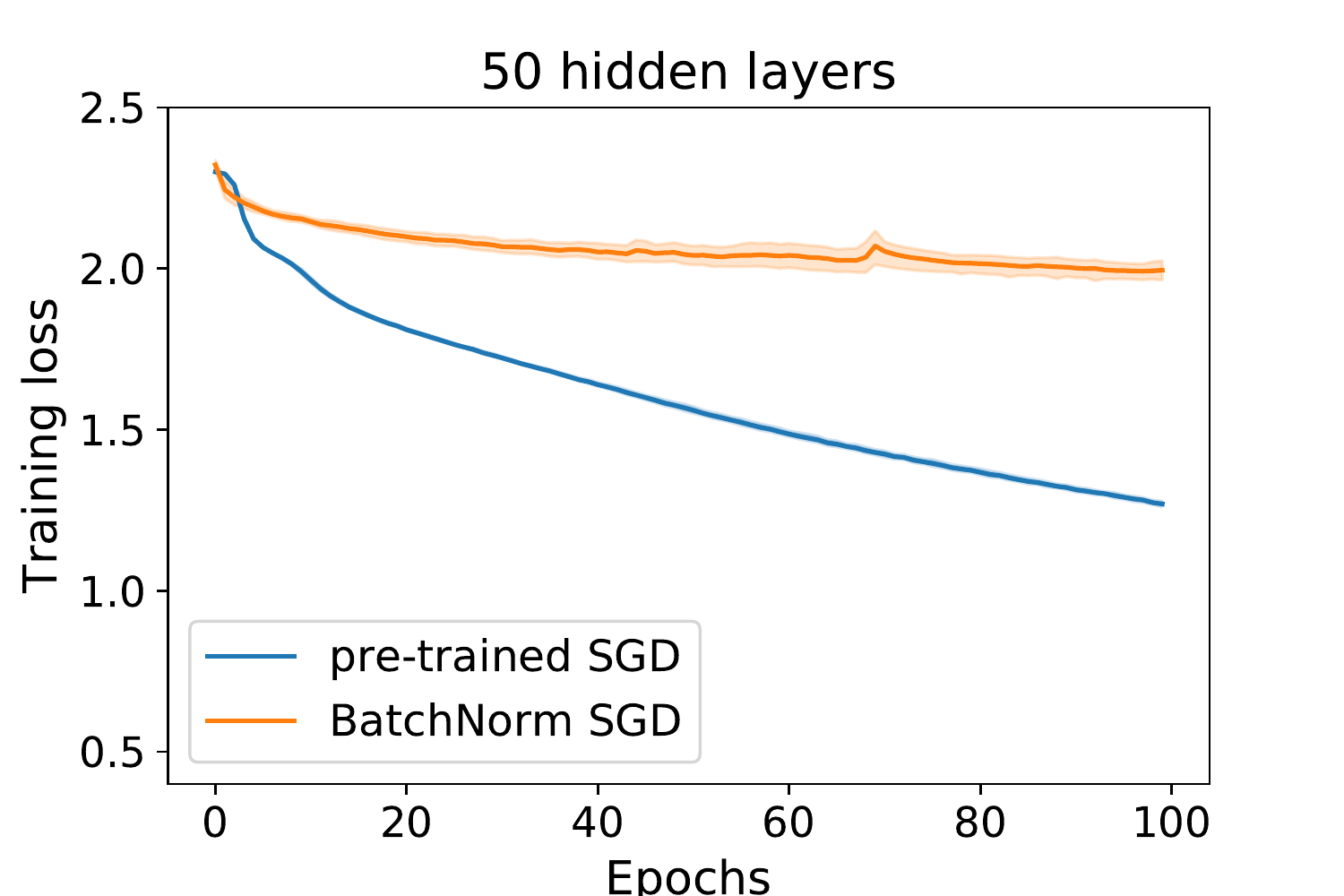} 
    \end{tabular}
    
    \caption{\footnotesize{\textbf{Pre-training versus \bn{}:} Loss over epochs on CIFAR-10 for MLPs of increasing depth with 128 hidden units and ReLU activation. Trained with SGD (batchsize 64) and grid-searched stepsize. See Fig. \ref{fig:beating_bn_app} for the corresponding test loss and accuracy as well as Fig. \ref{fig:beating_bn_FM} for FashionMNIST results.}}
    \label{fig:outperform_bn_cifar}
\end{figure}

Fig.~\ref{fig:outperform_bn_cifar} compares the convergence rate of SGD on pre-trained vanilla networks and \bn{}-networks. As can be seen, the slow down in depth is much less severe for the pre-trained networks. This improvement is, also, reflected both in terms of training accuracy and test loss (see Fig.~\ref{fig:beating_bn_app} in Appendix). Interestingly, the pre-training is not only faster than \bn~ on deep networks, but it is also straight-forward to use in settings where the application of \bn~ is rather cumbersome such as for very small batch sizes or on unseen data \cite{ioffe2017batch,wu2018group}.

\paragraph{Breaking batch normalization.}

Some scholars hypothesize that the effectiveness of \bn~ stems from a global landscape smoothing \cite{santurkar2018does} or a certain learning rate tuning \cite{arora2018theoretical}, that are thought to be induced by the normalization. Under these hypotheses, one would expect that SGD converges fast on \bn{}-nets \textit{regardless} of the initialization. Yet, we here show that the way that networks are initialized does play a crucial role for the subsequent optimization performance of \bn{}-nets.

\begin{figure}[H]
 \centering 
          \begin{tabular}{c@{}c@{}}
          \includegraphics[width=0.38\linewidth]{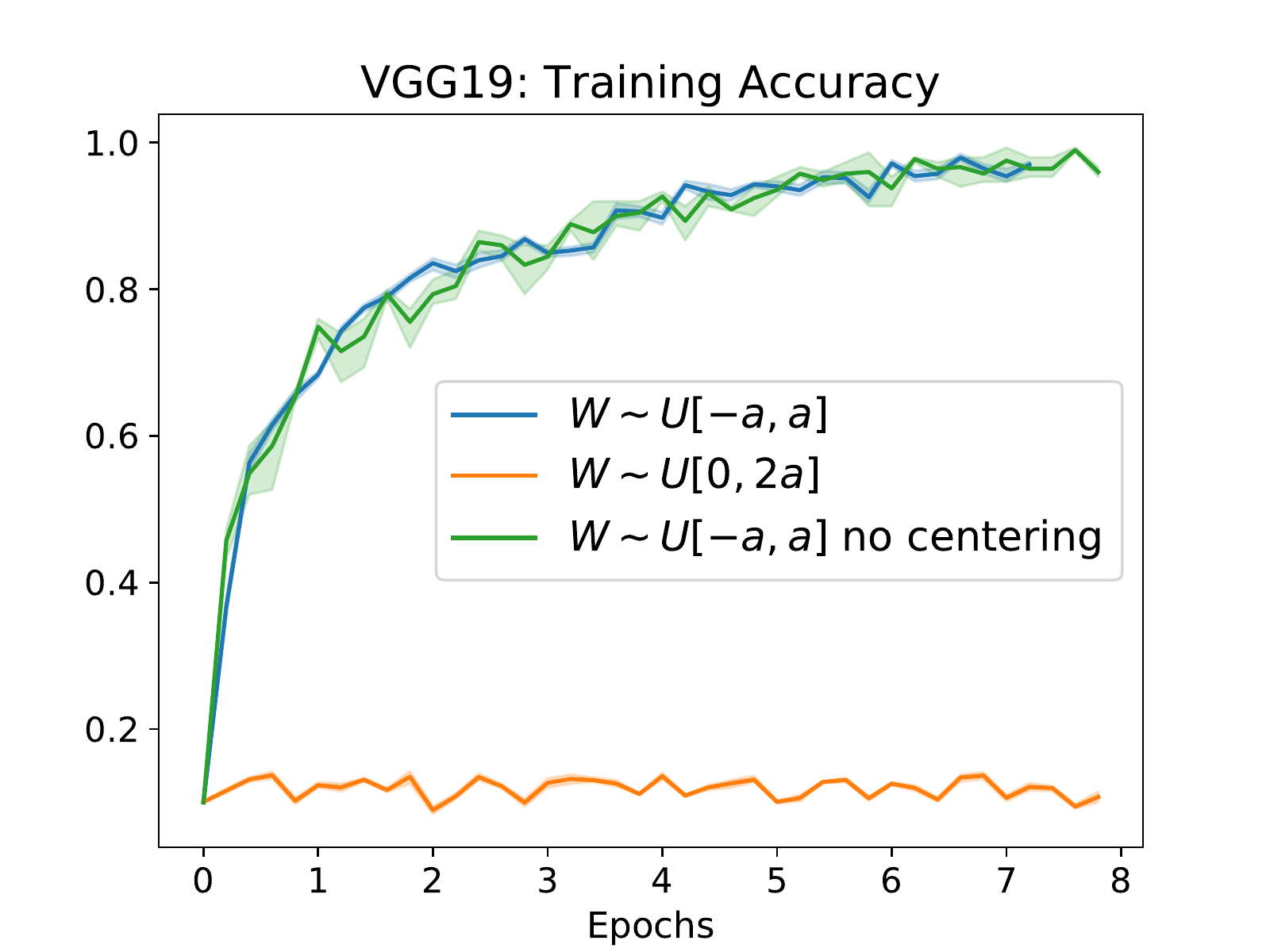}
                &   \includegraphics[width=0.38\linewidth]{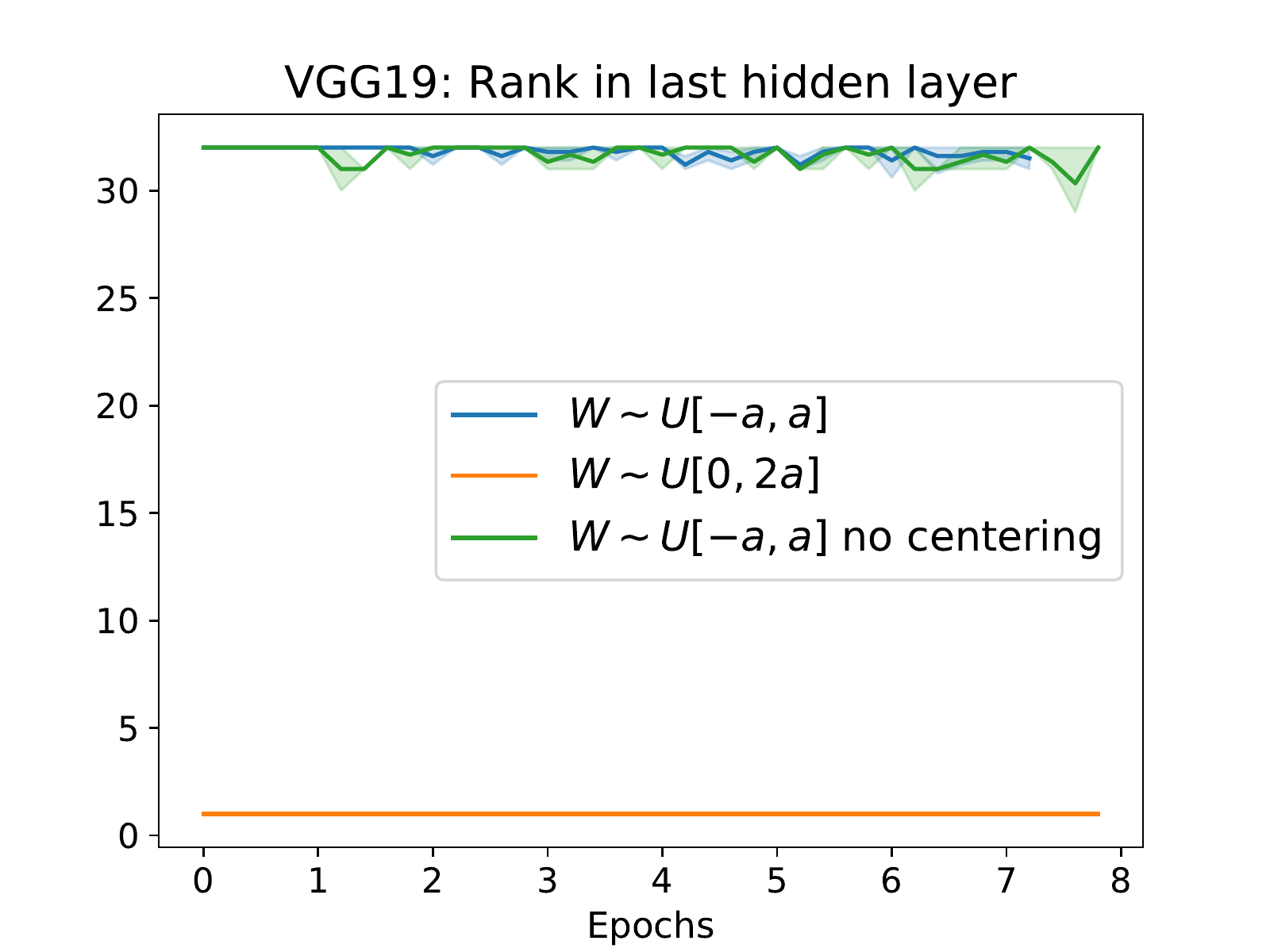}
           
 	  \end{tabular}
          \caption{ \footnotesize{\textbf{Breaking Batchnorm:} CIFAR-10 on VGG19 with standard PyTorch initialization as well as a uniform initialization of same  variance. (Left) training accuracy, (Right) Rank of last  hidden layer computed using $torch.matrix\_rank()$. Plot also shows results for standard initialization and \bn~\textit{without} mean deduction. Avg. and 95\% CI of 5 independent runs. (See Fig. \ref{fig:broken_bn_app} in Appendix for similar results on ResNet-50).
          }}
          \label{fig:broken_bn}
\end{figure}

Particularly, we train two MLPs with batchnorm, but change the initialization for the second net from the standard PyTorch way $W_{l,i,j}\sim \text{uniform}\left[-1/\sqrt{d_l},1/\sqrt{d_l}\right]$ \cite{paszke2019pytorch,glorot2010understanding} to $W_{l,i,j}\sim \text{uniform}\left[0,+2/\sqrt{d_l}\right]$, where $d_l$ is the layer size. As can be seen to the right, this small change reduces the rank preserving quality of BN significantly, which is reflected in much slower learning behaviour. Even sophisticated modern day architectures such as VGG and ResNet networks are unable to fit the CIFAR-10 dataset after changing the initialization in this way (see Fig.~\ref{fig:broken_bn}).

\paragraph{Rank through the optimization process.}\label{sec:rank_causality}
The theoretical result of Theorem \ref{thm:linear-rank-lowerboun} considers the rank at random initialization. To conclude, we perform two further experiments which confirm that the initial rank strongly influences the speed of SGD throughout the entire optimization process. In this regard,  Fig.~\ref{fig:pre_training} reports that SGD preserves the initial magnitude of the rank to a large extent, regardless of the specific network type. This is particularly obvious when comparing the two \bn{} initializations. A further noteworthy aspect is the clear correlation between the level of pre-training and optimization performance on vanilla nets.
Interestingly, this result does again not only hold on simple MLPs but also generalizes to modern day networks such as the VGG-19 (see Fig.~\ref{fig:broken_bn}) and  ResNet50 architecture (see Appendix~\ref{sec:additional_experiments}).

\begin{figure}[h!]
    \centering
    \begin{tabular}{c@{}c@{}c@{}}
     \includegraphics[width=0.4\textwidth]{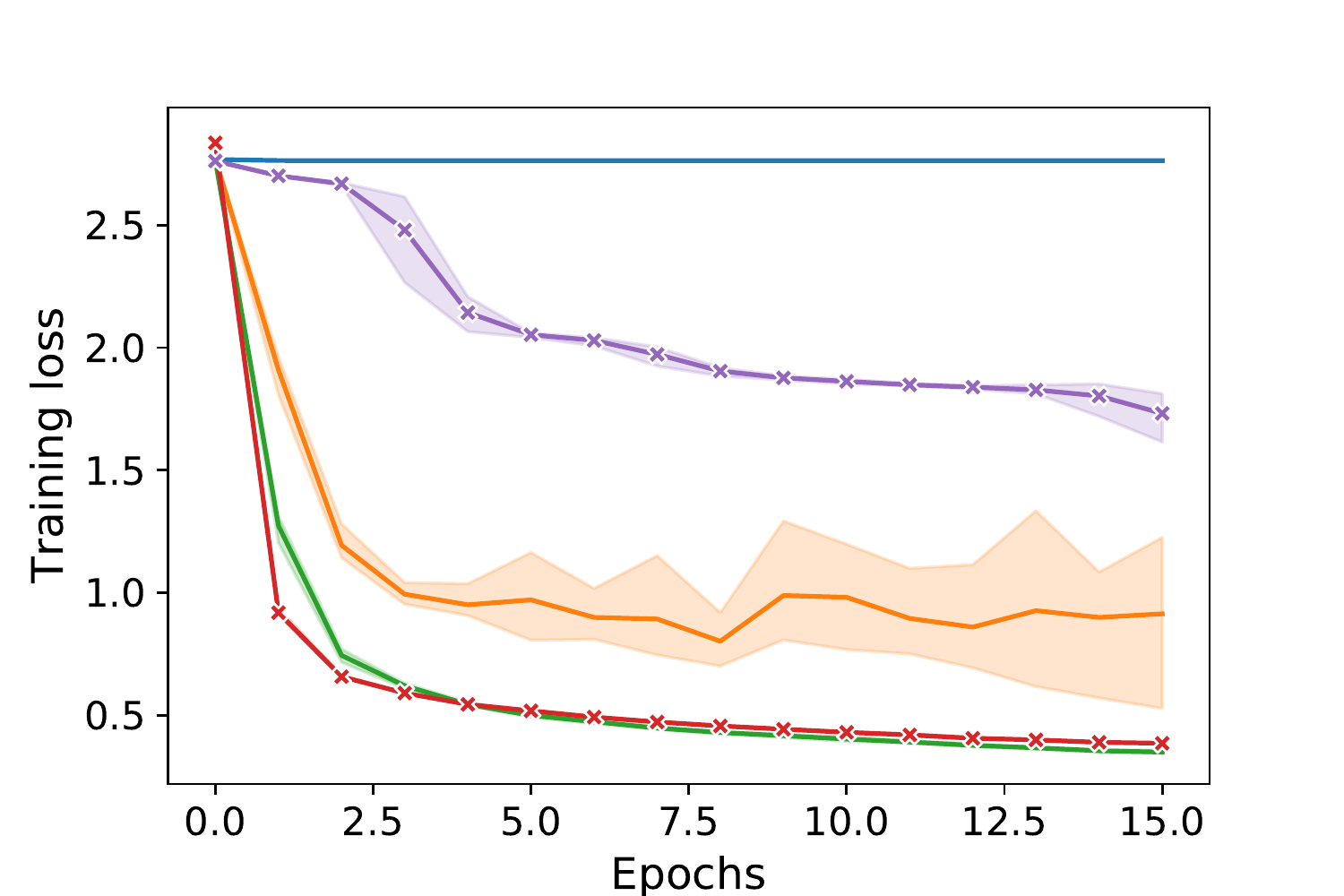} & \hspace{-10pt} \includegraphics[width=0.4\textwidth]{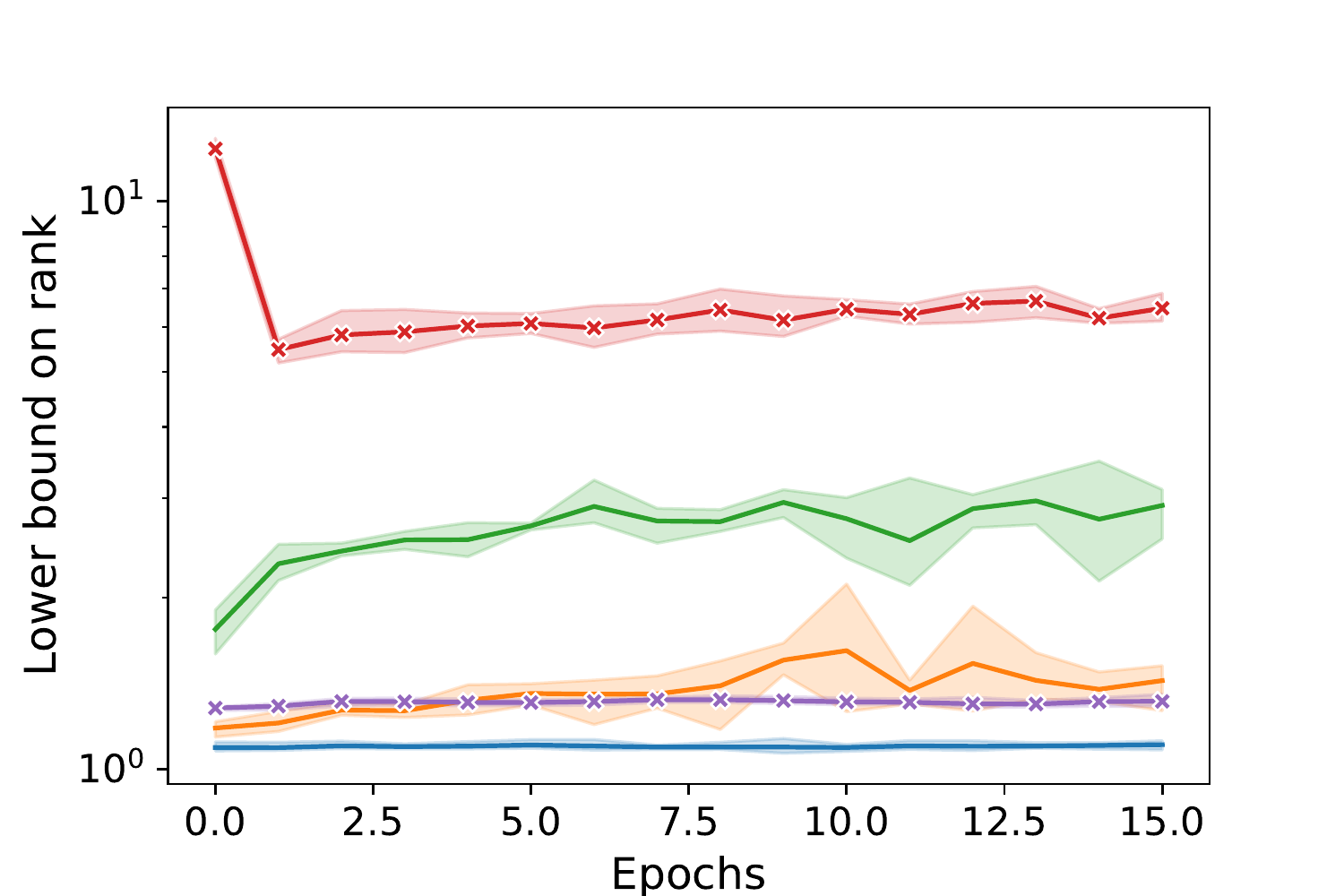} & \includegraphics[width=0.2\textwidth]{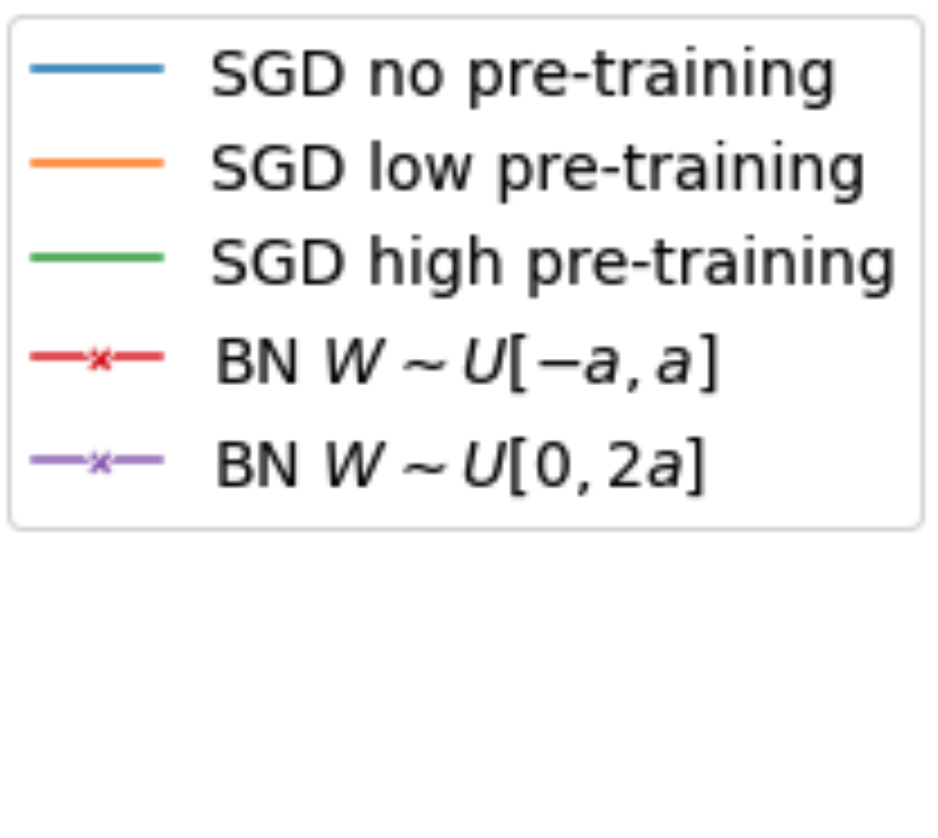} 
    \end{tabular}
    
    \caption{\footnotesize{\textbf{Pretraining:} Fashion-MNIST on MLPs of depth 32 and width 128. (Left) Training accuracy, (Right) Lower bound on rank. Blue line is a ReLU network with standard initialization. Other solid lines are pre-trained layer-wise with 25 (orange) and 75 (green) iterations to increase the rank. Dashed lines are batchnorm networks with standard and asymmetric initialization. Average and 95\% confidence interval of 5 independent runs.}}
    \label{fig:pre_training}
\end{figure}

\bibliographystyle{plain}
\bibliography{bn.bib}

\begin{thebibliography}{10}

\bibitem{allen2018convergence}
Zeyuan Allen-Zhu, Yuanzhi Li, and Zhao Song.
\newblock A convergence theory for deep learning via over-parameterization.
\newblock {\em arXiv preprint arXiv:1811.03962}, 2018.

\bibitem{arora2018convergence}
Sanjeev Arora, Nadav Cohen, Noah Golowich, and Wei Hu.
\newblock A convergence analysis of gradient descent for deep linear neural
  networks.
\newblock {\em arXiv preprint arXiv:1810.02281}, 2018.

\bibitem{arora2018theoretical}
Sanjeev Arora, Zhiyuan Li, and Kaifeng Lyu.
\newblock Theoretical analysis of auto rate-tuning by batch normalization.
\newblock {\em arXiv preprint arXiv:1812.03981}, 2018.

\bibitem{barak2016proofs}
Boaz Barak and David Steurer.
\newblock Proofs, beliefs, and algorithms through the lens of sum-of-squares.
\newblock {\em Course notes: http://www. sumofsquares. org/public/index. html},
  2016.

\bibitem{bartlett2019gradient}
Peter~L. Bartlett, David~P. Helmbold, and Philip~M. Long.
\newblock Gradient descent with identity initialization efficiently learns
  positive-definite linear transformations by deep residual networks.
\newblock {\em Neural computation}, 31(3):477--502, 2019.

\bibitem{bjorck2018understanding}
Nils Bjorck, Carla~P. Gomes, Bart Selman, and Kilian~Q. Weinberger.
\newblock Understanding batch normalization, 2018.

\bibitem{boucheron2013concentration}
St{\'e}phane Boucheron, G{\'a}bor Lugosi, and Pascal Massart.
\newblock {\em Concentration inequalities: A nonasymptotic theory of
  independence}.
\newblock Oxford university press, 2013.

\bibitem{bougerol2012products}
Philippe Bougerol.
\newblock {\em Products of Random Matrices with Applications to Schr{\"o}dinger
  Operators}, volume~8.
\newblock Springer Science \& Business Media, 2012.

\bibitem{dieuleveut2017bridging}
Aymeric Dieuleveut, Alain Durmus, and Francis Bach.
\newblock Bridging the gap between constant step size stochastic gradient
  descent and {M}arkov chains.
\newblock {\em arXiv preprint arXiv:1707.06386}, 2017.

\bibitem{douc2018markov}
Randal Douc, Eric Moulines, Pierre Priouret, and Philippe Soulier.
\newblock {\em Markov Chains}.
\newblock Springer, 2018.

\bibitem{forrester2013lyapunov}
Peter~J. Forrester.
\newblock Lyapunov exponents for products of complex {G}aussian random
  matrices.
\newblock {\em Journal of Statistical Physics}, 151(5):796--808, 2013.

\bibitem{glorot2010understanding}
Xavier Glorot and Yoshua Bengio.
\newblock Understanding the difficulty of training deep feedforward neural
  networks.
\newblock In {\em Proceedings of the thirteenth international conference on
  artificial intelligence and statistics}, pages 249--256, 2010.

\bibitem{he2015delving}
Kaiming He, Xiangyu Zhang, Shaoqing Ren, and Jian Sun.
\newblock Delving deep into rectifiers: Surpassing human-level performance on
  imagenet classification.
\newblock In {\em Proceedings of the IEEE international conference on computer
  vision}, pages 1026--1034, 2015.

\bibitem{hochreiter1998vanishing}
Sepp Hochreiter.
\newblock The vanishing gradient problem during learning recurrent neural nets
  and problem solutions.
\newblock {\em International Journal of Uncertainty, Fuzziness and
  Knowledge-Based Systems}, 6(02):107--116, 1998.

\bibitem{ioffe2017batch}
Sergey Ioffe.
\newblock Batch renormalization: Towards reducing minibatch dependence in
  batch-normalized models.
\newblock In {\em Advances in neural information processing systems}, pages
  1945--1953, 2017.

\bibitem{ioffe2015batch}
Sergey Ioffe and Christian Szegedy.
\newblock Batch normalization: Accelerating deep network training by reducing
  internal covariate shift.
\newblock {\em arXiv preprint arXiv:1502.03167}, 2015.

\bibitem{karakida2019normalization}
Ryo Karakida, Shotaro Akaho, and Shun-ichi Amari.
\newblock The normalization method for alleviating pathological sharpness in
  wide neural networks.
\newblock In {\em Advances in Neural Information Processing Systems}, pages
  6403--6413, 2019.

\bibitem{kawaguchi2016deep}
Kenji Kawaguchi.
\newblock Deep learning without poor local minima.
\newblock In {\em Advances in neural information processing systems}, pages
  586--594, 2016.

\bibitem{kohler2018exponential}
Jonas Kohler, Hadi Daneshmand, Aurelien Lucchi, Ming Zhou, Klaus Neymeyr, and
  Thomas Hofmann.
\newblock Exponential convergence rates for batch normalization: The power of
  length-direction decoupling in non-convex optimization.
\newblock {\em arXiv preprint arXiv:1805.10694}, 2018.

\bibitem{liu2016bulk}
Dang-Zheng Liu, Dong Wang, and Lun Zhang.
\newblock Bulk and soft-edge universality for singular values of products of
  ginibre random matrices.
\newblock In {\em Annales de l'Institut Henri Poincar{\'e}, Probabilit{\'e}s et
  Statistiques}, volume~52, pages 1734--1762. Institut Henri Poincar{\'e},
  2016.

\bibitem{paszke2019pytorch}
Adam Paszke, Sam Gross, Francisco Massa, Adam Lerer, James Bradbury, Gregory
  Chanan, Trevor Killeen, Zeming Lin, Natalia Gimelshein, Luca Antiga, et~al.
\newblock Pytorch: An imperative style, high-performance deep learning library.
\newblock In {\em Advances in Neural Information Processing Systems}, pages
  8024--8035, 2019.

\bibitem{pennington2018emergence}
Jeffrey Pennington, Samuel~S Schoenholz, and Surya Ganguli.
\newblock The emergence of spectral universality in deep networks.
\newblock {\em arXiv preprint arXiv:1802.09979}, 2018.

\bibitem{salimans2016weight}
Tim Salimans and Diederik~P. Kingma.
\newblock Weight normalization: A simple reparameterization to accelerate
  training of deep neural networks.
\newblock In {\em Advances in Neural Information Processing Systems}, pages
  901--909, 2016.

\bibitem{santurkar2018does}
Shibani Santurkar, Dimitris Tsipras, Andrew Ilyas, and Aleksander Madry.
\newblock How does batch normalization help optimization?(no, it is not about
  internal covariate shift).
\newblock {\em arXiv preprint arXiv:1805.11604}, 2018.

\bibitem{saxe2013exact}
Andrew~M. Saxe, James~L. McClelland, and Surya Ganguli.
\newblock Exact solutions to the nonlinear dynamics of learning in deep linear
  neural networks.
\newblock {\em arXiv preprint arXiv:1312.6120}, 2013.

\bibitem{schoenholz2016deep}
Samuel~S. Schoenholz, Justin Gilmer, Surya Ganguli, and Jascha Sohl-Dickstein.
\newblock Deep information propagation.
\newblock {\em arXiv preprint arXiv:1611.01232}, 2016.

\bibitem{telgarsky2016benefits}
Matus Telgarsky.
\newblock Benefits of depth in neural networks.
\newblock {\em arXiv preprint arXiv:1602.04485}, 2016.

\bibitem{tropp2015introduction}
Joel~A. Tropp.
\newblock An introduction to matrix concentration inequalities.
\newblock {\em arXiv preprint arXiv:1501.01571}, 2015.

\bibitem{wu2018group}
Yuxin Wu and Kaiming He.
\newblock Group normalization.
\newblock In {\em Proceedings of the European Conference on Computer Vision
  (ECCV)}, pages 3--19, 2018.

\bibitem{xiao2017fashion}
Han Xiao, Kashif Rasul, and Roland Vollgraf.
\newblock Fashion-mnist: a novel image dataset for benchmarking machine
  learning algorithms.
\newblock {\em arXiv preprint arXiv:1708.07747}, 2017.

\bibitem{yang2017mean}
Ge~Yang and Samuel Schoenholz.
\newblock Mean field residual networks: On the edge of chaos.
\newblock In {\em Advances in neural information processing systems}, pages
  7103--7114, 2017.

\bibitem{yang2019mean}
Greg Yang, Jeffrey Pennington, Vinay Rao, Jascha Sohl-Dickstein, and Samuel~S.
  Schoenholz.
\newblock A mean field theory of batch normalization.
\newblock {\em arXiv preprint arXiv:1902.08129}, 2019.

\bibitem{yao2019pyhessian}
Zhewei Yao, Amir Gholami, Kurt Keutzer, and Michael Mahoney.
\newblock Py{H}essian: Neural networks through the lens of the {H}essian.
\newblock {\em arXiv preprint arXiv:1912.07145}, 2019.

\end{thebibliography}
\newpage 
\appendix

\begin{center}
   \textbf{\LARGE{Appendix}}
\end{center}

\section{Preliminaries}
Recall that $H_{\ell}^{(\gamma)}$ denotes the hidden representations in layer $\ell$. These matrices make a Markov chain that obeys the recurrence of Eq.~\eqref{eq:nn}, which we restate here
\begin{align} \label{eq:chain_app}
    H_{\ell+1}^{(\gamma)} = \bn{}(H_\ell^{(\gamma)} + \gamma W_\ell H_\ell^{(\gamma)}), \quad  H_0^{\gamma} = X,
\end{align}
where we use the compact notation $\bn$ for $\bn_{0,\ones_d}$.
Let $M_{\ell}^{(\gamma)}$ be second moment matrix of the hidden representations $H_{\ell}^{(\gamma)}$, i.e.  $ M_\ell^{(\gamma)}:= H_{\ell}^{(\gamma)} \left( H_{\ell}^{(\gamma)}\right)^{\top}/N$. Batch normalization ensures that the rows of $H_{\ell}$ have the same norm $\sqrt{N}$ for $\ell>0$ --where $N$ is the size of mini-batch. Let $\Hs$ be space of $d\times d$-matrices that obey this propery. This property enforces two key characteristics on $M_{\ell}^{(\gamma)}$: 
\begin{align}
    \text{(p.1)} & \quad \quad \text{its diagonal elements are one}  \label{p1m}\\ 
    \text{(p.2)} & \quad \quad \text{the absolute value of its off-diagonal elements is less than one} \label{p2m}
\end{align}
 Property (p.1) directly yields that the trace of $M_\ell^{(\gamma)}$ (and hence the sum of its eigenvalues) is equal to $d$. We will repeatedly use these properties in our analysis.

Furthermore, the sequence $\{H_{\ell}^{(\gamma)}\}_{\ell=1}^\infty$ constitute a Markov chain. Under mild assumptions, this chain admits an invariant distribution that is defined bellow\cite{douc2018markov}. 
\begin{definition} \label{definition:invariance}
Distribution $\nu$ is an invariant distribution of the hidden representations $\{H_{\ell}^{(\gamma)}\}_{\ell=1}^\infty$ if it obeys
\begin{align}
    \int \bn(H+\gamma W H) \mu(d W) \nu(d H) = \int \bn(H) \nu(dH)
\end{align}
where $\mu$ denotes the probability measure of random weights. 
\end{definition}
Later, we will see that the above invariance property allows us to determine the commutative behaviour of the sequence of hidden presentations.

\section{Lower bounds on the (soft) rank}
Recall that we introduced the ratio $r(H)= \tr(M(H))^2/\| M(H)\|_F^2$ in Eq.~\eqref{eq:rH} as a lower bound on both the $\rank(H)$ as well as the soft rank $\rank_\tau(H)$ (stated in Lemma~\ref{lemma:soft_rank_lowerbound}). This section establishes these lower bounds.  
\begin{proof}[Proof of Lemma~\ref{lemma:soft_rank_lowerbound} (part 1).]
We first prove that $\rank(H) \geq r(H)$. 
Let $M := M(H) = H H^\top/N$. Since the eigenvalues of $H$ are obtained by a constant scaling factor of the squared singular values of $H$, these two matrices have the same rank. We now establish a lower bound on $\rank(M)$. 
Let $\blambda \in \R^d$ contain the eigenvalues of matrix $\M$, hence $\|\blambda\|_1= \tr(\M)$ and $\| \lambda \|_2^2 = \| M \|_F^2$. Given $\lambda$, we define the vector $\w \in \R^d$ as
\begin{align}
    \w_i = \begin{cases} 
     1/\|\lambda\|_0 & : \lambda_i \neq 0 \\ 
     0 & : \lambda_i = 0.
    \end{cases}
\end{align}
To proof the assertion, we make use if a straightforward application of Cauchy-Schwartz
\begin{align}
    & | \langle \lambda, \w \rangle | \leq \| \lambda\|_2 \| \w \|_2 \\ 
    \implies&  \| \lambda \|_1/ \| \lambda\|_0 \leq \| \lambda\|_2/\| \lambda \|_0^{1/2} \\ 
    \implies&  \frac{\|\lambda\|_1}{\|\lambda\|_2} \leq \| \lambda\|_0^{1/2}.
\end{align}
Replacing $\| \lambda \|_2 = \| \M \|_F$ and $\| \lambda \|_1 = \tr(\M)$ into the above equation concludes the result. Note that the above proof technique has been used in the planted sparse vector problem~\cite{barak2016proofs}. 
\end{proof}

\begin{proof}[Proof of Lemma~\ref{lemma:soft_rank_lowerbound} (part 2)]
Now, we prove that $\rank_\tau(H_\ell^{(\gamma)}) \geq (1-\tau)^2 r(H_\ell^{(\gamma)}) $. Let $\lambda \in \R^d_+$ be a vector containing the eigenvalues of the matrix $M_\ell^{(\gamma)} = M(H_\ell^{(\gamma)})$. Let $\sigma \in \R^d_+$ contain the singular values of $H$. Then, one can readily check that $\sigma_i^2/N = \lambda_i$. Furthermore, $\| \lambda \|_1 = d$ due to (p.1) in Eq.~\eqref{p1m}. Furthermore, we have by definition that
\begin{align}
    \rank_\tau(H_\ell^{(\gamma)}) = h_\tau(\lambda) :=  \sum_{i=1}^d \ones(\sigma_i^2/N \geq \tau) =  \sum_{i=1}^d \ones(\lambda_i \geq \tau).
\end{align}
Let us now define a vector $w \in \R^d$ with entries
\begin{align}
    \w_i = \begin{cases} 
     1/h_\tau(\lambda) & : \lambda_i \geq \tau \\ 
     0 & : \text{otherwise.} 
    \end{cases}
\end{align}
Then, we use Cauchy-Schwartz to get 
\begin{align} \label{eq:chauchy_2}
    | \langle \lambda, w \rangle | \leq \| \lambda \|_2 \| w \|_2.
\end{align}
It is easy to check that $\| \w \|_2 = h_\tau(\lambda)^{-\sfrac{1}{2}}$ holds. Furthermore, 
\begin{align}
    h_\tau(\lambda) |\langle w, \lambda \rangle| & = \sum_{|\lambda_i| \geq \tau}^d |\lambda_i|  \\ 
    & \geq \| \lambda \|_1 - d \tau \\ 
    & \geq (1-\tau) \| \lambda \|_1,
\end{align}
where we used the fact that $\| \lambda \|_1 =d$ in the last inequality.
Replacing this result into the bound of Eq.~\eqref{eq:chauchy_2} yields 
\begin{align} 
 \rank_\tau(H_\ell^{(\gamma)}) = h_\tau(\lambda) \geq (1-\tau)^2 \| \lambda \|_1^2/ \| \lambda \|_2^2 = (1-\tau)^2 r(H_\ell^{(\gamma)}),
\end{align}
which conludes the proof.
\end{proof}

\section{Initialization consequences} \label{section:init_consequences_appendix}
The particular weight initialization scheme consider through out this work (recall Def.~\ref{definition:weight_init}), imposes an interesting structure in the invariant distribution of the sequence of hidden presentations (defined in Def.~\ref{definition:invariance}). 

\begin{lemma} \label{lemma:symmetricity}
Suppose that the chain $\{H_{\ell}^{(\gamma)}\}_{\ell=1}^\infty$ (defined in Eq.~\ref{eq:chain_app}) admits a unique invariant distribution $\nu_\gamma$ and $H$ is drawn from $\nu_\gamma$; then, the law of $H_{i:}$ equates the law of $-H_{i:}$ where $H_{i:}$ denotes the $i$th row of matrix $H$. 
\end{lemma}
\begin{proof}
Let $S$ be a sign filliping matrix: it is diagonal and its diagonal elements are in $\{+1,-1\}$.   Then $S \W \stackrel{d}{=}\W$ holds for a random matrix $\W$ whose elements are drawn i.i.d. from a symmetric distribution. Let $H$ be drawn from the invariant distribution of the chain denoted by $\nu_\gamma$; Leveraging the invariance property, we get
\begin{align}
   H \stackrel{d}{=}  H_{+} \stackrel{d}{=}  \left(\diag 
    (H_{\sfrac{1}{2}}H_{\sfrac{1}{2}}^\top/N)\right)^{-\sfrac{1}{2}} H_{\sfrac{1}{2}}, \quad H_{\sfrac{1}{2}}:= H + \gamma S \W S H
\end{align}
By multiplying both sides with $S$, we get 
\begin{align} \label{eq:SQplus}
    SH \stackrel{d}{=} S H_+ \stackrel{d}{=} \left(\diag \left( H_{\sfrac{1}{2}} H_{\sfrac{1}{2}}^\top/N \right) \right)^{-\sfrac{1}{2}} \widetilde{H}_{\sfrac{1}{2}}, \quad  \widetilde{H}_{\sfrac{1}{2}} := S H + \gamma  W  S H
\end{align}
Note that we use the fact that diagonal matrices commute in the above derivation.  According to the definition, $S^2 = I$ holds. Considering this fact, we get
\begin{align}
   \diag\left( H_{\sfrac{1}{2}}  H_{\sfrac{1}{2}}^\top\right) & = \diag\left( \left( H + \gamma S \W S H \right) \left(H + \gamma S \W S H \right)^\top \right)  \\
    & = \diag \left(\left( S SH + \gamma S \W S H \right) \left(S SH + \gamma S \W S H \right)^\top \right)\\ 
    & = \diag\left( S\left( SH + \gamma \W S H \right) \left( SH + \gamma \W S H \right)^\top S \right) \\ 
    & = \diag\left(\left( SH + \gamma \W S H \right) \left( SH + \gamma \W S H \right)^\top \right) \\ 
    & = \widetilde{H}_{\sfrac{1}{2}} \widetilde{H}_{\sfrac{1}{2}}^\top
\end{align}
Replacing the above result into Eq.~\eqref{eq:SQplus} yields 
\begin{align} \label{eq:SQplus}
    S H \stackrel{d}{=} S H_+ \stackrel{d}{=} \diag^{-\sfrac{1}{2}} \left( \widetilde{H}_{\sfrac{1}{2}} \widetilde{H}_{\sfrac{1}{2}}^\top/N \right) \widetilde{H}_{\sfrac{1}{2}}, \quad  \widetilde{H}_{\sfrac{1}{2}} := S H + \gamma  W  S H.
\end{align}
Hence the law of $S H$ is invariant too. Since the invariant distribution is assumed to be unique, $S H \stackrel{d}{=} H$ holds and thus $H_{i:} \stackrel{d}{=} - H_{i:}$.

\end{proof}

\paragraph{Comment on \bn{}-centering} 

Let $\nu_\gamma$ be the unique invariant distribution associated with Markov chain $\{H_{\ell}^{(\gamma)}\}$. A straightforward implication of last Lemma is $\E \left[ H_i \right] =0 $ for $H\sim \nu_\gamma$, hence the rows of $H_{\ell}^{(\gamma)}$ are mean-zero, hence their average is close to zero \footnote{When $d$ is sufficiently large and assuming that coordinates in one row are weakly dependent, the central limit theorem implies that the empirical average of the rows converges to zero.} and the mean-zero operation in \bn{} is redundant. Although this theoretical argument is established for linear networks, we empirically observed that \bn{} without  centering also works well on modern neural architectures. For example, Fig.~\ref{fig:mean_deduction_test} shows that the centering does not affect the performance of \bn{} on a VGG net when training CIFAR-10.

\begin{figure}[h!]
    \centering
    \begin{tabular}{c c}
       \includegraphics[width=0.35\textwidth]{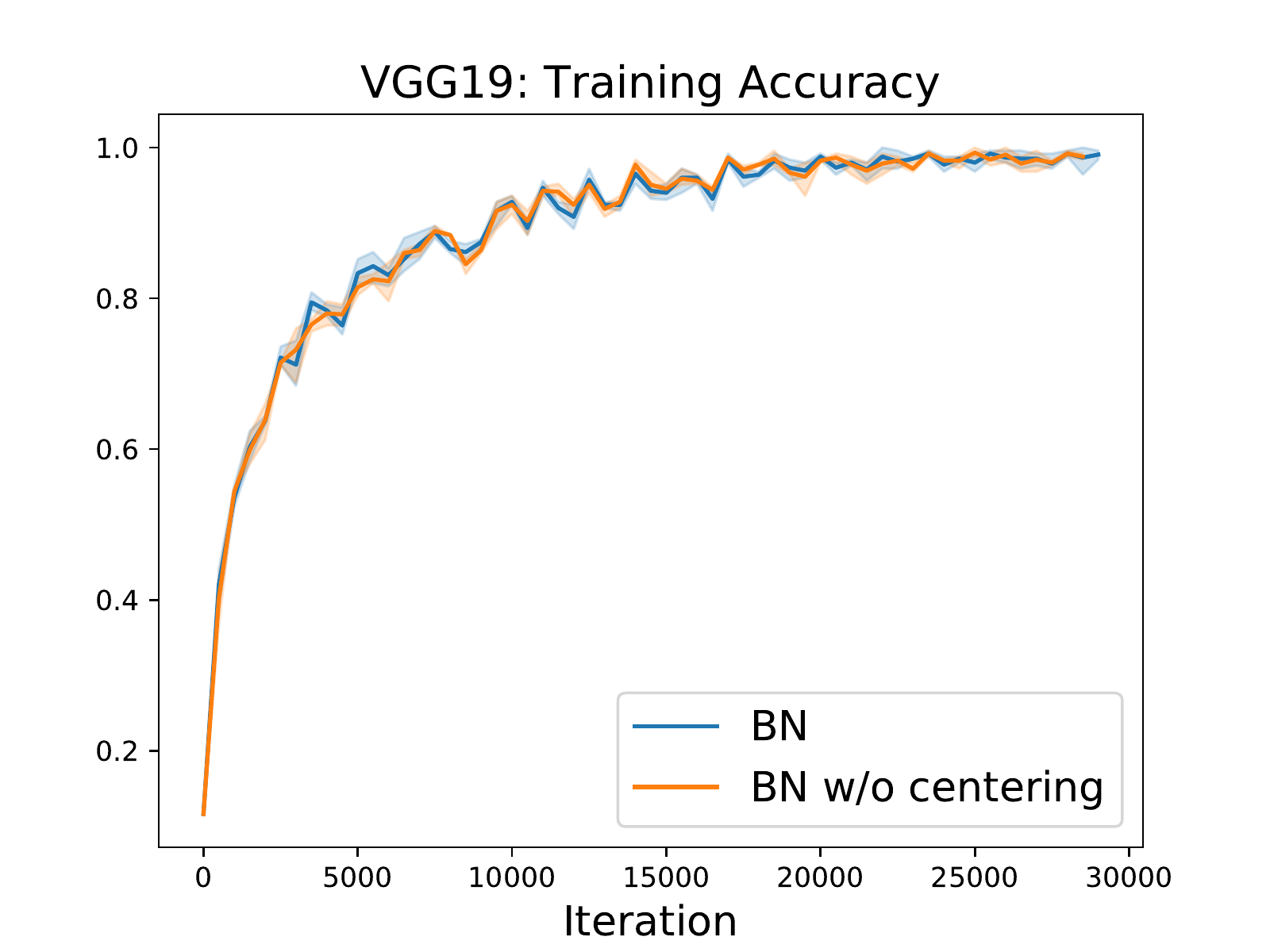}  &  \includegraphics[width=0.35\textwidth]{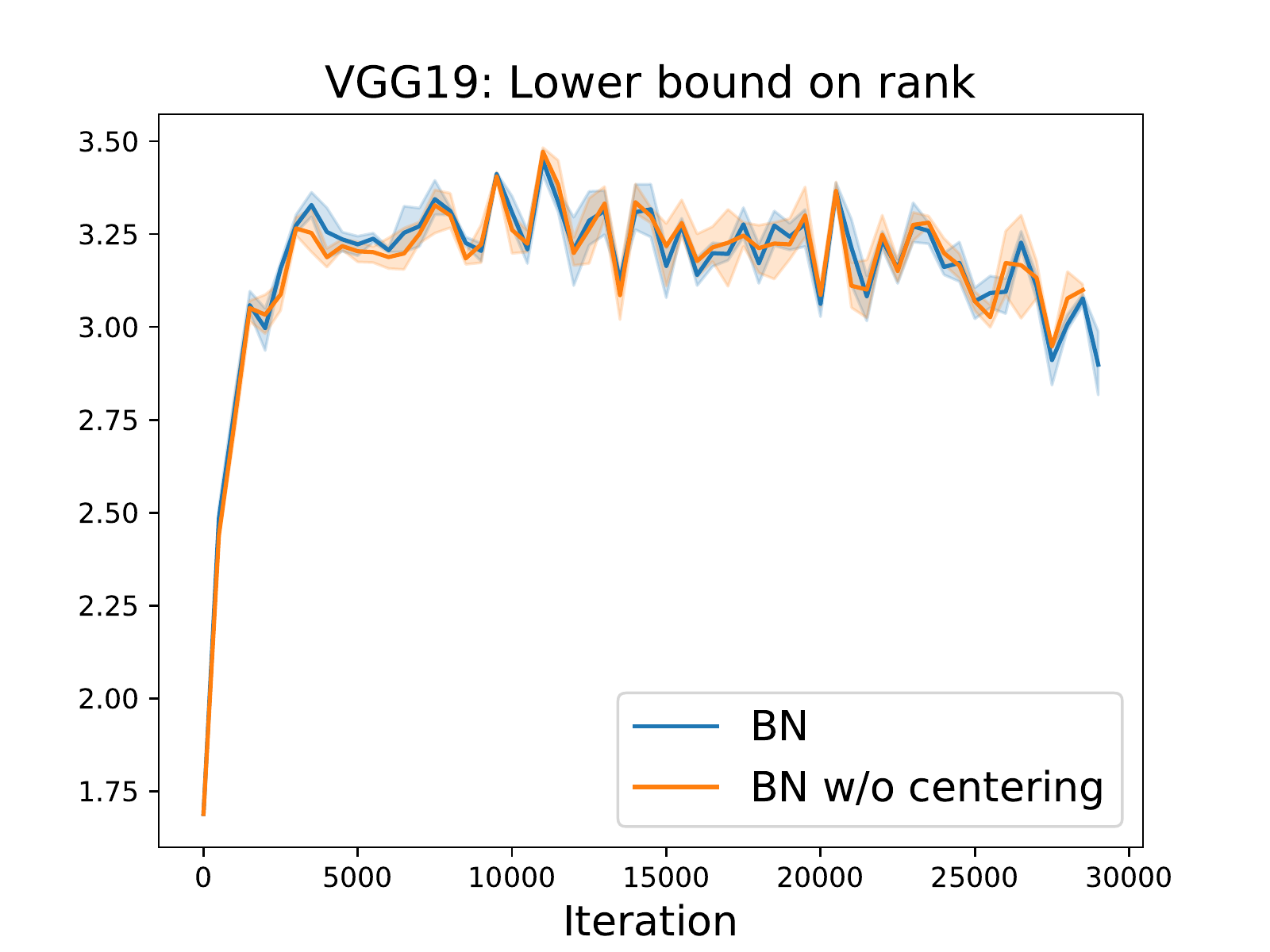}\\
        Training Accuracy & $r(H_L)$
    \end{tabular}
    \caption{Centering for \bn{}. The experiment is conducted on a VGG network. The blue line indicates the original \bn{} network and the orange line is \bn{} without mean adaption. The vertical axis in the left plot is training accuracy. In the right plot it is $r(H_L)$, where $H_L$ is the data representation in the last hidden layer $L$. The horizontal axis indicates the number of iterations.}
    \label{fig:mean_deduction_test}
\end{figure}
\section{Main Theorem: warm-up analysis}

As a warm-up analysis, the next lemma proves that $\rank(H_{\ell}^{(\gamma)}) \geq 2 $ holds. Later, we will prove a stronger result. Yet, this initial results provides valuable insights into our proof technique. Furthermore, we will use the following result in the next steps. 
\begin{lemma}
 \label{lemma:rank2_odd_activations}
 Suppose that each element of the weight matrices is independently drawn from distribution $\P$ that is zero-mean, unit-variance, and its support lies in interval $[-B,B]$.  If the Markov chain $\{ H_\ell \}_{\ell\geq 1}$ admits a unique invariant distribution, then
\begin{align} 
\rank(H_\ell^{(\gamma)}) \geq 2 
\end{align}
holds almost surely for all integers $\ell$ and $\gamma \leq 1/(8d)$. 
\end{lemma}
\begin{proof}
Let the weights $\{ W_\ell\}$ be drawn from the distribution $\mu$, defined in Def.~\ref{definition:weight_init}. Such a distribution obeys an important property: element-wise symmetricity. That is, $[W_{\ell}]_{ij}$ is distributed as $-[W_\ell]_{ij}$. Such an initialization enforces an interesting structural property for the invariant distribution $\nu_\gamma$ that is stated in Lemma~\ref{lemma:symmetricity}. It is easy to check that this implies 
\begin{align} \label{eq:symmetricity_consequence}
    \E \left[ [M(H_{\ell}^{(\gamma)})]_{ij}\right] = -\E \left[ [M(H_{\ell}^{(\gamma)})]_{ij}\right] = 0, 
\end{align}
for any $i \neq j$. Recall, $M(H) = H H^\top/N$. The above property enforces $[M(H)]^2_{ij}$ to be small and hence $\| M_\ell^{(\gamma)}\|_F^2$ is small as well. Now, as $\rank(H_\ell^{(\gamma)})$ is proportional to $1/\| M_{\ell}^{(\gamma)} \|_F^2$ (compare Eq.~(\ref{eq:rH})), it must consequently stay large. The rest of the proof is based on this intuition. Given the uniqueness of the invariant distribution, we can invoke Birkhoff’s Ergodic Theorem for Markov Chains (Theorem 5.2.1 and 5.2.6 \cite{douc2018markov}) which yields 
\begin{align}\label{eq:non_linear_ergod}
    \lim_{L\to \infty}  \frac{1}{L} \sum_{\ell=1}^L [M_{\ell}^{(\gamma)}]_{ij} &= \E_{H \sim \nu_\gamma} \left[ [M(H)]_{ij} \right].
\end{align}
This allows us to conclude the proof by a simple contradiction. Assume that $\rank(H_k^{(\gamma)})$ is indeed one. Then, as established in the following Lemma, in the limit all entries of $M(H_{\ell}^{(\gamma)})$ are constant and either $-1$ or $1$.
\begin{lemma} \label{lemma:contradiction_rank2_non_linear}Suppose the assumptions of Lemma~\ref{lemma:rank2_odd_activations} hold. 
If $\rank(H_k^{(\gamma)})=1$ for an integer $k$, then  $M(H_{\ell}^{(\gamma)}) = M(H_k^{(\gamma)})$ holds for all $\ell>k$. Furthermore, all elements of all matrices $\{M(H_{\ell}^{(\gamma)})\}_{\ell\geq k}$ have absolute value one, hence 
\begin{align} \label{eq:rank2_contradiction_lemma_limit_app}
    \lim_{L\to \infty}  \frac{1}{L} \sum_{\ell=1}^L [M(H_{\ell}^{(\gamma)})]_{ij}  \in \{ 1, -1 \}
\end{align}
holds.
\end{lemma}

As a result, leveraging the ergodicity established in (\ref{eq:non_linear_ergod}), we get that then
\begin{align}
     \E_{H \sim \nu_\gamma} \left[ [M(H)]_{ij} \right] \in \{+1, -1 \}
\end{align}
must also hold. However, this contradicts the consequence of the symmetricity (Eq.~\eqref{eq:symmetricity_consequence}) which states that for any $j\neq i$ we have $ \E_{H \sim \nu_\gamma} \left[ [M(H)]_{ij} \right] = - \E_{H \sim \nu_\gamma} \left[ [M(H)]_{ij} \right] =0$. Thus, the rank one assumption cannot hold, which proves the assertion. 
\end{proof}
To complete the proof of the last theorem, we prove Lemma~\ref{lemma:contradiction_rank2_non_linear}. 
\begin{proof}[Proof of Lemma~\ref{lemma:contradiction_rank2_non_linear}]
 For the sake of simplicity, we omit all superscripts $(\gamma)$ throughout the proof. Suppose that $\rank(H_k) =1$, then $\rank(H_{\ell}) =1 $ for all $\ell \geq k$ as the sequence $\{ \rank(H_{\ell})\}$ is non-increasing~\footnote{Recall that the updates in Eq.~\eqref{eq:nn} is obtained by matrix multiplications, hence it does not increase the rank.}. Invoking the established rank bound from Lemma~\ref{lemma:soft_rank_lowerbound}, we get  
\begin{align} 
    r(H_\ell) = \frac{\tr(M_\ell)^2}{\| M_\ell \|_F^2} \leq \rank(H_\ell) = 1.
\end{align}
 The above inequality together with properties (p.1) and (p.2) (presented in Eqs~\ref{p1m} and ~\ref{p2m}) yield $\tr(M_{\ell}) = d$. Replacing this into the above equation gives that $\| M_\ell \|_F^2 \geq d^2$ must hold for the rank of $H_\ell$ to be one. Yet, recalling property (p.2), this can only be the case if $[M_{\ell}]_{ij} \in \{ +1, -1 \}$ for all $i,j$.
Replacing the definition $M(H) = H H^\top/N$ into updates of hidden presentation in Eq.~\ref{eq:nn} obtains 
\begin{align}
    M_{\ell+1} = \diag\left(M_{\ell + \frac{1}{2}}\right)^{-\sfrac{1}{2}} \left(M_{\ell + \frac{1}{2}}\right) \diag\left(M_{\ell + \frac{1}{2}}\right)^{-\sfrac{1}{2}}
\end{align}
where 
\begin{align}
    M_{\ell + \frac{1}{2}} = M_{\ell} + \Delta M_{\ell}, \quad \Delta M_{\ell} := \gamma W_\ell M_{\ell} + \gamma M_{\ell} W_\ell^\top + \gamma^2 W_\ell M_{\ell} W_{\ell}^\top
\end{align}
 We now prove that the sign of $[M_{\ell}]_{ij}$ and $[M_{\ell+1}]_{ij}$ are the same for $[M_{\ell}]_{ij} \in \{+1,-1\}$. The above update formula implies that the sign of $[M_{\ell+1}]_{ij}$ equates that of $[M_{\ell+\sfrac{1}{2}}]_{ij}$. 
Furthermore, it is easy to check that $|[\Delta M_\ell]_{ij}|\leq 4\gamma B$. For $\gamma \leq 1/(8Bd)$, this bound yields $|[\Delta M_\ell ]_{ij}|\leq \frac{1}{2}$. Therefore, the sign of $[M_{\ell+\sfrac{1}{2}}]_{ij}$ is equal to the one of $[M_{\ell}]_{ij}$. Since furthermore $[M_{\ell+1}]_{ij} \in \{1,-1\}$ holds, we conclude that all elements of $M_\ell$ remain constant for all $\ell \geq k$, which yields the limit stated in Eq.~ \ref{eq:rank2_contradiction_lemma_limit_app} . 
\end{proof}

\section{Main theorem: Proof}\label{sec:main_theorm}

In this section, we prove that \bn{} yields an $\Omega(\sqrt{d})$-rank for hidden representation. 

\begin{proof}[Proof sketch for Thm.~\ref{thm:linear-rank-lowerboun}]
The proof is based on an application of ergodic theory (as detailed for example in Section 5 of \cite{douc2018markov}). In fact, the chain of hidden representations, denoted by $H_\ell^{(\gamma)}$ \eqref{eq:nn}, constitutes a Markov chain in a compact space. This chain admits at least one invariant distribution $\nu$ for which the following holds
\begin{align}
   \int g(\bn_{0,\ones_d}(H + \gamma W H)) \mu(dW) \nu(dH) = \int g(H) \nu(dH),
\end{align}
for every bounded Borel function $g:\R^{d\times d} \to \R^d$. The above invariance property provides an interesting characterization of the invariant measure $\nu$. Particularly, we show in Lemma~\ref{lemma:fnorm_invariance} that 
\begin{align}\label{eq:invariance_sqrt_d}
      \int r(H) \nu(dH) = \Omega(\sqrt{d})
\end{align}
holds, where $r(H)$ is the established lower-bound on the rank (see Lemma~\ref{lemma:soft_rank_lowerbound}). Under weak assumptions, the chain obey Birkhoff's Ergodicity, which yields that the average behaviour of the hidden representations is determined by the invariant measure $\nu$: 
\begin{align}
   \lim_{L \to \infty} \frac{1}{L} \sum_{i=\ell} r(H_{\ell}^{(\gamma)}) = \int r(H) \nu(dH) \overset{\eqref{eq:invariance_sqrt_d}}{=}\Omega(\sqrt{d}). 
\end{align}
Finally, the established lower bound in Lemma~\ref{lemma:soft_rank_lowerbound} allows us to directly extend this result to a lower bound on the soft rank itself. 
\end{proof}

\noindent\textbf{\underline{Characterizing the change in Frobenius norm}}
Recall the established lower bound on the rank denoted by $r(H)$, for which 
\begin{align}
    r(H_\ell) = \frac{\tr(M_\ell)^2}{\| M_{\ell} \|_F^2 } = \frac{d^2}{\| M_\ell \|_F^2}
\end{align}
holds for all $H_{\ell}$ defined in Eq.~\ref{eq:nn}.\footnote{Recall $\tr(M_{\ell}) = d$ holds due to property (p.2) in Eq.~\ref{p2m}} Therefore, $\| M_\ell \|_F^2$ directly influences $\rank_\tau(H_\ell)$ (and also $\rank(H_\ell)$) according to Lemma~\ref{lemma:soft_rank_lowerbound}. Here, we characterize the change in $\|M(H)\|_F^2$ after applying one step of the recurrence in Eq.~\ref{eq:chain_app} to $H$, i.e. passing it trough one hidden layer. This yields
\begin{align}
    H_+ = \left(\diag(M(H_\gamma(W)) \right)^{-\sfrac{1}{2}} H_\gamma(W), \quad  H_\gamma(W)  = (I + \gamma W) H.
\end{align}
Let $M = M(H)$ and $M_+ = M(H_+)$ for simplicity.
The next lemma estimates the expectation (taken over the randomness of $W$) of the difference between the Frobenius norms of $M_+$ and $M$.

\begin{lemma} \label{lemma:Fnorm_expansion}
If $\W \sim \mu$ (defined in Def.~\ref{definition:weight_init}), then
\begin{align} \label{eq:}
     \left(\E_W \| M_+\|_F^2  - \|M \|_F^2\right)/(\gamma^2)= \underbrace{2 d^2 - 2 \| M \|_F^2 - 8 \tr(M^3) + 8 \tr(\diag(M^2)^2)}_{\delta_F(M)} + \bigo(\gamma)
\end{align}
holds as long as the support of distribution $\P$ (in Def.~\ref{definition:weight_init}) lies in a finite interval $[-B,B]$. 
\end{lemma}

The proof of the above lemma is based on a Taylor expansion of the \bn{} non-linear operator. We postpone the detailed proof to the end of this section. While the above equation seems complicated at first glance, it provides some interesting insights.

\noindent\textbf{Interlude: Intuition behind Lemma~\ref{lemma:Fnorm_expansion}.}
In order to gain more understanding of the implications of the result derived in Lemma~\ref{lemma:Fnorm_expansion}, we make the simplifying assumption that all the rows of matrix $M$ have the same norm. We emphasize that this assumption is purely for intuition purposes and is not necessary for the proof of our main theorem. Under such an assumption, the next proposition shows that the change in the Frobenius norm directly relates to the spectral properties of matrix $M$. 

\begin{proposition} \label{proposition:spectral_presentation}
Suppose that all the rows of matrix $M$ have the same norm. Let $\lambda \in \R^d$ contain the eigenvalues of matrix $M$. Then, 
\begin{align}
    \tr(M^3) = \| \lambda \|^3_3,\quad  \tr(\diag(M^2))^2 = \| \lambda \|^4/d, \quad  \| M \|_F^2 = \| \lambda \|^2_2
\end{align}
holds and hence 
\begin{align}
    \delta_F(M) = \delta_F(\lambda) := 2d^2 -2 \| \lambda \|_2^2 - 8 \| \lambda \|^3_3 + 8 \| \lambda \|^4/d.
\end{align}
\end{proposition}
We postpone the proof to the end of this section. This proposition re-expresses the polynomial of Lemma \ref{lemma:Fnorm_expansion} in terms of the eigenspectrum of $M$. 

Based on the above proposition, we can make sense of interesting empirical observation reported in Figure~\ref{fig:markov_chain_d}.b. This figure plots the evolution of the eigenvalues of $M(H_\ell^{(\gamma)})$ after starting from a matrix $M(H_0)$ whose leading eigenvalue is large and all other eigenvalues are very small. We observe that a certain fraction of the small eigenvalues of $M(H_\ell^{(\gamma)})$ grow quickly with $\ell$, while the leading eigenvalue is decreases in magnitude. In the next example, we show that the result of the last proposition actually predicts this observation.
 \begin{example} \label{exm:eigv}
Suppose that $M$ is a matrix whose rows have the same norm. Let $\lambda_1 \geq \lambda_2, \dots, \lambda_d$ be the eigenvalues associated with the matrix $M$ such that $\lambda_d = \lambda_{d-1} = \lambda_2 = \gamma^2$ and $\lambda_1 = d-\gamma^2(d-1)$. In this setting, Prop.~\ref{proposition:spectral_presentation} implies that  $\E_W\|M_+ \|_F^2 < \| M \|_F^2- \gamma^4 d^2$ for a sufficiently small $\gamma$. This change has two consequences in expectation:(i.) the leading eigenvalue of $M_+$ is $\bigo(-\gamma^4 d)$ smaller than the leading eigenvalue of $M$, and (ii.) some small eigenvalues of $M_+$ are greater than those of $M$ (see Fig.~\ref{fig:markov_chain_d}.b).
\end{example}

We provide a more detailed justification for the above statement at the end of this section. 
This example illustrates that the change in Frobenius norm (characterized in Lemma~\ref{lemma:Fnorm_expansion}) can predict the change in the eigenvalues of $M(H_{\ell}^{(\gamma)})$ (singular values of $H_\ell^{(\gamma)}$) and hence the desired rank. Inspired by this, we base the proof of Theorem~\ref{thm:linear-rank-lowerboun} on leveraging the invariance property of the unique invariant distribution with respect to Frobenius norm -- i.e. setting $g(H) = \| M(H) \|_F^2$ in Def.~\ref{definition:invariance}.

\noindent\textbf{\underline{An observation: regularity of the invariant distribution}} We now return to the result derived in Lemma~\ref{lemma:Fnorm_expansion} that characterizes the change in Frobenius norm of $M(H)$ after the recurrence of Eq.~\eqref{eq:chain_app}. We show how such a result can be used to leverage the invariance property with respect to the Frobenius norm. First, we observe that the term $\tr(M(H)^3)$ in the expansion can be shown to dominate the term $\tr(\diag(M(H)^2)^2)$ in expectation. The next definition states this dominance formally. 

\begin{definition} \label{def:regularity}
(Regularity constant $\alpha$)
Let $\nu$ be a distribution over $H \in \Hs$. Then the regularity constant associated with $\nu$ is defined as the following ratio:
\begin{align}
    \alpha = \E_{H \sim \nu} \left[ \tr\left(\diag(M(H)^2)^2 \right)\right]/\left( \E_{H \sim \nu} \left[ \tr\left(M(H)^3\right) \right]\right).
\end{align}
\end{definition}

The next lemma states that the regularity constant $\alpha$ associated with the invariant distribution $\nu_\gamma$ is always less than one. Our analysis will in fact directly rely on $\alpha<1$.

\begin{lemma} \label{lemma:regularity}
Suppose that the chain $\{ H_\ell^{(\gamma)}\}$ admits the unique invariant distribution $\nu_\gamma$ (in Def.~\ref{definition:invariance}). Then, the regularity constant of $\nu_\gamma$ (in Def.~\ref{def:regularity}) is less than one for a sufficiently small $\gamma$. 
\end{lemma}
\begin{proof}
We use a proof by contradiction where we suppose that the regularity constant of distribution $\nu_\gamma$ is greater than one. In this case, we prove that the distribution cannot be invariant with respect to the Frobenius norm.

If the regularity constant $\alpha$ is greater than one, then  
\begin{align} \label{eq:alpha_greater_one}
   \E_{H \sim \nu_\gamma} \left[ -\tr(M(H)^3) + \tr(\diag(M(H)^2)^2) \right]  \geq 0
\end{align}
holds.
According to Theorem~\ref{lemma:rank2_odd_activations}, the rank of $M(H)$ is at least 2. Since the sum of the eigenvalues is constant $d$, the leading eigenvalue is less than $d$. This leads to \[\| M(H)\|_F^2 = \sum_{i} \lambda_i^2 \leq \max_i \lambda_i \left( \sum_{j} \lambda_j \right) \leq d \max_i \lambda_i < d^2.\] Plugging the above inequality together with inequality~\ref{eq:alpha_greater_one} into the established bound in Lemma~\ref{lemma:Fnorm_expansion} yields 
\begin{align}
    \E_{W, H \sim \nu_\gamma} \left[ \| M(H_+) \|^2_F - \| M(H)\|_F^2 \right] > 0 
\end{align}
for a sufficiently small $\gamma$. Therefore, $\nu_\gamma$ does not obey the invariance property for $g(H) = \| M(H)\|_F^2$ in Def.~\ref{definition:invariance}. 
\end{proof}

We can experimentally estimate the regularity constant $\alpha$ using the Ergodicity of the chain. Assuming that the chain is Ergodic\footnote{The uniqueness of the invariant distribution implies Ergodicity (see Theorem 5.2.1 and 5.2.6 \cite{douc2018markov}).}, 
\begin{align} \label{eq:ergodic_averaging}
    \lim_{L\to \infty} \frac{1}{L} \sum_{\ell=1}^L g(H_\ell^{(\gamma)}) = \E_{H \sim \nu_\gamma} \left[ g(H) \right]  
\end{align}
holds almost surely for every Borel bounded function $g: \Hs \to \R$. By setting $g_1(H) = \tr(M(H)^3)$ and $g_2(H) = \tr(\diag(M(H)^2)^2)$, we can estimate $\E_{H \sim \nu_\gamma} \left[g_i(H)\right] $ for $i=1,$ and 2. Given these estimates, $\alpha$ can be estimated. Our experiments in Fig.~\ref{fig:regualrity} show that the regularity constant of the invariant distribution $\nu_\gamma$ is less than $0.9$ for $d>10$. 
\begin{figure}[h!]
    \centering
    \begin{tabular}{c c}
        \includegraphics[width=0.4\textwidth]{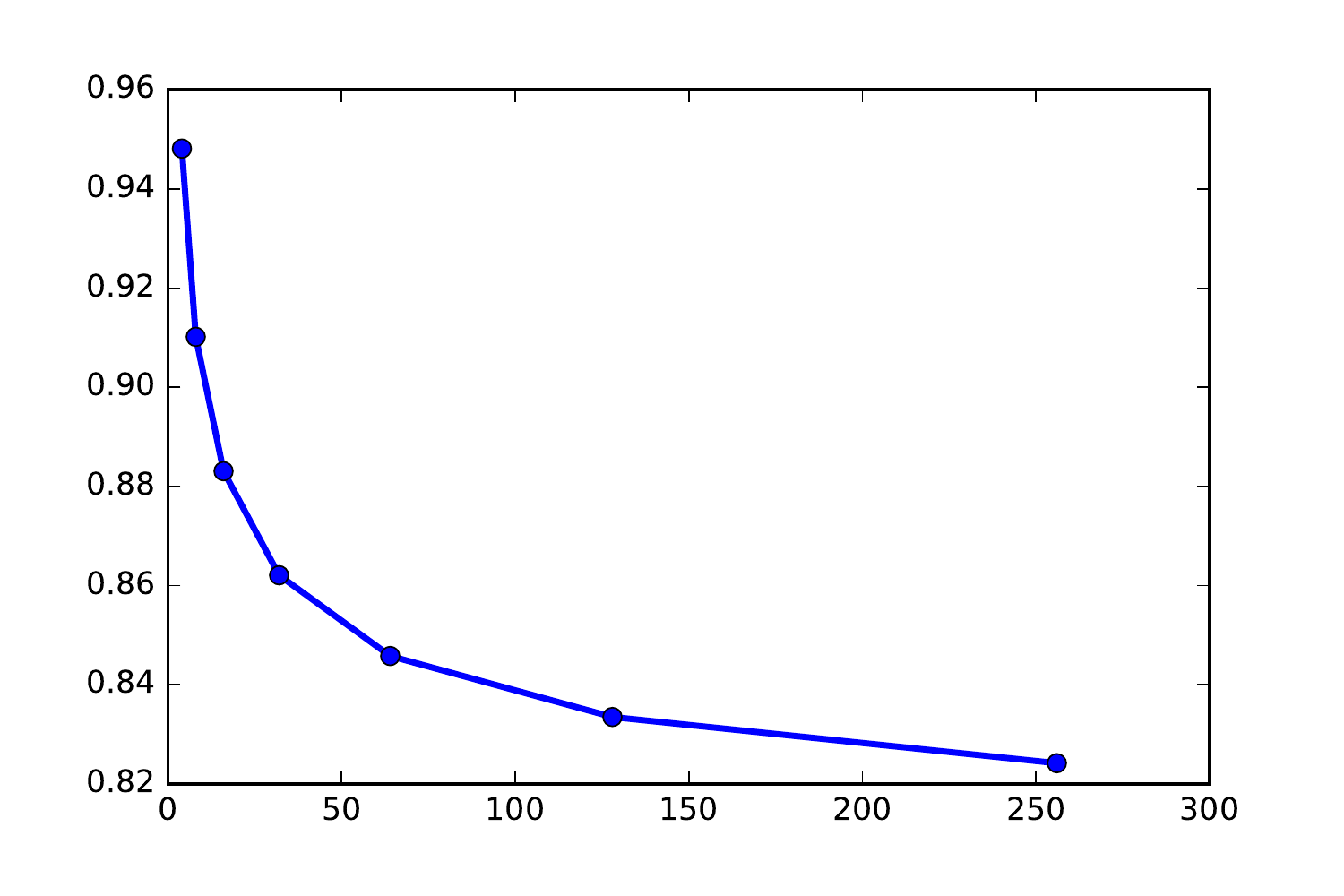} & \includegraphics[width=0.4\textwidth]{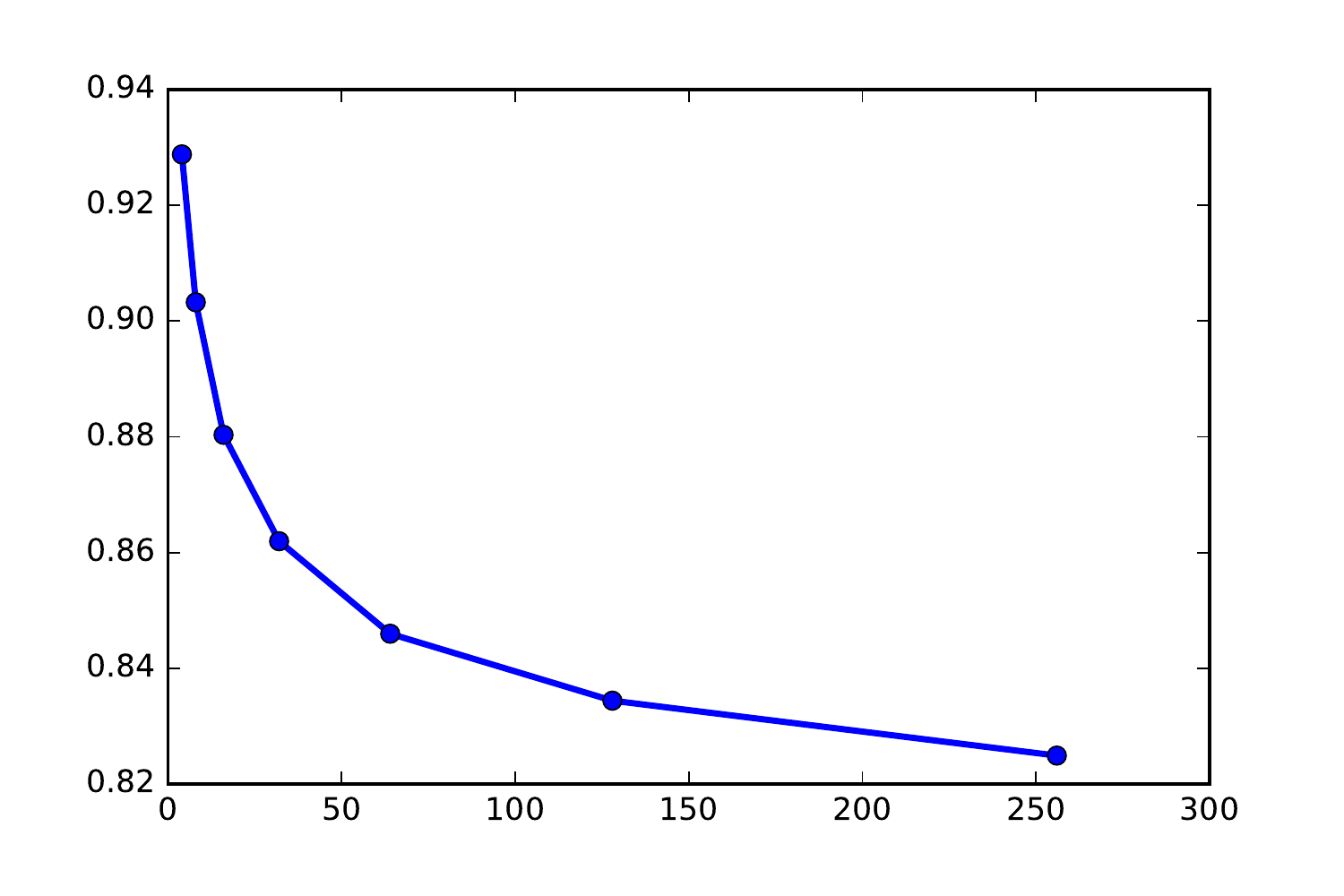} \\
         $\gamma=1$ &  $\gamma = 0.1$ 
    \end{tabular}
    
    \caption{Regularity constant of the invariant distribution. The vertical axis is the estimated regularity constant $\alpha$ and the horizontal axis is $d$. We use $L=10^{5}$ (in Eq.~\eqref{eq:ergodic_averaging}).}
    \label{fig:regualrity}
\end{figure}

\noindent\textbf{Interlude: intuition behind the regularity}
We highlight the regularity constant does by itself not yield the desired rank property in Theorem~\ref{thm:linear-rank-lowerboun}. This is illustrated in the next example that shows how the regularity constant relates to the spectral properties of $M(H)$.

\begin{example} \label{exm:regualrity}
Suppose that the support of distribution $\nu$ contains only matrices $H \in \Hs$ for which all rows of $M(H)$ have the same norm. If the regularity constant of $\nu$ is greater than or equal to one, then all non-zero eigenvalues of matrix $M(H)$ are equal. 
\end{example}

A detailed justification of the above statement is presented at the end of this section. This example shows that the regularity constant does not necessarily relate to the rank of $H$, but instead it is determined by how much non-zero eigenvalues are close to each other. We believe that a sufficient variation in non-zero eigenvalues of $M(H)$ imposes the regularity of the law of $H$ with a constant less than one (i.e. $\alpha<1$ in Def.~\ref{def:regularity}). The next example demonstrates this. 

\begin{example} \label{exm:regularity_variation}
Suppose the support of distribution $\nu$ contains matrices $H \in \Hs$ for which all rows of $M(H)$ have the same norm. Let $\lambda\in \R^d$ contain sorted eigenvalues of $M(H)$. If $\lambda_1 = \Theta(d^{\beta})$ and $\lambda_i = o(d^\beta)$ for $i>1$ and $\beta<1$,\footnote{According to definition, $\lim_{d\to \infty} o(d^\beta)/\Theta(d^{\beta}) = 0 $} then the regularity constant $\alpha$ associated with $\nu$ is less than $0.9$ for sufficiently large $d$.  
\end{example}
We later provide further details about this example. 

\noindent\textbf{\underline{Invariance consequence}}
The next lemma establishes a key result on the invariant distribution $\nu_\gamma$. 
\begin{lemma} \label{lemma:fnorm_invariance}
Suppose that the chain $ \{H_{\ell}^{(\gamma)}\}$ (see Eq.~\ref{eq:chain_app}) admits the invariant distribution $\nu_\gamma$ (see Def.~\ref{definition:invariance}). If the regularity constant associated with $\nu_\gamma$ is $\alpha<1$ (defined in Def.~\ref{def:regularity}), then 
\begin{align}
    \E_{H \sim \nu_\gamma}  \left[ \| M(H)\|_F^2 \right] \leq d^{\sfrac{3}{2}}/\sqrt{1-\alpha}
\end{align}
holds for a sufficiently small $\gamma$. 
\end{lemma}
\begin{proof}
Leveraging invariance property in Def.~\ref{definition:invariance}, \begin{align}
    \E_{W,H \sim \nu_\gamma} \left[ \|M(H_+) \|_F^2 - \| M(H) \|_F^2 \right] = 0 
\end{align}
holds where the expectation is taken with respect to the randomness of $W$ and $\nu_\gamma$.\footnote{This result is obtained by setting $g(H)= \|M(H)\|_F^2$ in Def.~\ref{definition:invariance}.} Invoking the result of Lemma~\ref{lemma:Fnorm_expansion}, we get 
\begin{align}
    \E_{H \sim \nu_\gamma}  \left[ 2 d^2 - 2 \| M(H) \|_F^2 - 8 \tr(M(H)^3) + 8 \tr(\diag(M(H)^2)^2) \right] + \bigo(\gamma) = 0.
\end{align}
Having 
a regularity constant less than one for $\nu_\gamma$ implies 
\begin{align}
  0 \leq   2 d^2 - \E_{H \sim \nu_\gamma}  \left[ 2 \| M(H) \|_F^2 - 8(1-\alpha) \tr(M(H)^3)\right] 
\end{align}
holds for sufficiently small $\gamma$. Let $\lambda \in \R^d$ be a random vector containing the eigenvalues of the random matrix $M(H)$.\footnote{Note that $H \in \Hs$ is a random matrix whose law is $\nu_\gamma$, hence $\lambda \in \R^d$ is also a random vector.} The eigenvalues of $M^3$ are $\lambda^3$, hence the invariance result can be written alternatively as 
\begin{align}
    0 \leq   2 d^2 - \E \left[ 2 \| \lambda \|_2^2 - 8(1-\alpha) \| \lambda \|_3^3 \right]. 
\end{align}
The above equation leads to the following interesting spectral property: 
\begin{align} \label{eq:invariance_result_lambda_3}
     \E \| \lambda \|_3^3  \leq  d^2/(1-\alpha).
\end{align}
A straightforward application of Cauchy-schwarz yields:
\begin{align}
     \| \lambda \|_2^2 & = \sum_{i} \lambda_i^2 = \sum_{i} \lambda_i^{\sfrac{1}{2}} \lambda_i^{\sfrac{3}{2}} \leq \sqrt{\sum_{i}\lambda_i  \sum_{j} \lambda_i^3}  \leq 
    \sqrt{d \| \lambda \|_3^3}
\end{align}
Given (i) the above bound, (ii) an application of Jensen's inequality, (iii) and the result of Eq.~\eqref{eq:invariance_result_lambda_3}, we conclude with the desired result: 
\begin{align}
    \E_{H \sim \nu_\gamma} \left[ M(H)\right] = \E \left[ \| \lambda \|_2^2 \right] 
    \stackrel{\text{(i)}}{\leq}  \E \sqrt{d \| \lambda \|_3^3}  \stackrel{\text{(ii)}}{\leq} \sqrt{d \E \| \lambda \|_3^3} \stackrel{\text{(iii)}}{\leq} d^{\sfrac{3}{2}}/\sqrt{1-\alpha}
\end{align}
\end{proof}

Notably, the invariant distribution is observed to have a regularity constant less than $0.9$ (in Fig.~\ref{fig:regualrity}) for sufficiently large $d$. This implies that an upper-bound $\bigo\left( d^{\sfrac{3}{2}}\right)$ is achievable on the Frobenius norm. Leveraging Ergodicity  (with respect to Frobenius norm in Eq.~\eqref{eq:ergodic_averaging}), we experimentally validate the result of the last lemma in Fig.~\ref{fig:fnorm_check}. 
\begin{figure}[h!]
    \centering
    \begin{tabular}{c c}
      \includegraphics[width=0.4\textwidth]{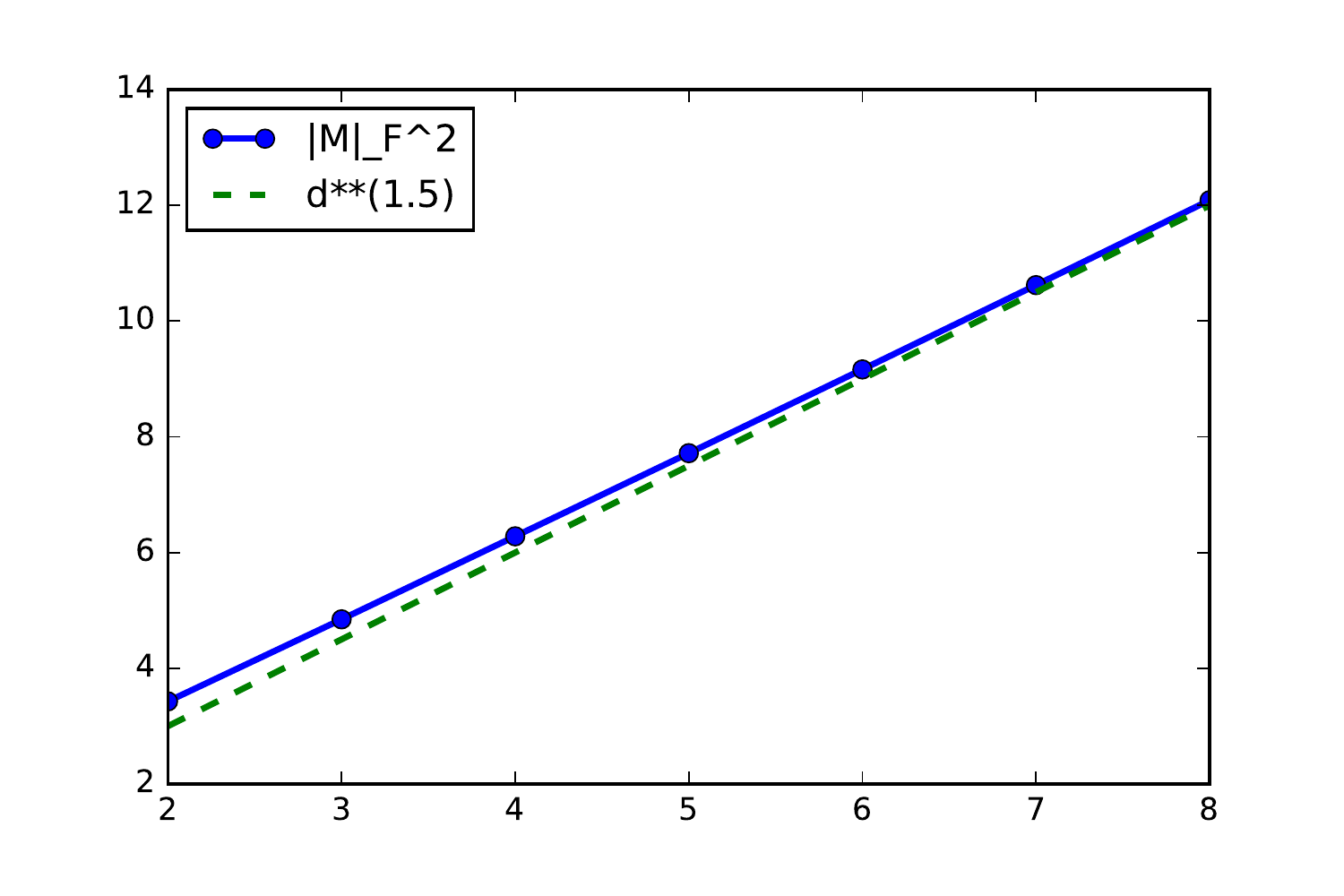}   & \includegraphics[width=0.4\textwidth]{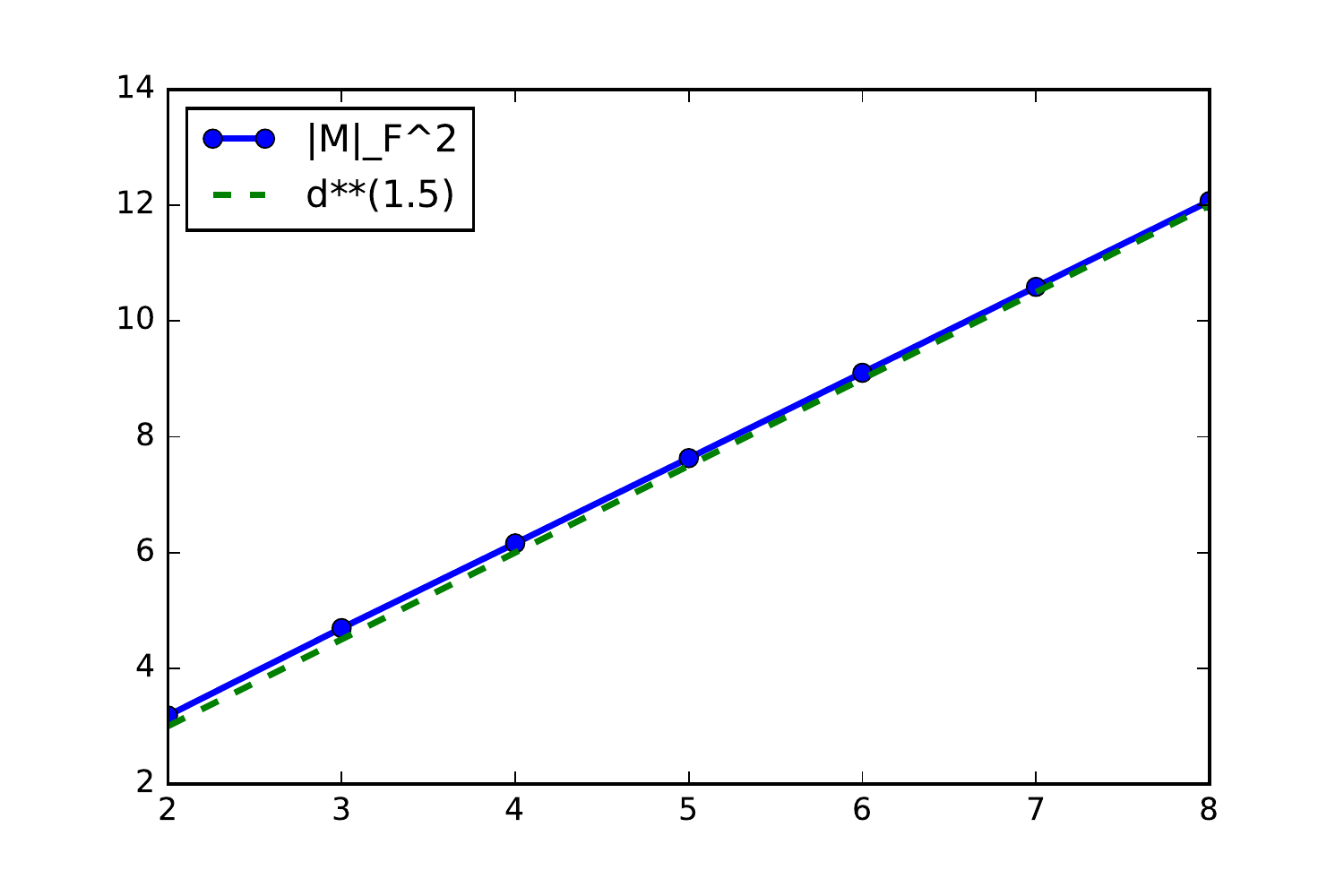} \\
        $\gamma = 1$ & $\gamma = 0.1$
    \end{tabular}
    
    \caption{Dependency of $\E_{\nu_\gamma} \|M(H) \|_F^2$ on $d$. The horizontal axis is $\log_2(d)$ and the vertical axis shows $\log_2(\frac{1}{L}\sum_{\ell=1}^L \| M(H_\ell^{(\gamma)}) \|_F^2)$ for $L=10^5$. The green dashed-line plots $\log_2(d^{1.5})$. }
    \label{fig:fnorm_check}
\end{figure}

\noindent\textbf{\underline{Proof of the Main Theorem}}
Here, we give a formal statement of the main Theorem that contains all required additional details (which we omitted for simplicity in the original statement).

\begin{theorem} [Formal statement of Theorem~\ref{thm:linear-rank-lowerboun}]\label{thm:linear-rank-lowerbound-formal}
Suppose that $\rank(X) = d$, $\gamma$ is sufficiently small, and all elements of the weight matrices $\{\W_\ell\}$ are drawn i.i.d. from a zero-mean, unit variance distribution whose support lies in $[-B,B]$ and its law is symmertic around zero. Furthermore, assume that the Markov chain $\{ H_{\ell}^{(\gamma)} \}$ (defined in Eq.~\ref{eq:nn}) admits a unique invariant distribution. Then, the regularity constant $\alpha>0$ associated with $\nu_\gamma$ (see Def.~\ref{def:regularity}) is less than one and the following limits exist such that
\begin{align} \label{eq:average_lowerbound_app}
    \lim_{L \to \infty} \frac{1}{L} \sum_{\ell=1}^L \rank_\tau(H_\ell^{(\gamma)}) 
    \geq \lim_{L \to \infty} \frac{(1-\tau)^2}{L} \sum_{\ell=1}^L  r(H_\ell^{(\gamma)})
    \geq (1-\tau)^2 (1-\alpha)^{\sfrac{1}{2}} \sqrt{d}
\end{align}
holds almost surely for all $\tau \in [0,1]$.
Assuming that the regularity constant $\alpha$ does not increase with respect to $d$, the above lower-bound is proportional to $(1-\alpha)^{\sfrac{1}{2}} \sqrt{d}=\Omega(\sqrt{d})$. 
\end{theorem}

Remarkably, we experimentally observed (in Fig.~\ref{fig:regualrity}) that the regularity constant $\alpha$ is decreasing with respect to $d$. Examples~\ref{exm:regualrity} and \ref{exm:regularity_variation} provide insights about the regularity constant. We believe that it is possible to prove that the constant $\alpha$ is non-increasing with respect to $d$.
\begin{proof}[Proof of Theorem~\ref{thm:linear-rank-lowerboun} ]
Lemma~\ref{lemma:regularity} proves that the regularity constant $\alpha$ is less than one for the unique invariant distribution. 
Suppose that $H \in \Hs$ is a random matrix whose law is the one of the unique invariant distribution of the chain. For $H \in \Hs$, we get $\tr(M(H))=d$. A straightforward application of Jensen's inequality yields the following lower bound on the expectation of $r(H)$ (i.e. the lower bound on the rank):
\begin{align}
    \E \left[ r(H) \right] = \E \left[ \tr(M(H))^2/\| \M(H)\|_F^2 \right]  = \E \left[ d^2/\| \M(H)\|_F^2 \right] \geq d^2/\E \left[ \| M(H) \|_F^2 \right]  
\end{align}
where the expectation is taken over the randomness of $H$ (i.e. the invariant distribution). 
Invoking the result of Lemma~\ref{lemma:fnorm_invariance}, we get an upper-bound on the expectation of the Frobenius norm -- in the right-side of the above equation. Therefore,  
\begin{align}
     \E  \left[ r(H) \right] \geq \sqrt{(1-\alpha)d}
\end{align}
holds. The uniqueness of the invariant distribution allows us to invoke Birkhoff’s Ergodic Theorem for Markov Chains (Theorem 5.2.1 and 5.2.6 \cite{douc2018markov}) to get 
\begin{align}\label{eq:non_linear_ergod}
    \lim_{L\to \infty}  \frac{1}{L} \sum_{\ell=1}^L r(H_{\ell}^{(\gamma)}) &= \E\left[ r(H) \right] \geq \sqrt{(1-\alpha)d}.
\end{align}
The established lower bound on $\rank_\tau(H_\ell^{(\gamma)})$ --in terms of $r(H_\ell^{(\gamma)})$-- in Lemma~\ref{lemma:soft_rank_lowerbound} concludes
\begin{align}
     \lim_{L\to \infty}  \frac{1}{L} \sum_{\ell=1}^L \rank_\tau(H_{\ell}^{(\gamma)})  \geq \lim_{L\to \infty}  \frac{(1-\tau)^2}{L} \sum_{\ell=1}^L r(H_{\ell}^{(\gamma)}) \geq (1-\tau)^2\sqrt{(1-\alpha)d}.
\end{align}
\end{proof}
As shown in the following corollary, one can extend the result of Theorem~\ref{thm:linear-rank-lowerbound-formal} for any finite $\ell$.
\begin{corollary} \label{cor:nonasymptotic}
Under the setting of Thm.~\ref{thm:linear-rank-lowerbound-formal}, $\rank(H_\ell)= \Omega(\sqrt{d})$ holds almost surely for all finite integer $\ell$. Assuming that $\{\rank_\tau(H_\ell)\}$ is a monotonically no-increasing sequence, then $\rank_\tau(H_\ell)= \Omega((1-\tau)^2\sqrt{d})$ holds almost surely for all finite $\ell$.  
\end{corollary}
\begin{proof}
The proof is based on the no-increasing property of the rank\footnote{Recall that the rank does not increases in updates of Eq.~\eqref{eq:chain_app} }. Next lemma presents a straightforward implication of this property. 
\begin{lemma} \label{lemma:seq}
Consider a sequence of non-increasing bounded finite integers $\{y_k\}_{k=1}^\infty$. If $\lim_{N\to \infty} \sum_{k=1}^N y_k/N$ exists and is greater than $\alpha$, then $y_k\geq\alpha$ for all finite $k$.
\end{lemma}
The proof of the last lemma is provided at the end of this section. Replacing the result of Thm.~\ref{thm:linear-rank-lowerbound-formal} into the above lemma concludes the proof of the corollary. 

\end{proof}

\paragraph{A remark on the number of hidden units.}
The focus of our analysis was networks with the same number of hidden units in each layer. Yet, this result extends to more general architectures. Most of modern neural architectures consists of blocks in which the number of hidden units are constant. For example, VGG19-Nets and ResNets are consist of blocks convolutional layers with 64, 128, 256, and 512 channels where the number channels are equal in each block. An analogy of such an architecture is an MLP with different blocks of hidden layers where the numbers of hidden units are the same in each block. According to Cor.~\ref{cor:nonasymptotic}, the rank preservation property holds in each block after applying BN. In this way, one can extend the established results of Thm.~\ref{thm:linear-rank-lowerbound-formal} and Cor.~\ref{cor:nonasymptotic} to a general family of architectures with varying number of hidden units.  
\newpage
\noindent\textbf{Postponed proofs.}
\begin{proof}[Proof of Lemma~\ref{lemma:Fnorm_expansion}]
The proof is based on a Taylor expansion of the \bn{} non-linear recurrence function, which we restate here for simplicity:
\begin{align}
    H_+ = \left(\diag(M(H_\gamma)) \right)^{-\sfrac{1}{2}} H_\gamma, \quad  H_\gamma  = (I+ \gamma W) H
\end{align}
Consider the covariance matrices $M = M(H)$ and $M_+ = M(H_+)$ which obey
\begin{align}
    M_\gamma &:= M(H_\gamma) = M + \Delta M, \quad \Delta \M := \gamma \W \M + \gamma \M \W^\top + \gamma^2 \W \M \W^\top \\
    [M_+]_{ij}^2 & = g_{ij}(M_\gamma)  = [M_\gamma]_{ij}^2/[M_{\gamma}]_{ii} [M_{\gamma}]_{jj}
\end{align}
For the sake of simplicity, we use the compact notation $g:=g_{ij}$ for $i\neq j$. 
We further introduce the set of indices $S= \{ii,ij,jj\}$. A taylor expansion of $g$ at $M$ yields 
\begin{multline}
     \E_W \left[ g(M_\gamma) \right] = g(M)  + \underbrace{\sum_{pq \in S}  \left(\frac{\partial g(M)}{\partial M_{pq}}\right) \E_\W \left[\Delta M_{pq} \right] }_{T_1}  \\ + \underbrace{\frac{1}{2}  \sum_{pq,km \in S} \left(\frac{\partial^2 g(M)}{\partial M_{pq}\partial M_{km}}\right)\E_W \left[ \Delta M_{pq} \Delta M_{km} \right]  }_{T_2} + \bigo(\gamma^3).
\end{multline}
Note that the choice of the element-wise uniform distribution over $[-\sqrt{3},\sqrt{3}]$ allows us to deterministically bound the Taylor remainder term by $\bigo(\gamma^3)$. 
Now, we compute the derivatives and expectations that appear in the above expansion individually. Let us start with the term $T_1$. The first-order partial derivative term in $T_1$ is computed bellow. 
\begin{align}
    \frac{\partial g(M)}{\partial M_{pq}} =\begin{cases} 
      - \M_{ij}^2/(\M_{ii}^2 M_{jj}) = - g(\M) & pq = \{ii, jj\}  
      \\ 
      2\M_{ij}/(\M_{ii} \M_{jj}) & pq = \{ ij \}.
    \end{cases} 
\end{align}
The expectation term in $T_1$ is
\begin{align}
    \E_\W \left[ \Delta \M_{pq} \right] &= \begin{cases} 
    0 & pq = \{ij\} \\ 
    \gamma^2 \sum_{k=1}^d \M_{kk} = \gamma^2 d & pq = \{ii, jj \}.
    \end{cases}
\end{align}
Given the above formula, we reach the following compact expression for $T_1$:
\begin{align}
    T_1 = - 2 \gamma^2 d g(M). 
\end{align}

The compute $T_2$ we need to compute second-order partial derivatives of $g$ and also estimate the following expectation:
\begin{align} 
\E_\W \left[ \Delta \M_{pq} \Delta \M_{km} \right] & = \gamma^2 \left( \underbrace{\E_\W \left[ [\W \M+ \M \W^\top ]_{pq} [\W \M + \M \W^\top]_{km} \right]}_{K_{pq,km}}\right) + \bigo(\gamma^3).
\end{align}
We now compute $K_{pq,km}$ in the above formula
\begin{align} 
K_{\alpha,\beta} = \begin{cases}
 \sum_{k} \M_{kj}^2 + \sum_{n} \M_{in}^2  & \alpha=\{ij\},\beta=\{ij\} \\
  2 \sum_{k} \M_{kj} \M_{ki}  &  \alpha=\{ij\},\beta=\{ii\} \\
   4\sum_{k} \M_{ki}^2   &  \alpha=\{ii\},\beta=\{ii\} \\
    0   &  \alpha=\{ii\},\beta=\{jj\}
\end{cases}  
\end{align} 

The second-order partial derivatives of $g$ reads as
\begin{align} 
\frac{\partial^2 g(M)}{\partial M_{\alpha}\partial M_{\beta}} = \begin{cases} 
 2 &  \alpha =\{ij\}, \beta =\{ij\} \\
 -2M_{ij} &  \alpha =\{ij\}, \beta =\{ii\}\\ 
+2 M_{ij}^2 &  \alpha =\{ii\}, \beta =\{ii\}\\
M_{ij}^2 & \alpha =\{jj\}, \beta =\{ii\}
\end{cases} 
\end{align} 
Now, we replace the computed partial derivatives and the expectations into $T_2$:
\begin{align}
    T_2 = \sum_{k} \M_{kj}^2 + \sum_{n} \M_{in}^2 -8 \sum_{k} M_{kj}M_{ij} M_{ki} + 4 \sum_{k} M_{ij}^2 M_{ki}^2 +  4 \sum_{k} M_{ij}^2 M_{kj}^2 
\end{align}
Plugging terms $T_1$ and $T_2$ into the Taylor expansion yields
\begin{multline}
       \E_W \left[ g_{ij}(M_+)  - g_{ij}(M)\right]/(\gamma^2) \\ =    \sum_{k} \M_{kj}^2 + \sum_{n} \M_{in}^2 - 2 d g_{ij}(M) -8\sum_{k}  M_{kj}M_{ij} M_{ki} +  4 \sum_{k}  M_{ij}^2 M_{ki}^2  +  4 \sum_{k} M_{ij}^2 M_{kj}^2 + \bigo(\gamma)  
\end{multline}
Summing over $i\neq j$ concludes the proof (note that the diagonal elements are one for the both of matrices $M$ and $M_+$). 
\end{proof}

\begin{proof}[Proof of Proposition~\ref{proposition:spectral_presentation}]
Consider the spectral decomposition of matrix $M$ as $M= U \diag(\lambda) U^\top$, then $M^k = U \diag(\lambda^k) U^\top$. Since $\tr(M^k)$ is equal to the sum of the eigenvalues of $M^k$, we get 
\begin{align} \label{eq:norm_trace}
    \tr(M^k) = \sum_{i=1}^d \lambda_i^k = \| \lambda \|^k_k
\end{align}
for $k=2$ and $k=3$. The sum of the squared norm of the rows in $M$ is equal to the Frobenius norm of $M$. Assuming that the rows have equal norm, we get 
\begin{align}
    \sum_{k=1}^d M_{ik}^2 =  \sum_{i=1}^d \sum_{k=1}^d M_{ik}^2/d = \| M \|_F^2/d = \|\lambda \|_2^2/d.
\end{align}
Therefore, 
\begin{align}
   \tr( \diag(M^2)^2) = \sum_{i=1}^d \left(\sum_{k=1}^d M_{ik}^2 \right)^2 = \| \lambda\|_2^4/d
\end{align}
holds. 

\end{proof}
\begin{proof}[Details of Example~\ref{exm:eigv}]
Under the assumptions stated in Example~\ref{exm:eigv}, we get \begin{align}
    \| \lambda \|_2^2 \approx d^2 - 2 \gamma^2 d, \quad  \| \lambda \|^3_3 \approx d^3 - 3 \gamma^2 d^2, \quad \| \lambda \|^4_2 \approx  d^4 - 4 \gamma^2 d^3
\end{align}
 where the approximations are obtained by a  first-order Taylor approximation of the norms at $\lambda'=(d,0,\dots,0)$, and all small terms $o(\gamma^2)$ are omitted. 
Using the result of Proposition~\ref{proposition:spectral_presentation}, we get 
\begin{align}
   \E \left[ \| M_+ \|_F^2 \right] - \E \left[ \| M \|_F^2 \right] \approx \gamma^2 \delta_F(\lambda) \approx \bigo(-\gamma^4 d^2).
\end{align}
Let $\lambda_+$ be the eigenvalues of the matrix $M_+$, then 
\begin{align}
    & \sum_{i=1}^d \E [\lambda_+^2]_i - \lambda_i^2 = \bigo(-\gamma^4 d^2) \\ 
    \implies& \max_i \E [\lambda_+^2]_i - \lambda_1^2 \leq \bigo(-\gamma^4 d^2) + \sum_{i=2}^d \lambda_i^2 \leq \bigo(-\gamma^4 d^2) + \gamma^4 d  = \bigo(-\gamma^4 d^2).
\end{align}
Let $j = \arg\max_i \E \left[ [\lambda_+]_i^2\right]$. A straight-forward application of Jensen's inequality yields 
\begin{align} 
\E \left[ [\lambda_+]_j \right] \leq \sqrt{\E \left[ [\lambda_+]^2_j \right] } \leq \lambda_1 - \bigo(\gamma^4 d).
\end{align}
Hence the leading eigenvalue of $M_+$ is smaller than the one of $M$. Since the sum of eigenvalues $\lambda_+$ and $\lambda$ are equal, some of the eigenvalues $\lambda_+$ are greater than those of $\lambda$ (in expectation) to compensate $\E  [\lambda_+]_j < \lambda_1$. 
\end{proof}

\begin{proof}[Details of Example~\ref{exm:regualrity}]
Invoking Prop.~\ref{proposition:spectral_presentation}, we get 
\begin{align}
    \E \left[ \tr(M(H)^3) \right] = \| \lambda \|^3, \quad \E \left[ \diag(M(H)^2)^2 \right] = \| \lambda \|^4_2/d,
\end{align}
where $\lambda \in \R^d$ contains the  eigenvalues of $M(H)$. Since $H \in \Hs$, $\| \lambda \|_1 = d$. If the regularity constant is greater than or equal to one, then
\begin{align} \label{eq:norm341}
    \| \lambda \|_3^3 \leq \| \lambda \|^4_2/d = \| \lambda \|^4_2/\| \lambda \|_1.
\end{align}
A straightforward application of Cauchy-Schwartz yields: 
\begin{multline}
    \| \lambda \|^4_2 =  \sum_{i=1}^d \sum_{j=1}^d \lambda_i^2 \lambda_j^2 = \sum_{i=1}^d \sum_{j=1}^d (\lambda_i \lambda_j )^{\sfrac{1}{2}}(\lambda_i \lambda_j )^{\sfrac{3}{2}} \\ \leq \sqrt{\left( \sum_{i,j} \lambda_i \lambda_j\right) \left( \sum_{i,j} \lambda_i^3 \lambda_j^3\right)} =\| \lambda \|_1 \| \lambda \|_3^3 
\end{multline}
The above result together with inequality~\ref{eq:norm341} yields that 
\begin{align}
     \| \lambda \|_3^3 = \| \lambda \|^4_2/d = \| \lambda \|^4_2/\| \lambda \|_1.
\end{align}
Finally, the above equality is met only when all non-zero eigenvalues are equal.

\end{proof}

\begin{proof}[Details of Example~\ref{exm:regularity_variation}]
Since $\lambda_1 = \Theta(d^{\beta})$ and $\lambda_{i>1} = o(d^{\beta})$, we get 
\begin{align}
    \| \lambda \|_3^3 = \Theta(d^{3\beta}), \quad \| \lambda \|_2^2 = \Theta(d^{2\beta}).
\end{align}
Thus, Prop.~\ref{proposition:spectral_presentation} yields 
\begin{align}
    \E \left[ \tr(M^3)\right] =  \Theta(d^{3\beta}), \quad \E \left[  \tr(\diag(M^2)^2)\right] = \| \lambda \|_2^4/d =  \Theta(d^{4\beta-1})
\end{align}
Therefore, 
\begin{align}
    \alpha = \lim_{d\to \infty} \frac{ \E \left[  \tr(\diag(M^2)^2)\right]}{\E \left[ \tr(M^3)\right]} = \bigo(d^{\beta-1}) = 0.
\end{align}
As a result, $\alpha$ is less than 0.9 for sufficiently large $d$. 
\end{proof}
\begin{proof} [Proof of Lemma~\ref{lemma:seq}]
The proof is based on a contradiction. Suppose that there exits a finite $n$ such that $y_n<\alpha$. Since the sequence is non-increasing, $y_m<\alpha$ for holds for all $m>n$. This yields
\begin{align}
    \lim_{N\to \infty}  \sum_{k=1}^N y_k/N &=  \lim_{N\to \infty} \left( \sum_{k>n}^N y_k/N+ \sum_{k\leq n} y_k/N \right) \\ 
    & < \frac{(N-n)}{N} \alpha + \lim_{N \to \infty} \sum_{k\leq n} y_k/N \\ 
    & = \frac{(N-n)}{N} \alpha,
\end{align}
where we used the fact that all $y_k$ are bounded. The above result contradicts the fact that $\lim_{n\to \infty} \sum_{k=1}^N y_k/N > \alpha$. 
\end{proof}
\section{Analysis for Vanilla Linear Networks.}
In this section, we prove Lemma~\ref{lemma:rank_vanishing} that states the rank vanishing problem for vanilla linear networks. Since the proof relies on existing results on products of random matrices (PRM) \cite{bougerol2012products}, we first shortly review these results. Let $T$ be the set of $d \times d$ matrices. Then, we review two notions for $T$: contractiveness and strong irreducibility. 
\begin{definition} [Contracting set~\cite{bougerol2012products}]
$T$ is contracting if there exists a sequence $\{M_n \in T, n \geq 0 \}$ such that $M_n/\|M_n\|$ converges to a rank one matrix.
\end{definition}
\begin{definition}[Invariant union of proper subspaces~\cite{bougerol2012products}] Consider a family of finite proper linear subspace $V_1, \dots, V_k \subset \R^{d}$. The union of these subspaces is invariant with respect to $T$, if $M v \in V_1 \; \text{or} \; V_2 \; \text{or} \dots \; \text{or} \; V_k$ holds for $\forall v \in V_1 \; \text{or} \; V_2 \; \text{or} \dots \; \text{or} \; V_k $ and $\forall M \in T$. 
\end{definition}

\begin{example} \label{exm:invariant_union_subspace}
Consider the following sets
\[
T = \left( \begin{bmatrix} 
0 & 1 \\ 
1 & 0 
\end{bmatrix} \right), \quad V_1 = \left( \text{span}(\underbrace{[0,1]}_{v_1}) \right), \quad V_2 = \left( \text{span}(\underbrace{[1,0]}_{v_2}) \right);
\]
then, union of $V_1$ and $V_2$ is invariant with respect to $T$ because $\alpha T v_1 \in V_2$ and $\alpha T v_2 \in V_1$ hold for $\alpha \neq 0$. 
\end{example}

\begin{definition}[Strongly irreducible set~\cite{bougerol2012products}]
The set $T$ is strongly irreducible if there does not exist a finite family of proper linear subspaces of   $\R^d$ such that their union is invariant with respect to $T$. 
\end{definition}
For example, the set $T$ defined in Example~\ref{exm:invariant_union_subspace} is not strongly irreducible. 

\begin{lemma}[Thm 3.1 of ~\cite{bougerol2012products}]  \label{lemma:prm}
Let $W_1, W_2, \dots $ be random $d\times d$ matrices drawn independently from a distribution $\mu$. Let $B_n = \prod_{k=1}^n W_k $. If the support of $\mu$ is strongly irreducible and contracting, then any limit point of $\{ B_n/\|B_n\| \}_{n=1}^\infty$ is a rank one matrix almost surely. 
\end{lemma}
This result allows us to prove Lemma~\ref{lemma:rank_vanishing}. 

\begin{proof}[Proof of Lemma~\ref{lemma:rank_vanishing}]

Recall the structure of the random weight matrices as  $\widehat{W}_k = I + \gamma W_k$ where the coordinates $W_k$ are i.i.d. from (a.) standard Gaussian, (b.) uniform$[-\sqrt{3}, \sqrt{3}]$ (i.e. with variance 1). One can readily check that for the Gaussian weights, the contracting and strong irreducibility hold and one can directly invoke the result of lemma~\ref{lemma:prm} to get part (a.) of Lemma~\ref{lemma:rank_vanishing}. Now, we prove part (b.).  Let $m$ be a random integer that obeys the law $p(m=k) = 2^{-k}$. Given the random variable $m$, we define the random matrix $Y = \prod_{k=1}^{m} \widehat{W}_k$ and use the notation $\mu'$ for its law. Let $\{ Y_i = \prod_{j=1}^{m_i} \widehat{W}_k \}_{i=1}^k $ be drawn i.i.d. from $\mu'$.  Then, $C_k := Y_k \dots Y_2 Y_1$ is distributed as $B_{\ell_k} := \widehat{W}_{\ell_k} \dots \widehat{W}_2 \widehat{W}_1 $ for $\ell_k = \sum_{i=1}^k m_i$. We prove that every limit point of  $\{ C_k/\| C_k \| \} $ converges to a rank one matrix, which equates the convergence of limit points of $\{ B_{\ell_k}/\| B_{\ell_k} \|\}$ to a rank one matrix. To this end, we prove that the support of $\mu'$ denoted by $T_{\mu'}$ is contractive and strongly contractive. Then, Lemma~\ref{lemma:prm} implies that the limit points of $\{C_k/\|C_k\|\}$ are rank one. \\
\textbf{Contracting.} Let $e_1\in \R^d$ be the first standard basis vector. Since $A_n := (\I+\gamma e_1 e_1^\top )^n \in T_{\mu'}$ and its limit point $\{ A_n/\|A_n \|\}$ converges to a rank one matrix, $T_{\mu'}$ is contractive. \\
\textbf{Strong irreduciblity.}
Consider an arbitrary family of linear proper subspace of $\R^d$ as $\{ V_1, \dots, V_q\}$. Let $v$ be an arbitrary unit norm vector which belongs to one of the subspaces $\{ V_i \}_{i=1}^q$. Given $v$, we define an indexed family of matrices $\{ M_{\alpha} \in T_{\mu'} | \alpha \in \R^d, | \alpha_i | \leq 1 \}$ such that 
\begin{align}
    M_{\alpha} = I + \frac{\gamma}{d} \sum_{i=1}^d \alpha_i e_i v^\top \in T_{\mu'}, 
\end{align}
where $e_i$ is the i-th standard basis\footnote{Notably, the absolute value of each element of $\frac{1}{d} \sum_{i=1}^d \alpha_i e_i v^\top$ is less than $1$, hence this matrix belongs to the support of $\mu$. }.  Then, we get 
\begin{align}
    M_\alpha v = v + \frac{\gamma}{d} \sum_{i=1}^d \alpha_i e_i.
\end{align}

Therefore, $\{ M_\alpha v | |\alpha_i| \leq 1\}$ is not contained in any union of finite proper $(m < k)$-dimensional linear subspace of $\R^d$, hence  $T_{\mu'}$ is strongly irreducible.

\end{proof}
\section{Details: Pretraining algorithm} \label{sec:details_pretraining}
In Section \ref{sec:pretraining}, we introduced a pre-training method that effectively obtains a better optimization performance compared ot \bn{}. In this section, we provide more details about the pre-training step. Recall $X \in \R^{d \times N}$ is a minibatch of $d$-dimensional inputs of size $N$. Let $H_L(X) \in \R^{d \times N}$ be the hidden representation of input $X$ in the last layer of a MLP. Using gradient descent method, we optimize $r(H_L(X))$ --with respect to the parameters of networks-- over different minibatches $X$. Algorithm~\ref{alg:pretraining} presents our pretraining method. As can be seen, the procedure is very simple.

\begin{algorithm}[h!]
\begin{algorithmic}[1]
   \STATE \textbf{Input:} Training set $S$, a network with parameters $\Theta$ and $L$ layers, and constant $N, M$, and $T$
   \FOR{$k=1,2,\dots, M$} 
        \STATE Draw minibatch $X_k$ of size $N$ i.i.d. from $S$ 
        \FOR{$t=1,2,\dots, T$}
            \STATE Take one GD step on $r(H_L(X_k))$ w.r.t $\Theta$. 
        \ENDFOR
   \ENDFOR
   \STATE \textbf{return} $\Theta$. 
\end{algorithmic}
\caption{Pretraining}
\label{alg:pretraining}
\end{algorithm}

\section{Details: Why the rank matters for gradient based learning.}\label{sec:details_why}
We now provide an intuitive explanation of why rank one hidden representations prevent randomly initialized networks from learning. Particularly, we argue that these networks essentially map all inputs to a very small subspace\footnote{A single line in $\mathbb{R}^d$ in the extreme case of rank one mappings} such that the final classification layer can no longer disentangle the hidden representations. As a result, the gradients of that layer also align, yielding a learning signal that becomes \textit{independent} of the input.

To make this claim more precise, consider training the linear network from Eq.~\eqref{eq:linear_net} on a dataset $X\in\mathbb{R}^{d\times N}$, where $x_i \in \mathbb{R}^{d}$ with $d_{out}$ targets $y_i \in \mathbb{R}^{d_{out}}, i=1,\ldots,N$. Each column $\widehat{H}^{(\gamma)}_{L,i}$ of the hidden representations in the last hidden layer $\widehat{H}^{(\gamma)}_L$ is the latent representation of datapoint $i$, which is fed into a final classification layer parametrized by $W_{L+1} \in \mathbb{R}^{d_{out}\times d}$. We optimize $\mathcal{L}(\mathbf{W})$, where $\mathbf{W}$ is a tensor containing all weights $W_1,\ldots,W_{L+1}$ and $\widehat{H}^{(\gamma)}_{L,i}$ is a function of $W_1,\ldots,W_{L}$ (as detailed in Eq.~\eqref{eq:linear_net}:

\begin{equation}\label{eq:opt_problem}
\min_{\mathbf{W}}\mathcal{L}(\mathbf{W}) =\sum_{i=1}^N  \underbrace{\ell\left(y_i,W_{L+1}\widehat{H}^{(\gamma)}_{L,i}(W_1,...,W_{L}) \right)}_{:=\mathcal{L}_i(\mathbf{W})},
\end{equation}
and $\ell:\mathbb{R}^{d_{out}}\rightarrow\mathbb{R}^+$ is a differentiable loss function. Now, if the the hidden representations become rank one (as predicted by Lemma \ref{lemma:rank_vanishing} and Fig. \ref{fig:vanilla_bn_rank}), one can readily check that the stochastic gradients of any neuron $k$ in the last linear layer, i.e., $\nabla_{W_{L,[k,:]}} \mathcal{L}_i(\mathbf{W})= (\nabla \ell_i)_k \widehat{H}^{(\gamma)}_{L,i}$, align for both linear and ReLU networks.

\begin{proposition}\label{prop:grad_direction_formal} Consider a network with rank one hidden representations in the last layer $\widehat{H}^{(\gamma)}_L(W_1,...,W_{L})$, then for any neuron $k$ and any two datapoints $i,j$ with non-zero errors $\mathcal{L}_i$ and $\mathcal{L}_j$ we have
\begin{equation}
\nabla_{W_{L+1,[k,:]}} \mathcal{L}_i(\mathbf{W})=\underbrace{\frac{c(\nabla \ell_i)_k}{(\nabla \ell_j)_k}}_{\in \mathbb{R}}\nabla_{W_{{L+1},[k,:]}} \mathcal{L}_j(\mathbf{W})
\end{equation} 
$\forall i,j$. That is, all stochastic gradients of neuron $k$ in the final classification layer align along one single direction in $\mathbb{R}^d$.
\end{proposition}

\begin{proof}
The result follows directly from a simple application of the chain rule
\begin{equation}
\begin{aligned}
\frac{\partial \mathcal{L}_i(\Wm) }{\partial W_{L+1}}=\frac{\partial \ell(\y_i,W_{L+1}\widehat{H}^{(\gamma)}_{L,i})}{\partial W_{L+1}\widehat{H}^{(\gamma)}_{L,i}} \frac{\partial W_{L+1}\h_{L,i}}{\partial W_{L+1}}=\nabla_{W_{L+1}\widehat{H}^{(\gamma)}_{L,i}} \ell(\y_i,W_{L+1}\h_{L,i})(\widehat{H}^{(\gamma)}_{L,i})^\intercal \\
=\begin{bmatrix} \nabla \ell_{i,1} \widehat{H}^{(\gamma)}_{L,i,1},\ldots,\nabla \ell_{i,1} \widehat{H}^{(\gamma)}_{L,i,d}\\
\ddots\\
\nabla \ell_{i,d_{out}} \widehat{H}^{(\gamma)}_{L,i,1},\ldots,\nabla \ell_{i,d_{out}} \widehat{H}^{(\gamma)}_{L,i,d}
\end{bmatrix}\; \in \mathbb{R}^{d_{out}\times d}
\end{aligned}
\end{equation}

The same holds for $j$. Now, if $\widehat{H}^{(\gamma)}_{L,i}=c\widehat{H}^{(\gamma)}_{L,i}, c \in \mathbb{R}  \setminus \{0\} $ then

\begin{equation*}
 \left(   \frac{\partial \mathcal{L}_i(\Wm) }{\partial W_{L+1}}\right)_{k,:}=\underbrace{c\frac{\nabla \ell_{i,k}}{\nabla \ell_{j,k}}}_{\in \mathbb{R}}\left(   \frac{\partial \mathcal{L}_j(\Wm) }{\partial W_{L+1}}\right)_{k,:}
\end{equation*}
\end{proof}

To validate this claim, we again train CIFAR-10 on the VGG19 network from Figure \ref{fig:broken_bn} (top).

\begin{figure}[H]
 \centering 
          \begin{tabular}{c@{}c@{}}
               \includegraphics[width=0.4\linewidth]{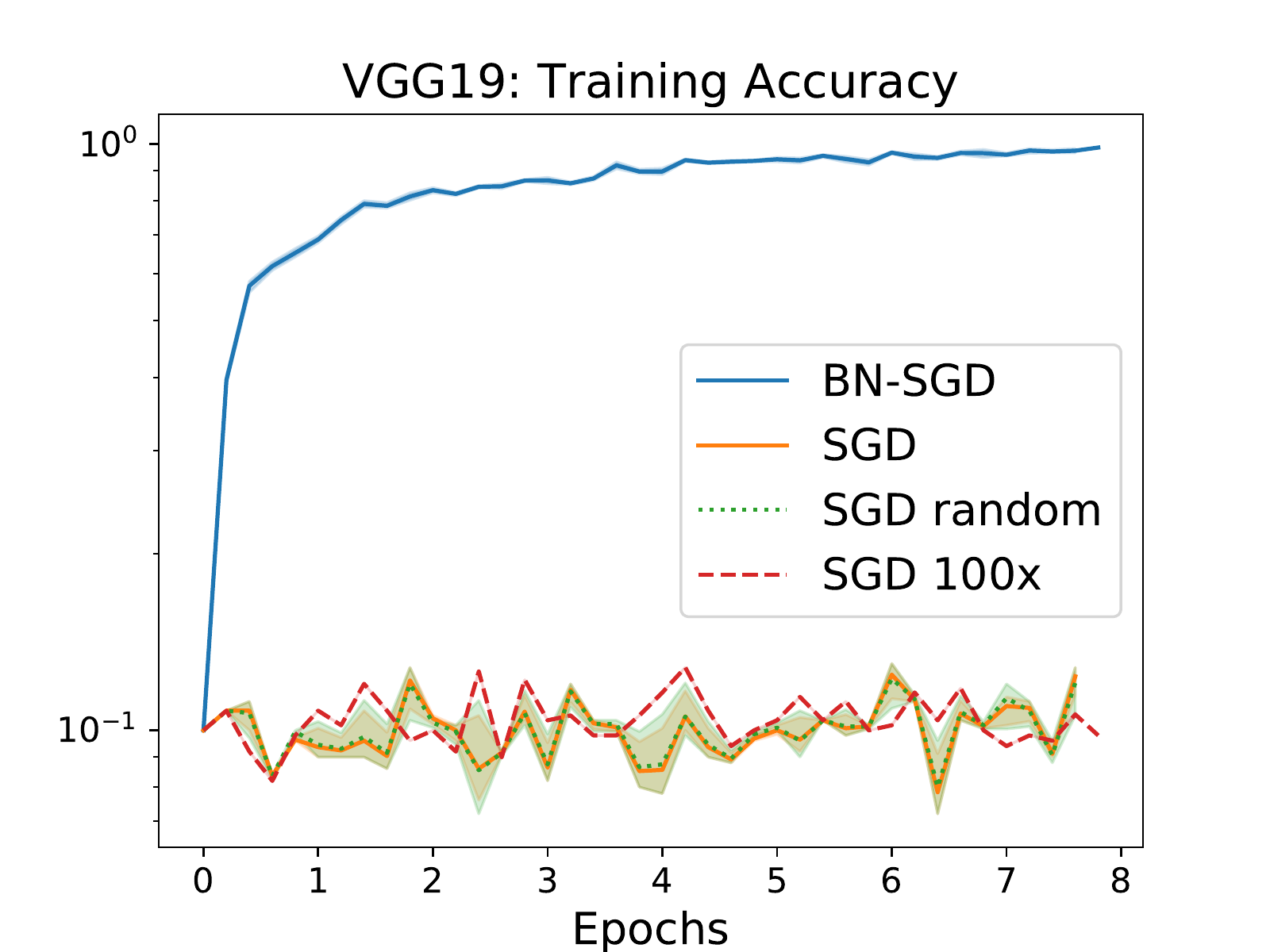}  &    \includegraphics[width=0.4\linewidth]{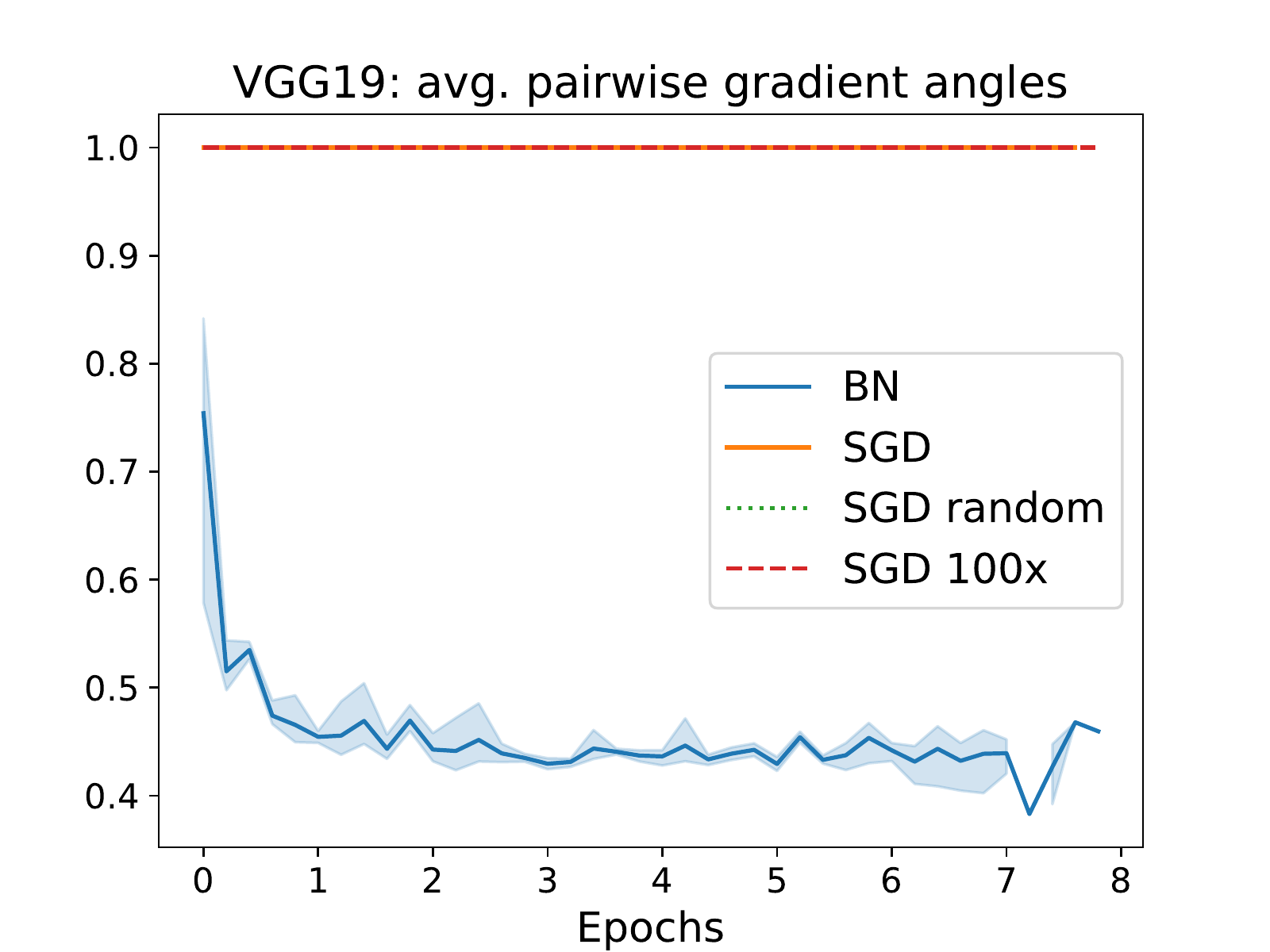}
           
 	  \end{tabular}
          \caption{ \footnotesize{\textbf{Directional gradient vanishing} CIFAR-10 on a VGG19 network with BN, SGD, SGD with 100x learning rate and SGD on random data. Average and 95\% confidence interval of 5 independent runs.
          }}
          \label{fig:vgg}
\end{figure}

As expected, the network shows perfectly aligned gradients without \bnp~(right hand side of Fig. \ref{fig:vgg}), which renders it un-trainable. In a next step, we replace the input by images generated randomly from a uniform distribution between 0 and 255 and find that SGD takes almost the exact same path on this data (compare log accuracy on the left hand side). Thus, our results suggest that the commonly accepted vanishing gradient \textit{norm} hypothesis is not descriptive enough since SGD does not take small steps into the \textit{right} direction- but into a \textit{random} one after initialization in deep neural networks. As a result, even a 100x increase in the learning rate does not allow training. We consider our observation as a potential starting point for novel theoretical analysis focusing on understanding the propagation of information through neural networks, whose importance has also been highlighted by \cite{bjorck2018understanding}.

\section{Additional Experiments} \label{sec:additional_experiments}
\paragraph{Outperforming \bn{}}

\begin{figure}[H]
    \centering
    \begin{tabular}{c@{}c@{}c@{}}
     \includegraphics[width=0.32\textwidth]{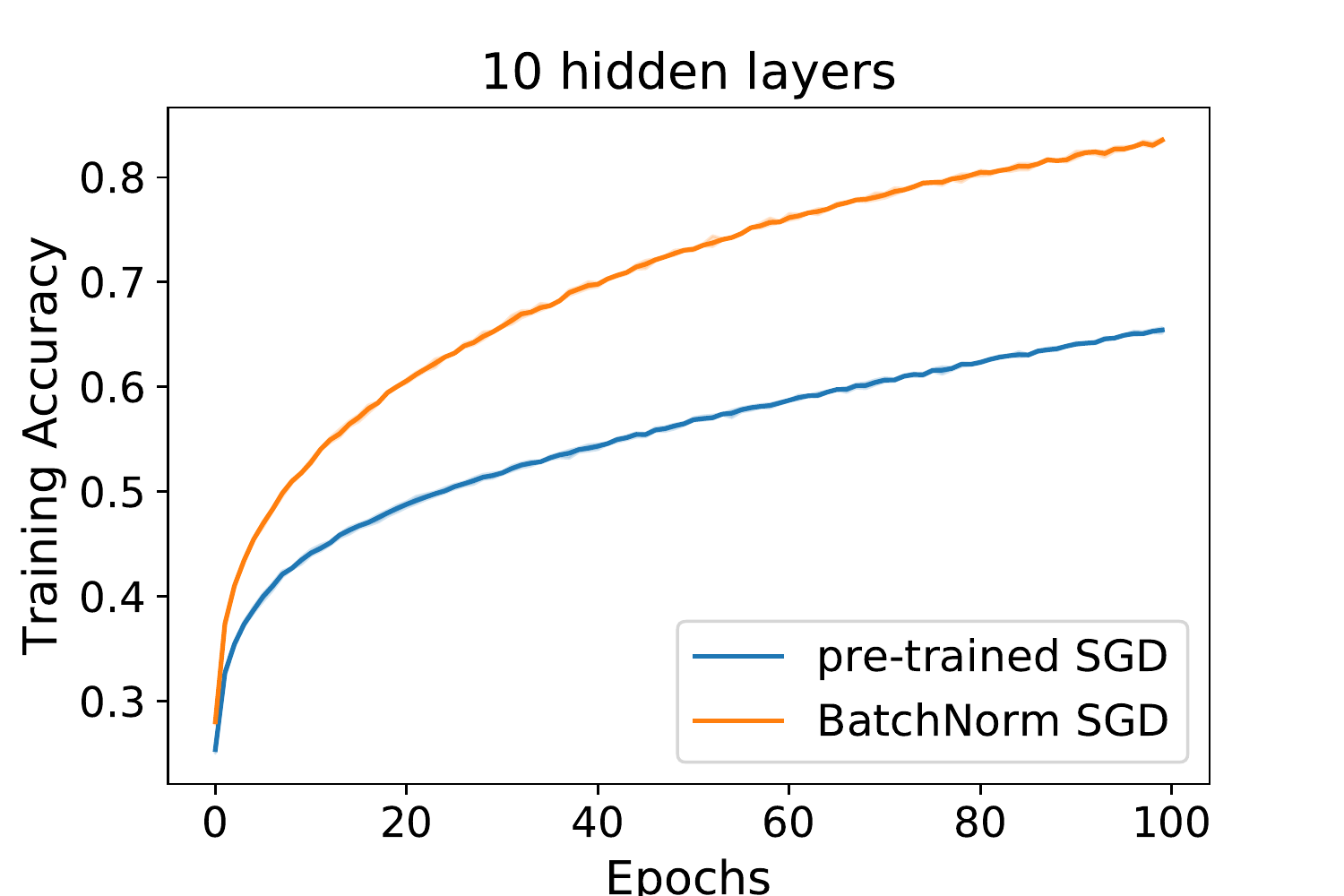} & \includegraphics[width=0.32\textwidth]{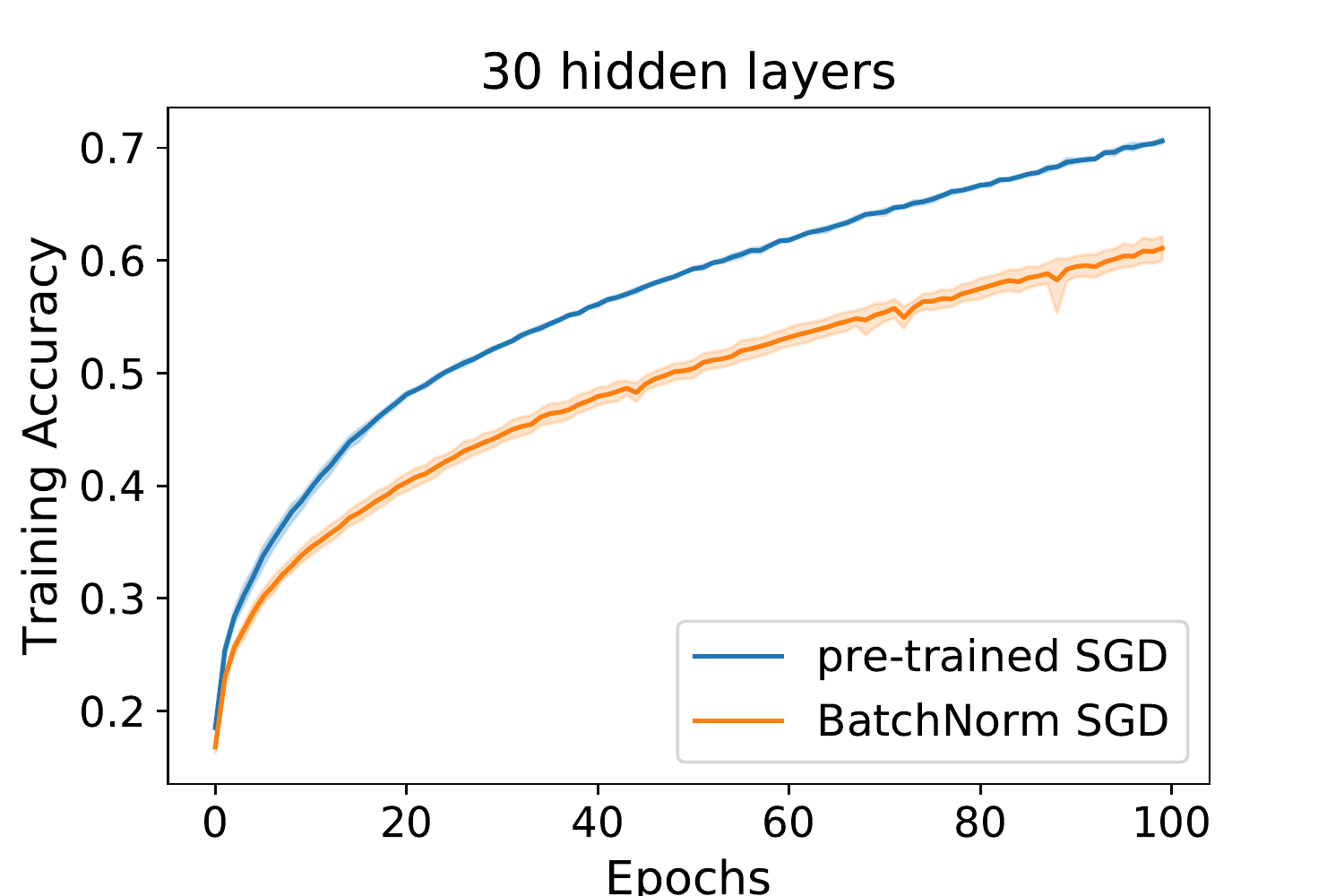} & \includegraphics[width=0.32\textwidth]{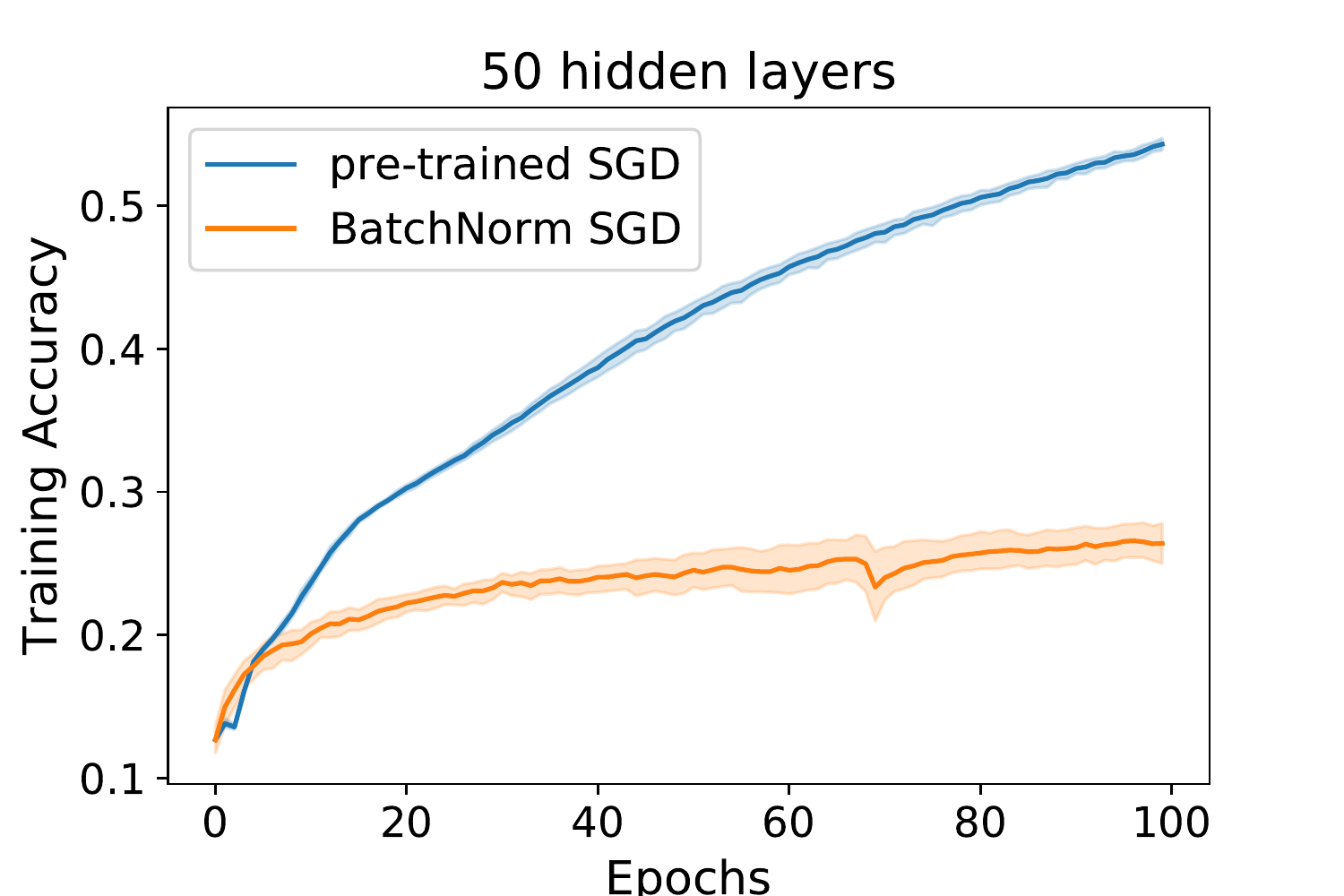} \\
     
     \includegraphics[width=0.32\textwidth]{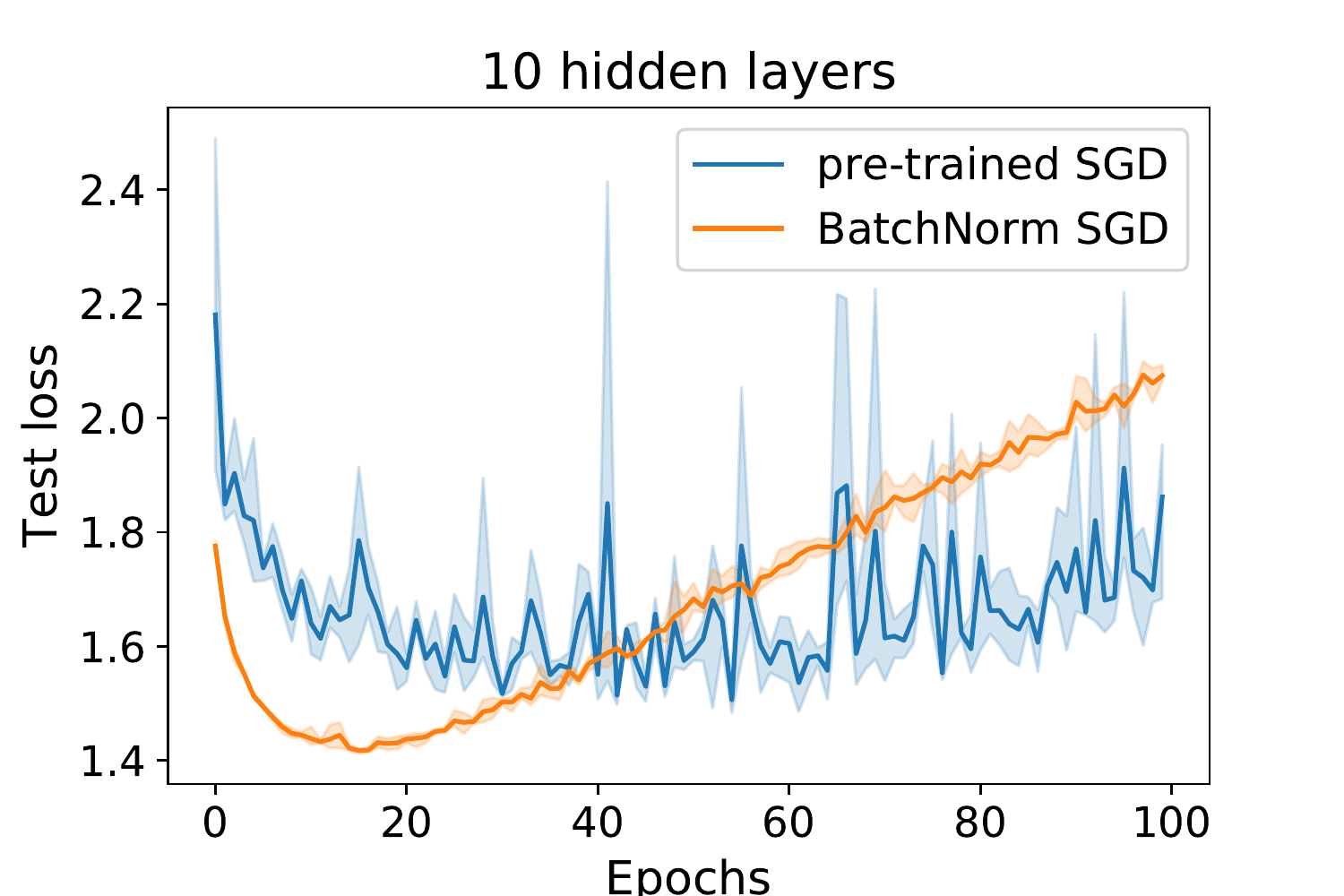} & \includegraphics[width=0.32\textwidth]{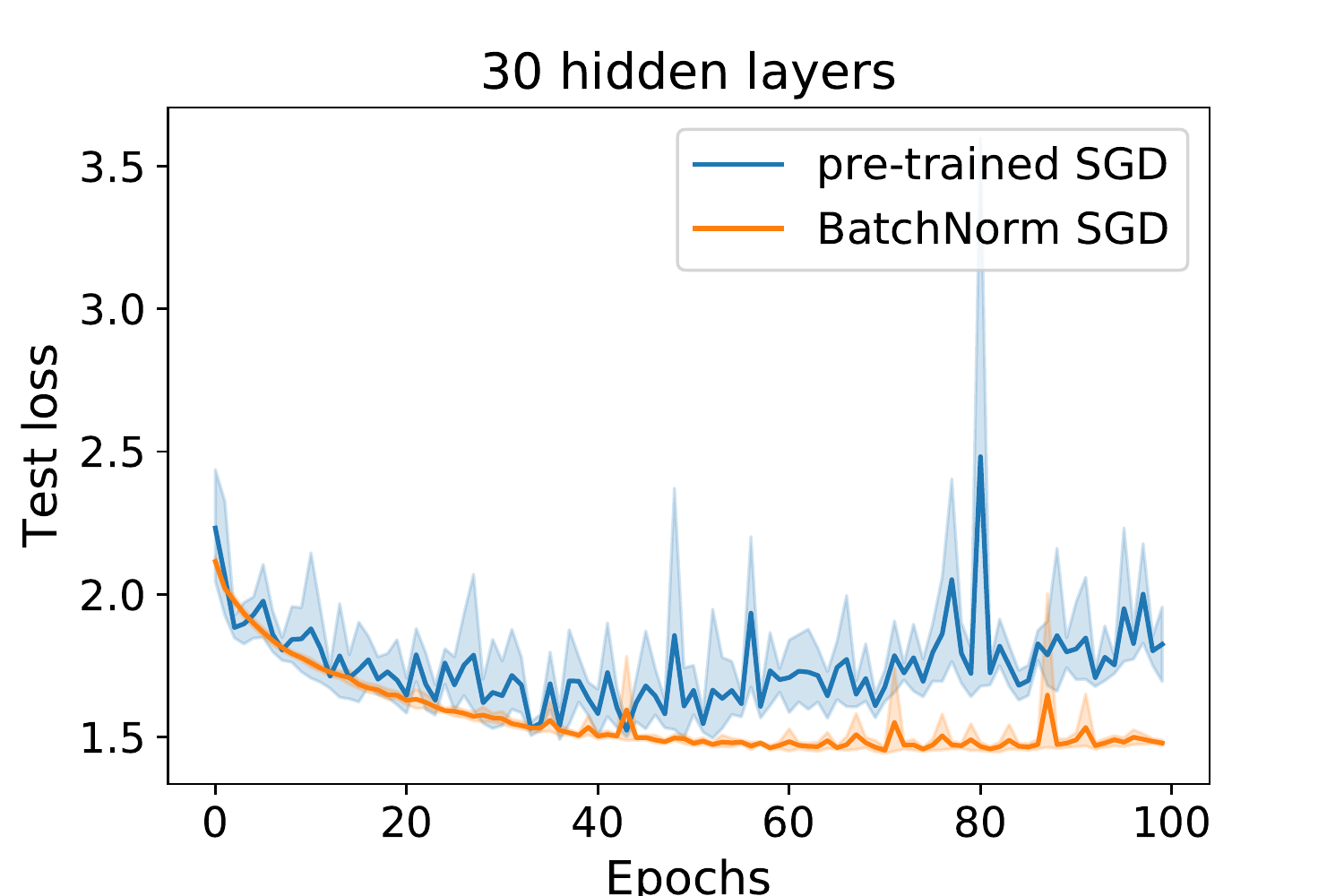} & \includegraphics[width=0.32\textwidth]{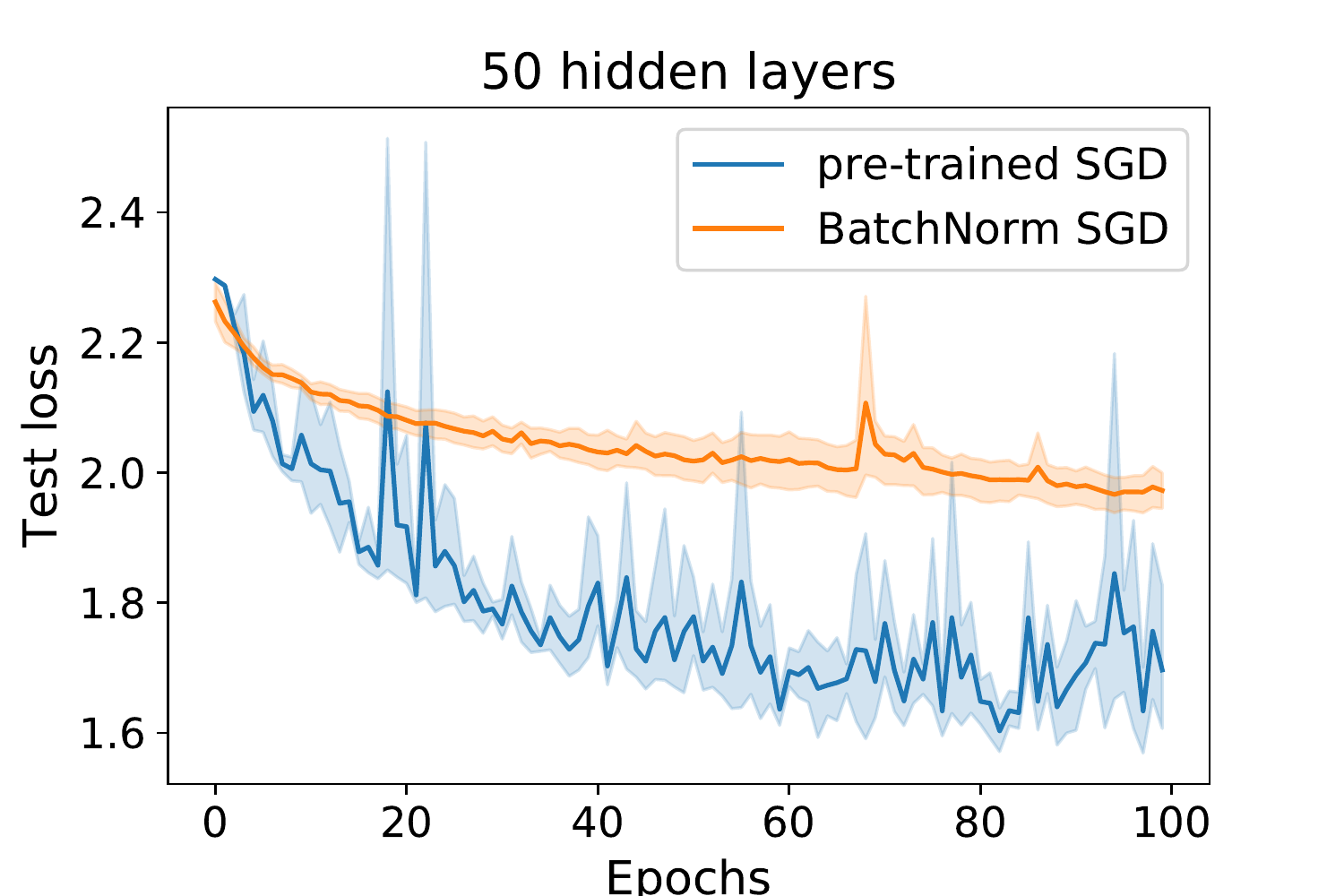}
    \end{tabular}
     \caption{\footnotesize{CIFAR-10: Same setting as Fig.\ref{fig:outperform_bn_cifar} but now showing accuracy and test loss}}
    \label{fig:beating_bn_app}

\end{figure}

The following Figure shows the result of the experiment of Fig.~\ref{fig:outperform_bn_cifar} that is repeated for FashionMNIST dataset.

\begin{figure}[H]
    \centering
    \begin{tabular}{c@{}c@{}c@{}}
     \includegraphics[width=0.32\textwidth]{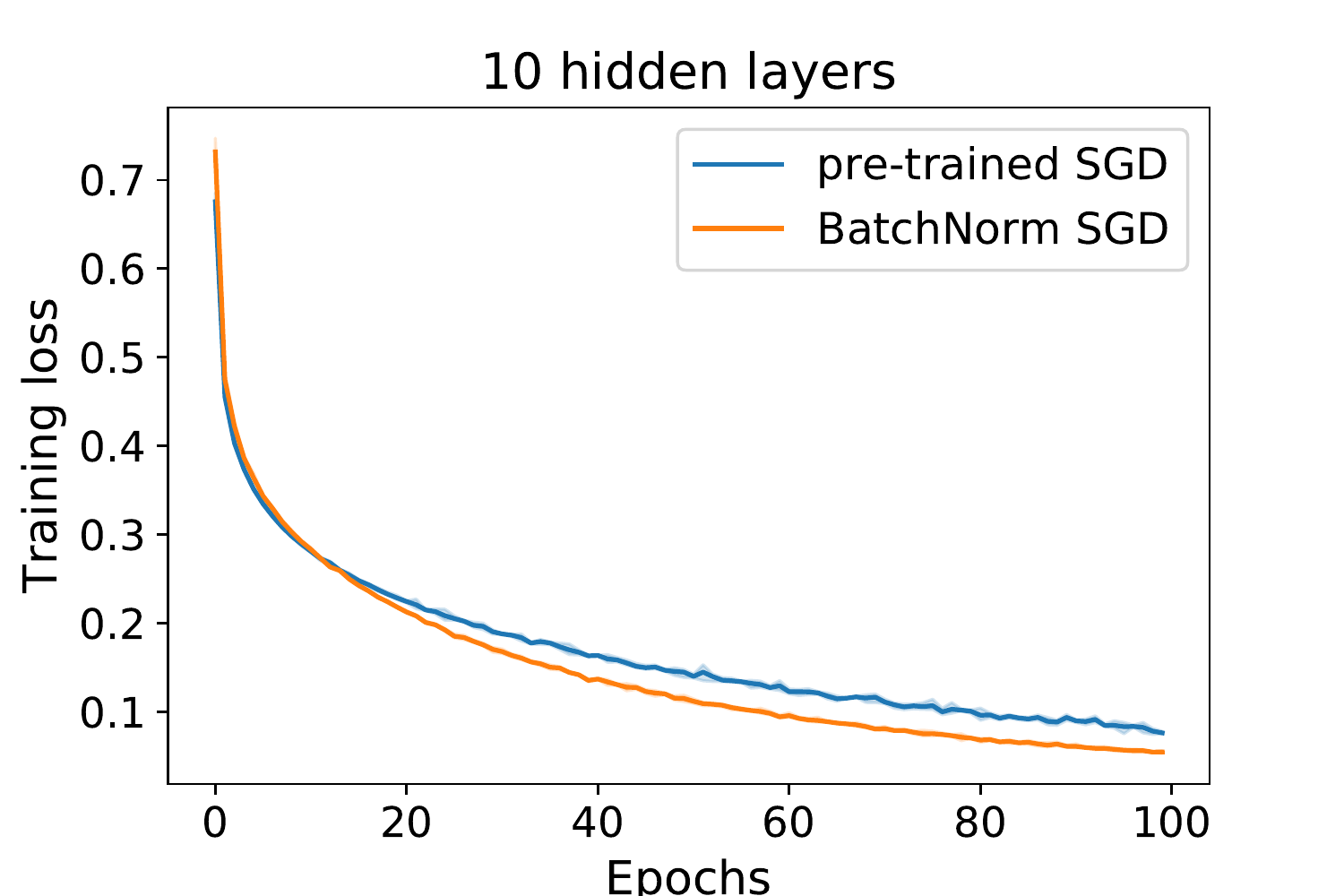} & \includegraphics[width=0.32\textwidth]{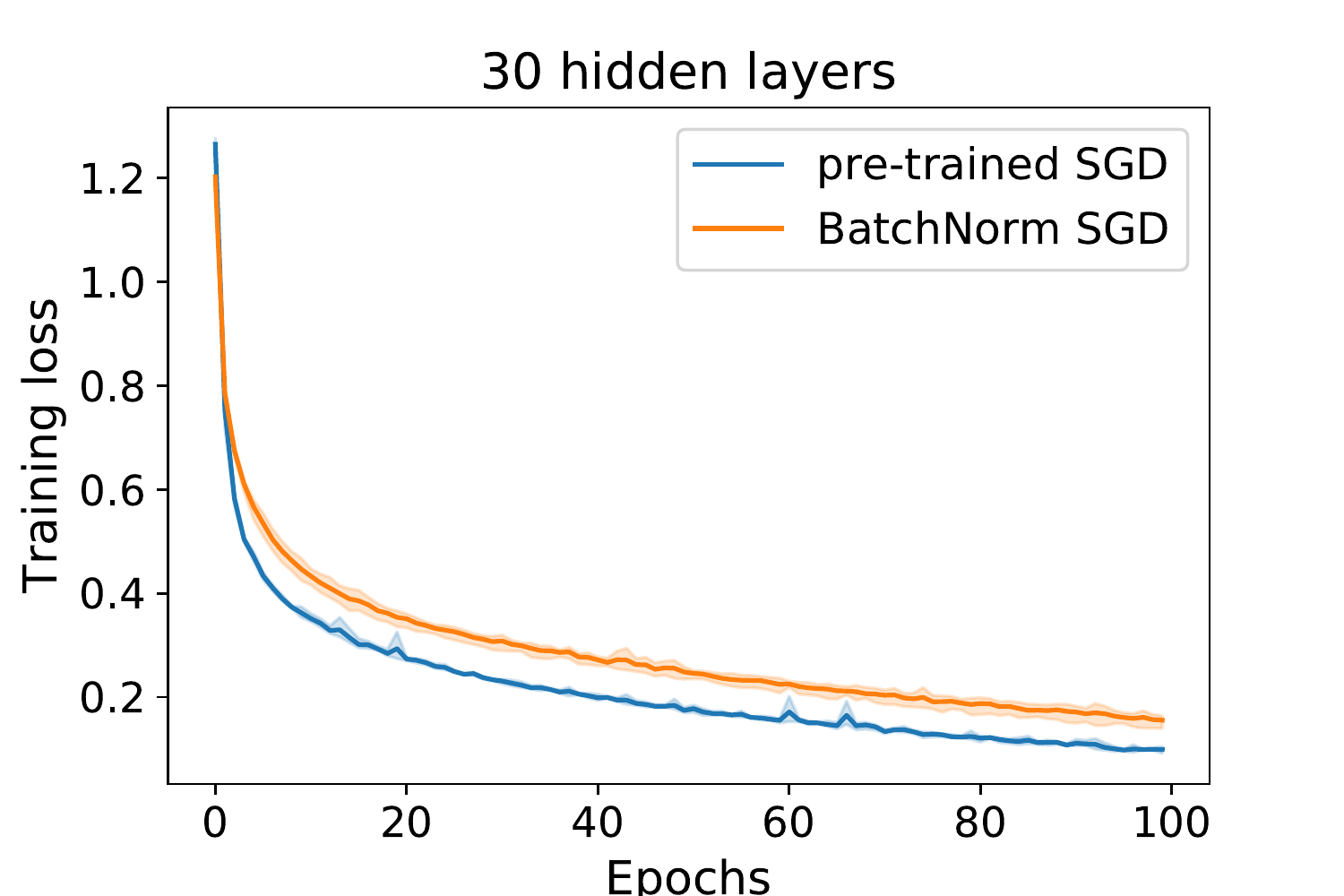} & \includegraphics[width=0.32\textwidth]{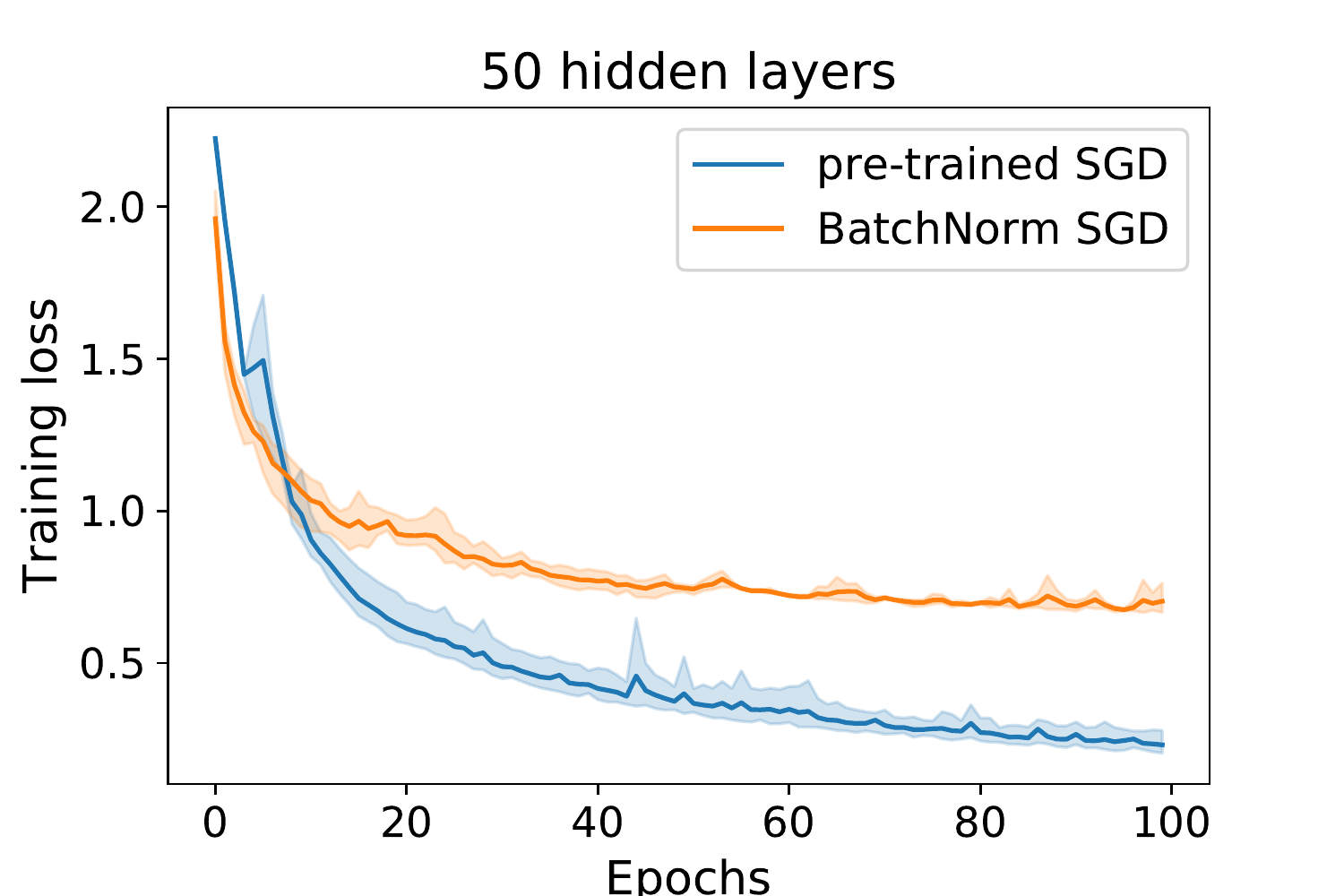} \\
 \includegraphics[width=0.32\textwidth]{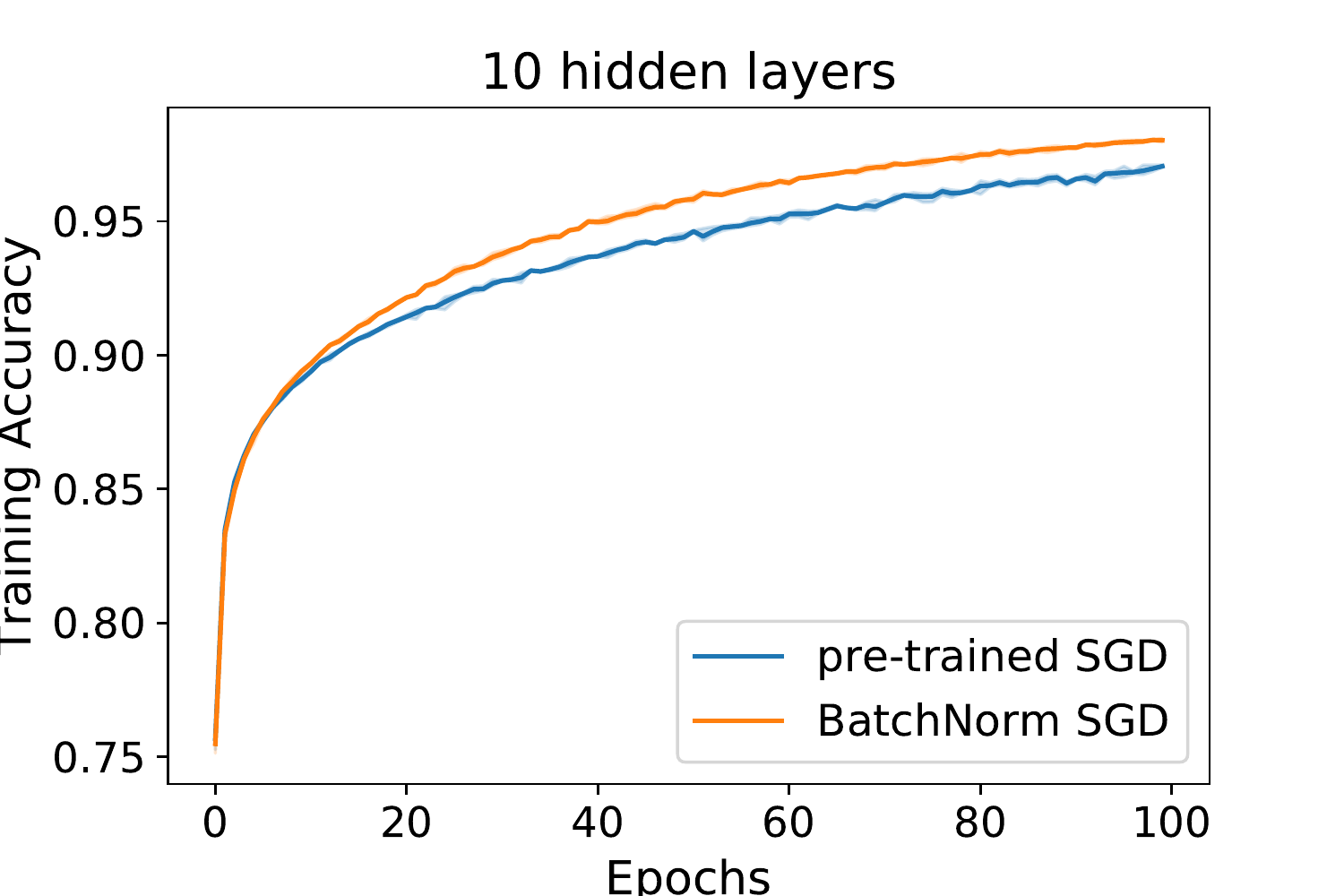} & \includegraphics[width=0.32\textwidth]{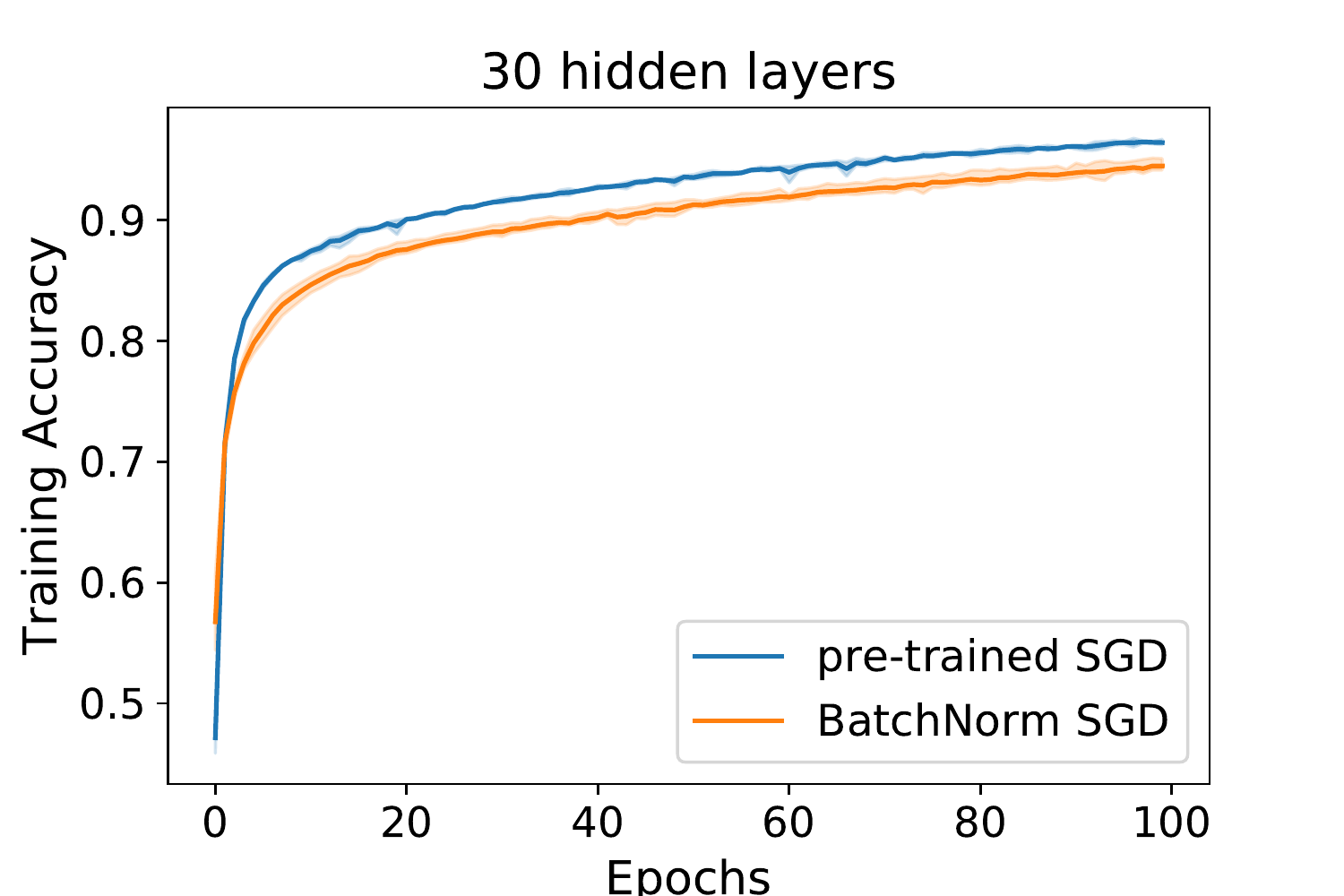} & \includegraphics[width=0.32\textwidth]{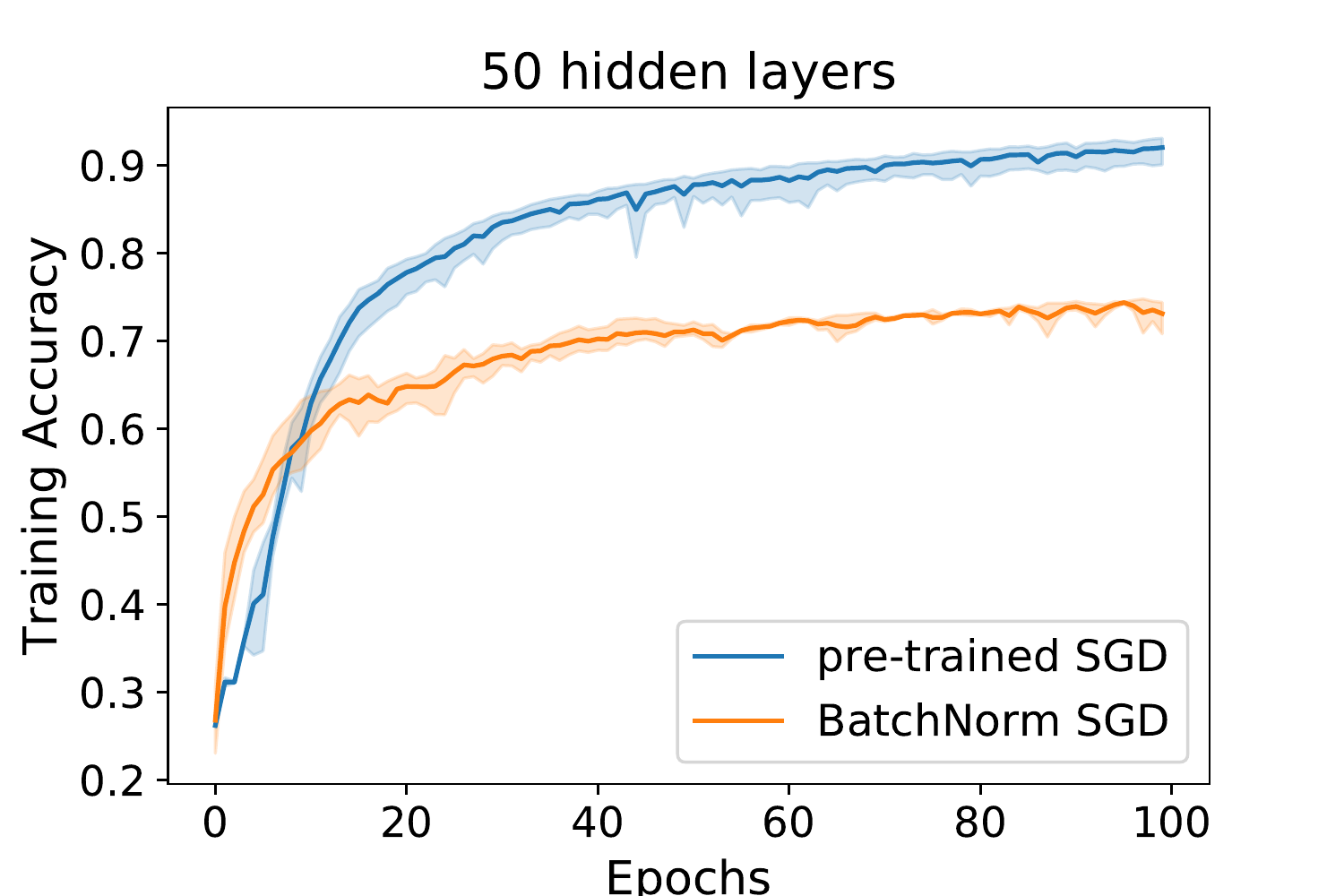} 
    \end{tabular}
     \caption{\footnotesize{Results of Fig.\ref{fig:outperform_bn_cifar} for FashionMNIST}}
    \label{fig:beating_bn_FM}

\end{figure}

\paragraph{Breaking \bn{}}
In the following result, we repeated the experiment of Fig.~\ref{fig:broken_bn} for ResNets.

\begin{figure}[H]
 \centering 
          \begin{tabular}{c@{}c@{}}
          \adjincludegraphics[width=0.3\linewidth, trim={30pt 20pt 30pt 20pt},clip]{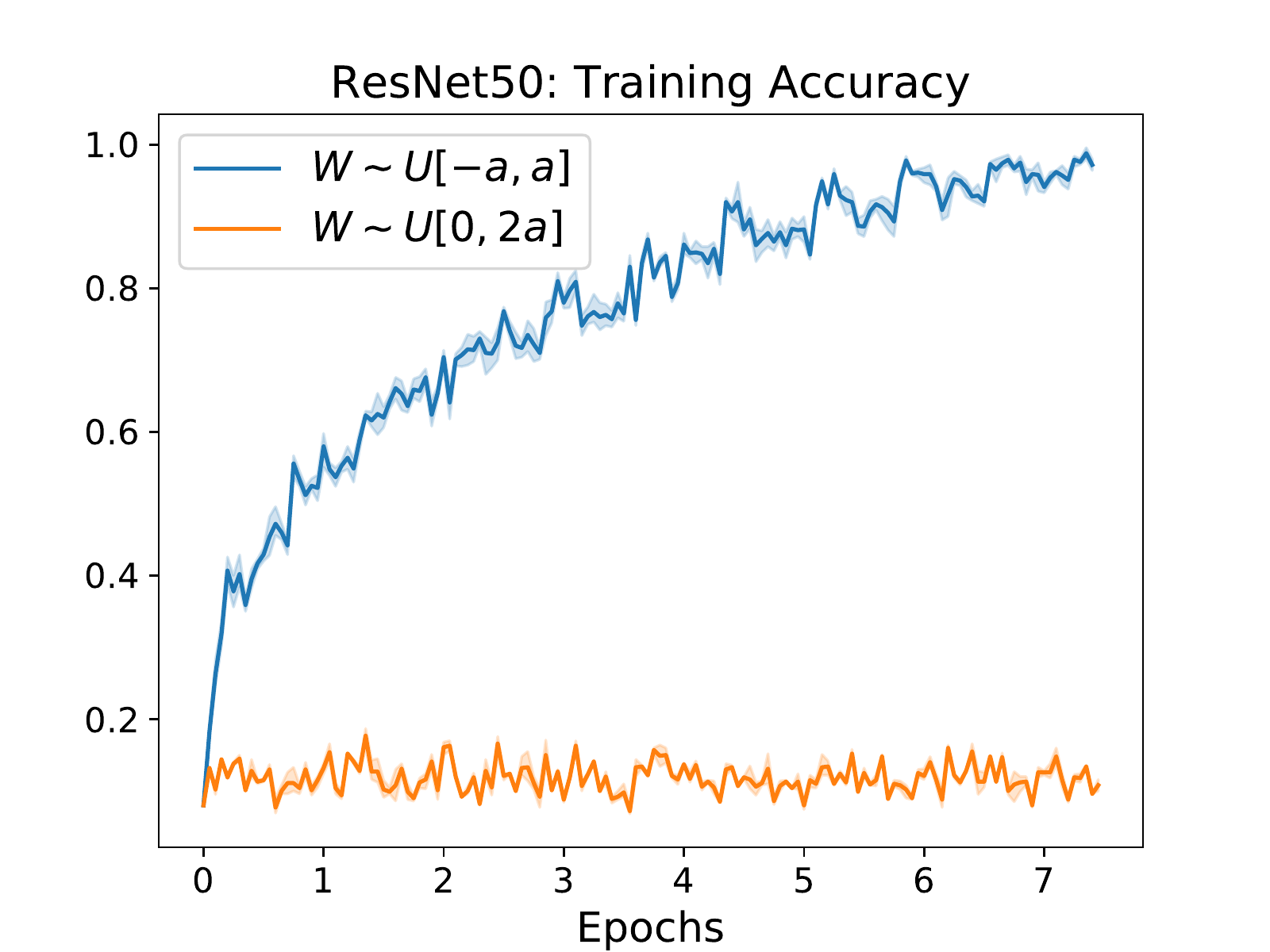}  &    \adjincludegraphics[width=0.3\linewidth, trim={30pt 20pt 30pt 20pt},clip]{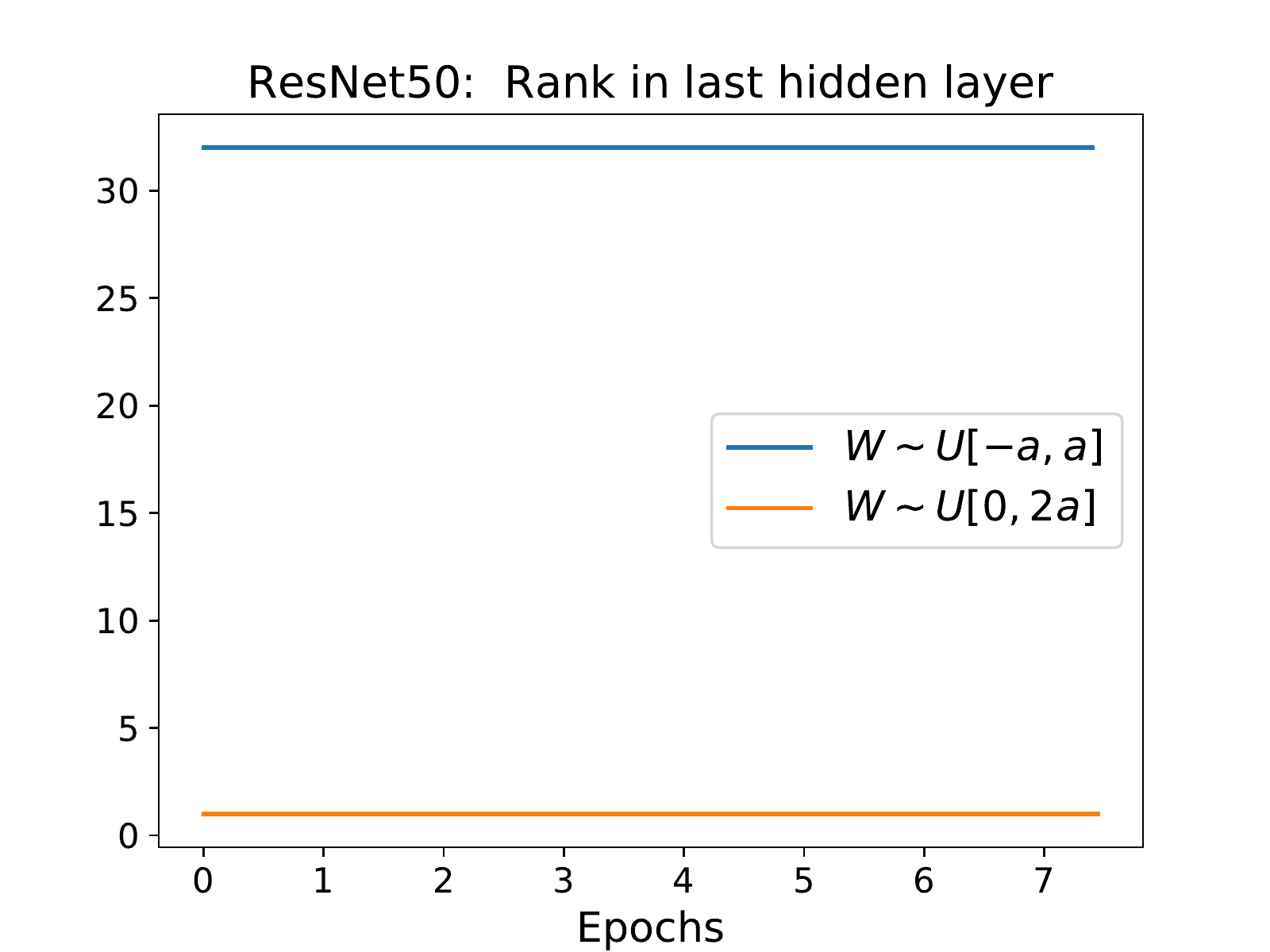}
           
 	  \end{tabular}
          \caption{ \footnotesize{\textbf{Breaking Batchnorm:} CIFAR-10 on a ResNet-50 with standard PyTorch initialization as well as a uniform initialization of same variance in $\mathbb{R}^+$. Average and 95\% confidence interval of 5 independent runs. This plot also shows results for a \bn~ network without mean deduction/adaption, validating our claim from Section 2.
          }}
          \label{fig:broken_bn_app}
\end{figure}

\end{document}